\documentclass[11pt,letterpaper]{article}

 \usepackage[hscale=0.75,vscale=0.8]{geometry}

\usepackage{todonotes}

\usepackage[utf8]{inputenc} 
\usepackage[T1]{fontenc}    
\usepackage{hyperref}       
\usepackage{url}            
\usepackage{booktabs}       
\usepackage{amsfonts}       
\usepackage{nicefrac}       
\usepackage{microtype}      
\usepackage{color,xcolor}         

\usepackage{amsmath,amsthm,bbm}
\usepackage{amssymb}
\usepackage{url}
\usepackage{graphicx}
\usepackage{float}
\usepackage{wrapfig}
\usepackage{hyperref}
\newcommand{\rmd}{\mathrm{d}}
\newcommand{\Lm}{\mathrm{L}^\alpha}
\newcommand{\rset}{\mathbb{R}}

\newtheorem{theorem}{Theorem}
\newtheorem*{theorem*}{Theorem}

\newtheorem{corollary}[theorem]{Corollary}
\newtheorem{definition}[theorem]{Definition}

\newtheorem{lemma}[theorem]{Lemma}
\newtheorem*{lemma*}{Lemma}

\newtheorem{remark}[theorem]{Remark}

\usepackage{comment}

\usepackage{enumitem}

\newcommand{\imagi}{\mathsf{i}}

\allowdisplaybreaks[4]

\usepackage{cleveref}

\usepackage[square,sort,comma,numbers]{natbib}

\title{Algorithmic Stability of Heavy-Tailed Stochastic Gradient Descent on Least Squares} 

\author{Anant Raj \\
 Coordinated Science Laboraotry \\
 University of Illinois Urbana-Champaign. \\
  Inria, Ecole Normale Sup\'erieure \\
  PSL Research University, Paris, France. \\
  \texttt{anant.raj@inria.fr} 
  \vspace{.25cm}
\and Melih Barsbey \\
Department of Computer Engineering \\ Bo\u{g}azi\c{c}i University, Istanbul, Turkey. \\
  \texttt{melih.barsbey@boun.edu.tr} \\
  \and Mert Gürbüzbalaban \\
  Department of Management\\ Science and Information Systems \\
 Rutgers University, Piscataway, USA. \\
  \texttt{mg1366@rutgers.edu} \\
  \and Lingjiong Zhu \\
  Department of Mathematics \\
  Florida State University, FL, USA. \\
  \texttt{zhu@math.fsu.edu} \\
  \and Umut \c{S}im\c{s}ekli \\
  Inria, CNRS, Ecole Normale Sup\'erieure \\
  PSL Research University, Paris, France. \\
  \texttt{umut.simsekli@inria.fr} \\
}

\begin{document}

\maketitle
\begin{abstract}%
  Recent studies have shown that heavy tails can  emerge in stochastic optimization and that the heaviness of the tails have links to the generalization error. While these studies have shed light on interesting aspects of the generalization behavior in modern settings, they relied on strong topological and statistical regularity assumptions, which are hard to verify in practice. Furthermore, it has been empirically illustrated that the relation between heavy tails and generalization might not always be monotonic in practice, contrary to the conclusions of existing theory. In this study, we establish novel links between the tail behavior and generalization properties of stochastic gradient descent (SGD), through the lens of algorithmic stability. We consider a quadratic optimization problem and use a heavy-tailed stochastic differential equation (and its Euler discretization) as a proxy for modeling the heavy-tailed behavior emerging in SGD. We then prove uniform stability bounds, which reveal the following outcomes: (i) Without making any exotic assumptions, we show that SGD will not be stable if the stability is measured with the squared-loss $x\mapsto x^2$, whereas it in turn becomes stable if the stability is instead measured with a surrogate loss $x\mapsto |x|^p$ with some $p<2$. (ii) Depending on the variance of the data, there exists a \emph{`threshold of heavy-tailedness'} such that the generalization error decreases as the tails become heavier, as long as the tails are lighter than this threshold. This suggests that the relation between heavy tails and generalization is not globally monotonic. (iii) We prove matching lower-bounds on uniform stability, implying that our bounds are tight in terms of the heaviness of the tails. We support our theory with synthetic and real neural network experiments.
\end{abstract}

\section{Introduction}

\label{sec:intro}

Over the last decade, understanding the generalization behavior in modern machine learning settings has been one of the main challenges in statistical learning theory. Here, the main goal has been deriving upper-bounds on the so-called \emph{generalization error}, i.e., the gap between the true and the empirical risks $|F(\theta) - \hat{F}(\theta,X)|$, which are respectively defined as follows:
\begin{align}
    F(\theta) := \mathbb{E}_{x\sim P_X }[f(\theta,x)], \qquad \hat{F}(\theta, X) := (1/n)\sum\nolimits_{i=1}^n f(\theta,x_i),
    \label{eqn:risks}
\end{align}
where $\theta \in \mathbb{R}^d$ denotes the \emph{parameter vector}, $f : \mathbb{R}^d \times \mathcal{X} \mapsto \mathbb{R}_+$ is the \emph{loss function}, $\mathcal{X}$ is the space of \emph{data points}, $P_X$ is the unknown \emph{data distribution}, and finally $X = \{x_1,\dots,x_n\}$ denotes a (random) dataset with $n$ points, such that each $x_i$ is independently and identically distributed (i.i.d.) from $P_X$.  

The past few years have witnessed the development of a variety of mathematical frameworks for analyzing the generalization error (see e.g.,  \citet{liu2019deepsurvey,he2020recent} for recent surveys). In the context of empirical risk minimization (ERM), i.e., solving $\min_{\theta \in \mathbb{R}^d} \hat{F}(\theta,X)$, one promising direction has been to explicitly take into account the statistical properties of the \emph{optimization algorithm} used during training, which is typically chosen as stochastic gradient descent (SGD) that is based on the following recursion:
\begin{align}
    \theta_{k+1} = \theta_k - \eta \nabla \tilde{F}_{k+1}(\theta_k,X), \label{eqn:sgd}
\end{align}
where $\eta$ is the step-size (or learning-rate), and $\nabla \tilde{F}_{k}(\theta,X) := \frac1{b} \sum_{i\in \Omega_k} f(\theta, x_i)$ is the stochastic gradient, with $\Omega_k \subset \{1,\dots,n\}$ being a random subset drawn with or without replacement, and $b := |\Omega_k| \ll n$ being the batch-size.
In this line of research, \cite{csimcsekli2019heavy} and \cite{martin2019traditional} empirically demonstrated that, perhaps surprisingly, a \emph{heavy-tailed} behavior can emerge in SGD in different ways, and the heaviness of the tails correlates with the generalization error, suggesting that heavier tails indicate better generalization.

Theoretically investigating these empirical observations, \cite{gurbuzbalaban2020heavy} and  \cite{hodgkinson2021multiplicative} explored the origins of the observed heavy-tailed behavior. They simultaneously showed that, in online SGD\footnote{The framework of \citet{hodgkinson2021multiplicative} can  handle stochastic optimization algorithms other than SGD as well.} (i.e., when the data is streaming), due the multiplicative nature of the gradient noise, i.e.,  $\nabla \tilde{F}_k(\theta,X) - \nabla \hat{F}(\theta,X)$, the distribution of the iterates $\theta_k$ can converge to a heavy-tailed distribution as $k \to \infty$. Furthermore, \cite{gurbuzbalaban2020heavy} showed that, when the loss $f$ is a quadratic and the data distribution is Gaussian, the tails become monotonically heavier when $\eta$ gets larger or $b$ gets smaller.

Due to the fact that analyzing the heavy-tailed behavior arising from \eqref{eqn:sgd} can be highly non-trivial, relatively simpler heavy-tailed mathematical models have been used as a proxy for the original heavy-tailed SGD recursion in stationarity; e.g., SGD with heavy-tailed noise, i.e.,
\begin{align}
\label{eqn:sgd_heavy}
    \theta_{k+1} = \theta_k - \eta_{k+1} \left[ \nabla \hat{F}(\theta_k,X) + \xi_{k+1} \right], \quad \text{with} \quad \mathbb{E}[\|\xi_k\|^2] = +\infty, \quad \text{for every $k = 1,2,\dots$},
\end{align}
where $\xi_k$ denotes the heavy-tailed noise and $\eta_k$ denotes a sequence of decreasing step-sizes. 
It has been revealed that another interesting situation emerges in this setting, this time in the behavior of the optimization error.
Notably, \citet[Remark 1]{zhang2020adam} pointed out that, when the loss function $f$ is chosen as a simple quadratic, i.e., $f(\theta,x) = \|\theta\|^2$, we have that $\mathbb{E}[\|\theta_k -\theta_\star\|^2] = \mathbb{E}[\|\theta_k\|^2] = \mathbb{E}[f(\theta_k,x)] = +\infty $ for all $k$, where $\theta_\star = 0$ is the global minimum of $f$.
While this result might appear daunting as it might seemingly suggest that ``SGD diverges'' under heavy-tailed perturbations, \cite{wang2021convergence} refined this result and showed that, if there exists $p \in [1,2)$ such that $\mathbb{E}[\|\xi_k\|^p] < \infty$, then $\mathbb{E}[\|\theta_k -\theta_\star\|^p]$ converges to zero, for a class of strongly convex losses $f$. This result is particularly remarkable, since it shows that, even when the iterates may diverge under the `true' loss function $f$ (which SGD is originally trying to minimize), i.e., $\mathbb{E}[f(\theta_k,x)] = +\infty $, they might still converge to the \emph{minimum of the original loss} $\theta_\star$ when a surrogate loss function $\tilde{f}$ is used for measuring the optimization error, which in this example is $\tilde{f}(\theta,x) = \|\theta\|^p$ with $p<2$.

In an initial attempt for formalizing the relation between the tail behavior and generalization, \cite{simsekli2020hausdorff} also modeled the original heavy-tailed recursion \eqref{eqn:sgd} by using a proxy and
considered the following stochastic differential equation (SDE) as a model (which can be seen as a continuous-time version of \eqref{eqn:sgd_heavy}):
\begin{align}
    \rmd \theta_t = - \nabla \hat{F}(\theta_t,X) \rmd t + \Sigma (\theta_t) \rmd \Lm_t, \label{eqn:sdegen}
\end{align}
where $\Sigma : \rset^d \mapsto \rset^{d\times d}$ is a matrix-valued function and $\Lm_t$ denotes a heavy-tailed $\alpha$-stable L\'{e}vy process, which is a random process parameterized by $\alpha \in (0,2]$, such that a smaller $\alpha$ indicates heavier tails (we will make the definition of $\Lm_t$ precise in the next section). They showed that, under several assumptions on the SDE \eqref{eqn:sdegen}, the worst-case generalization error over the trajectory, i.e., $\sup_{t\in [0,1]} |\hat{F}(\theta_t,X) - F(\theta_t)|$ scales with the intrinsic dimension of the trajectory $(\theta_t)_{t\in [0,1]}$, which is then  upper-bounded as a particular function of 
the tail-exponent around a local minimum, indicating that heavier tails imply lower generalization error. 
Their results were later extended to discrete-time recursions as well in \citet{hodgkinson2021generalization}. More recently, \cite{barsbey2021heavy} linked heavy-tails to generalization through a notion of compressibility in the over-parameterized regimes. Yet, these bounds require several topological and statistical regularity assumptions that are hard to verify in realistic settings, and the experiments in \citet{barsbey2021heavy} illustrated that the relation between the tail-exponent and the generalization error is not always monotonic; hence, a generalization bound that requires less assumptions while being more pertinent to the practical observations is still missing.

In this study, we aim at establishing novel links between tail behavior and generalization and address the aforementioned shortcomings. We consider the problem through the lens of \emph{algorithmic stability}  \citep{bousquet2002stability,hardt2016train}, and explore the effects of heavy tails on the stability of SGD. Similar to recent work  \citep{ali2020implicit,gurbuzbalaban2020heavy,latorre2021effect}, in order to have a more explicit control over the problem, we limit our scope to quadratic optimization, and consider the following SDE as a proxy for heavy-tailed SGD: 
\begin{equation}\label{eqn:htou}
\rmd \theta_t=- \frac1{n}\left(X^\top X\right) \theta_t\rmd t+\Sigma \rmd \Lm_{t},
\end{equation}
where $\Sigma \in \rset^{d\times d}$ is a matrix that scales the noise and is assumed to be fixed (i.e., state-independent), and by a slight abuse of notation we represent the dataset as a matrix $X \in \rset^{n\times d}$, such that $i$-th row of $X$ is equal to $x_i$. This SDE naturally arises from the ERM problem with the loss function being $f(\theta,x) = (\theta^\top x)^2$. 

As the learning algorithm, we first consider the case where we assume that we have a sample from the stationary distribution of \eqref{eqn:htou} (i.e., the case where $t \to \infty$) and analyze the stability of this sample. Then we extend our analysis in two directions: we analyze (i) the case where $t$ is finite and (ii) the case where the SDE is discretized by using a constant step-size. Our contributions are as follows:
\begin{itemize}[itemsep=0pt,topsep=0pt,leftmargin=*,align=left]
\item As opposed to classical SDEs driven by a Brownian motion (rather than $\alpha$-stable L\'{e}vy processes $\Lm_t$ as we consider here), the stationary distribution of \eqref{eqn:htou} does not admit a simple analytical closed-form expression. As a remedy, we perform the stability analysis in the Fourier domain, and introduce new proof techniques. 
\item We prove upper-bounds on the stability of \eqref{eqn:htou}, which suggest that the algorithm will not be stable, when $\alpha <2$ and the stability is measured with respect to the quadratic loss $(\theta^\top x)^2$. We further show that, when the stability is instead measured with respect to a surrogate loss function $|\theta^\top x|^p$ with $p<\alpha<2$, the algorithm in turn becomes stable, where the level of stability depends on $\alpha$, among several other quantities. This result reveals a similar phenomenon to that of \citet{zhang2020adam} and \citet{wang2021convergence} as discussed above. Furthermore, our results do not require any non-trivial assumptions, compared to the existing heavy-tailed generalization bounds \citep{simsekli2020hausdorff,hodgkinson2021generalization,barsbey2021heavy}.
\item Our theory further discloses an interesting property: depending on the variance of the data distribution $P_X$, there exits an $\alpha_0 >1$, such that the algorithm becomes more stable as $\alpha \in [\alpha_0,2]$ get smaller, i.e., the tails get heavier up to a certain point determined by $\alpha_0$. This result implies that the stability of the algorithm, hence the generalization error will be monotonic with respect to the tail-exponent $\alpha$ only when $\alpha$ is large enough. This outcome sheds more light on the experimental results presented in \citet{barsbey2021heavy}, where the relation between the tail-exponent and the generalization error is only partially monotonic.    
\item We prove matching lower-bounds on the stability of \eqref{eqn:htou}, implying that our stability bounds are tight in the tail-exponent $\alpha$.
\item We show that the same conclusions hold for finite $t$, and for the Euler discretization of \eqref{eqn:htou} when the step-size is small enough.
\end{itemize}
We support our theory on both synthetic data and real experiments conducted on standard benchmark datasets by using fully-connected and convolutional neural networks. All the proofs and the implementation details are provided in the Appendix.

\section{Notation and Background} \label{sec:background}

\textbf{Notation.} Consider a real-valued function $f : \mathbb{R}^d \rightarrow \mathbb{R}$ defined on $\mathbb{R}^d$. The Fourier transform of $f(\theta)$ for $\theta \in \mathbb{R}^d$ is denoted by $\mathcal{F}f(u)$ and is defined as,
    $\mathcal{F}f(u) := \int_{\mathbb{R}^d} f(\theta) e^{- \imagi u^\top \theta} \rmd \theta$. 
 Similarly, the inverse Fourier transform of a function $\hat{f}(u)$ that is from $\mathbb{R}^d$ to $\mathbb{R}$ is denoted by $\mathcal{F}^{-1}\hat{f}(\theta)$ and is defined by,
     $\mathcal{F}^{-1}\hat{f}(\theta) := \frac{1}{(2\pi)^d}\int_{\mathbb{R}^d} \hat{f}(u) e^{ \imagi u^\top \theta} du$. 
In both of these definitions, $\imagi := \sqrt{-1}$.

\paragraph{$\alpha$-stable distributions.}
The $\alpha$-stable distribution appears as the limiting distribution
in the generalized central limit theorems for
a sum of i.i.d. random variables
with infinite variance \citep{paul1937theorie}. 
A scalar random variable $X$ is called symmetric $\alpha$-stable, denoted
by $X\sim\mathcal{S}\alpha\mathcal{S}(\sigma)$, if its characteristic function
takes the form:
$\mathbb{E}\left[e^{\imagi uX}\right]=\exp\left(-\sigma^{\alpha}|u|^{\alpha}\right)$, for any $u\in\mathbb{R}$,
where $\sigma>0$ is known as the scale
parameter that measures the spread
of $X$ around $0$ and $\alpha\in(0,2]$ which is known as the tail-index
that determines
the tail thickness of the distribution
and the tail becomes heavier as $\alpha$ gets smaller.
In general, the probability density function of a symmetric $\alpha$-stable distribution, $\alpha\in(0,2]$,
does not yield closed-form expression except for a few special cases.
When $\alpha=1$ and $\alpha=2$, $\mathcal{S}\alpha\mathcal{S}$ reduces to the Cauchy and the Gaussian distributions, respectively.
When $0<\alpha<2$, the moments
are finite only up to the order $\alpha$
in the sense that 
$\mathbb{E}[|X|^{p}]<\infty$
if and only if $p<\alpha$,
which implies infinite variance.
Moreover, $\alpha$-stable distribution can be extended
to the high-dimensional case for random vectors. 
One natural extension is the rotationally symmetric $\alpha$-stable distribution.
$X$ follows a $d$-dimensional rotationally symmetric $\alpha$-stable distribution
if it admits the characteristic function $\mathbb{E}\left[e^{\imagi \langle u,X\rangle}\right]=e^{-\sigma^{\alpha}\Vert u\Vert_{2}^{\alpha}}$ for
any $u\in\mathbb{R}^{d}$.
We refer to \citet{ST1994} for the details of $\alpha$-stable distributions.

\paragraph{L\'{e}vy processes.}
L\'{e}vy processes are stochastic processes with independent and stationary increments.
Their successive displacements can be viewed as the continuous-time
analogue of random walks.
L\'{e}vy processes include the Poisson process, Brownian motion,
the Cauchy process, and more generally stable
processes; see e.g. \citet{bertoin1996,ST1994,Applebaum}.
L\'{e}vy processes in general admit jumps
and have heavy tails which are appealing
in many applications; see e.g. \citet{Cont2004}.
In this paper, we will consider the rotationally symmetric $\alpha$-stable L\'{e}vy process $\Lm_{t}$ in $\mathbb{R}^d$ that is defined as follows.
\begin{enumerate}[label=(\roman*),itemsep=0pt,topsep=0pt,leftmargin=*,align=left]
\item
$\Lm_0=0$ almost surely;
\item
For any $t_{0}<t_{1}<\cdots<t_{N}$, the increments $\Lm_{t_{n}}-\Lm_{t_{n-1}}$
are independent;
\item
The difference $\Lm_{t}-\Lm_{s}$ and $\Lm_{t-s}$
have the same distribution, with the characteristic function $\exp(- (t-s)^\alpha\|u\|_2^\alpha)$ for $t>s$;
\item
$\Lm_{t}$ has stochastically continuous sample paths, i.e.
for any $\delta>0$ and $s\geq 0$, $\mathbb{P}(\|\Lm_{t}-\Lm_{s}\|>\delta)\rightarrow 0$
as $t\rightarrow s$.
\end{enumerate}
When $\alpha=2$, $\Lm_{t}=\sqrt{2}\mathrm{B}_{t}$, where $\mathrm{B}_{t}$ is the standard $d$-dimensional Brownian motion. 

\paragraph{Ornstein-Uhlenbeck processes.}
Ornstein-Uhlenbeck (OU) process \citep{OUpaper} is a $d$-dimensional Markov and Gaussian process
that satisfies the SDE:
\begin{equation}\label{eqn:BM}
\rmd X_t=-AX_t\rmd t+\Sigma \rmd\mathrm{B}_{t},
\end{equation}
where $A$ and $\Sigma$ are a $d\times d$ matrices 
and $\mathrm{B}_{t}$ is a standard $d$-dimensional Brownian motion.
The OU process is a special case of the Langevin equation in physics \citep{pavliotis2014stochastic},
and has wide applications including for example modeling
the change in organismal phenotypes in evolutionary biology \citep{Martins1994},
and the short-rate in the interest rate modeling in finance \citep{Vasicek1977}.
More generally, we can consider an OU process
driven by a L\'{e}vy process, for example,
replacing $\mathrm{B}_{t}$ in \eqref{eqn:BM} by a rotationally symmetric $\alpha$-stable L\'{e}vy process $\Lm_{t}$ so that
\begin{equation}\label{eqn:2}
\rmd X_t=-AX_t \rmd t+\Sigma \rmd \Lm_t.
\end{equation}
Under some mild conditions on $A$ and $\Sigma$, the OU process $X_t$ in \eqref{eqn:2} admits a unique stationary distribution that can be fully characterized; see e.g. \citet{Sato1984,Masuda2004}.

\paragraph{Algorithmic stability and generalization. } 
In this paper, we study the generalization of the continuous-time heavy-tailed SGD by using the tools of algorithmic stability. Several notions of stability have been defined in the literature of statistical learning theory \citep{bousquet2002stability,elisseeff2005stability}. We will use the notion of algorithmic stability of the randomized algorithm $\mathcal{A}$ defined in \citet{hardt2016train}. 
We denote the set  $\mathcal{X}_n$ as the set of all possible size  $n$ datapoints subsampled uniformly at random from $P_X$. 
\begin{definition}[\cite{hardt2016train}, Definition 2.1] \label{def:stability}
 For a loss function $f:\mathbb{R}^d \times \mathcal{X} \rightarrow \mathbb{R}$, an algorithm $\mathcal{A}$ is $\varepsilon$-uniformly stable if
 \begin{align}
     \varepsilon_{\text{stab}}(\mathcal{A}) := \sup_{X\cong \hat{X}}\sup_{z \in \mathcal{X}}~ \mathbb{E}\left[f(\mathcal{A}(X),z) - f(\mathcal{A}(\hat{X}),z) \right] \leq \varepsilon ,
 \end{align}
where the first supremum is taken over data $X, \hat{X} \in \mathcal{X}_n$   that differ by one element, denoted by $X\cong \hat{X}$.
\end{definition}
 Since its introduction in statistical learning theory in \citet{bousquet2002stability}, stability based arguments have been useful in deriving generalization bound for several learning algorithms \citep{belkin2004regularization,cortes2012algorithms,cheng2021algorithmic,maurer2005algorithmic} and have also been extended to get generalization bound for randomized algorithm like SGD and SGLD \citep{hardt2016train,raginsky2017non,kuzborskij2018data,mou2018generalization,chen2018stability,bassily2020stability,charles2018stability,lei2020fine,farghly2021time}. Here below, we provide a result from \citet{hardt2016train} which relates algorithmic stability with the generalization performance of a randomized algorithm.
 
\begin{theorem}[\citep{hardt2016train}, Theorem 2.2]
Suppose that $\mathcal{A}$ is an $\varepsilon$-uniformly stable algorithm, then the expected generalization error is bounded by 
\begin{align}
    \left|\mathbb{E}_{\mathcal{A},X}~\left[ \hat{F}(\mathcal{A}(X),X) - F(\mathcal{A}(X))\right]  \right| \leq \varepsilon.
\end{align}
\end{theorem}
In several of recent works \citep{feldman2019high,bousquet2020sharper,klochkov2021stability}, high probability bounds have been obtained using algorithmic stability bounds.



\section{Algorithmic Stability of Heavy-Tailed SGD on Least Squares Regression}
\label{sec:char_function}

In this section, we will investigate the effects of heavy-tails on algorithmic stability. We consider the setting of least square regression with $f(\theta,(x,y)) = (\theta^\top x -y)^2/2$. 
We assume that we only have the access to the data generation distribution $P_{X}$ via the generated training samples and our goal is to learn a parameter vector $\theta \in \mathbb{R}^d$ which minimize the corresponding population risk. 
We denote the training data by the matrix $X=\left[x_1^\top , x_2^\top,\dots, x_i^\top,\dots , x_n^\top\right] \in \mathbb{R}^{n\times d}$ and $y = [y_1, y_2, \dots, y_i, \dots, y_n] \in \rset^{n}$, where $n$ is the number of data points, $d$ is the dimension of the problem, and $x_i \in \mathbb{R}^d$, $y_i \in \rset$ for all $i$.  Training data points are i.i.d.\ from the distribution $P_X$. We consider the ERM problem as defined in \eqref{eqn:risks}:
    $ \min_{\theta\in \mathbb{R}^{d}} \frac{1}{2n} \sum_{i=1}^n (\theta^\top x_i - y_i) ^2$.  

In the context of algorithmic stability, we assume that we have two training datasets $(X,y)$ and $(\hat{X},\hat{y})$ that differ in only one data point. Without loss of generality, we have 
\begin{align*}
\hat{X}=\left[x_1^\top , x_2^\top,\dots, \tilde{x}_i^\top,\dots,x_n^\top\right] \in \mathbb{R}^{n\times d},
\qquad
\hat{y} = [y_1, y_2, \dots, \tilde{y}_i, \dots, y_n]\in\mathbb{R}^{n}.
\end{align*}
For our ERM problem, we consider the continuous-time heavy-tailed stochastic gradient descent, which is represented by the following two SDEs that are driven by a rotationally symmetric $\alpha$-stable L\'{e}vy process $\Lm_{t}$ in $\mathbb{R}^d$,
\begin{align}
    \rmd \theta_t &= - \frac{1}{n} \left(X^\top X \theta_t - X^\top y \right)\rmd t + \Sigma \rmd \Lm_t, \label{eq:sde_invar_1}  \\
    \rmd \hat{\theta}_t &= -\frac{1}{n} \left(\hat{X}^\top \hat{X} \hat{\theta}_t - \hat{X}^\top \hat{y} \right) \rmd t + \Sigma \rmd \Lm_t, \label{eq:sde_invar_2}
\end{align}
where $\Sigma \in \mathbb{R}^{d\times d}$ is a real-valued matrix.
 
Under mild conditions, the SDEs \eqref{eq:sde_invar_1} and \eqref{eq:sde_invar_2} have unique strong solutions, which are Markov processes and they admit unique invariant distributions \citep{Sato1984}. Thanks to the linearity of the drifts of these SDEs, the stationary distribution is achieved very quickly, with an exponential rate \citep{xie2020ergodicity}. Hence, to ease our analysis, we will assume that we have two samples from the stationary distributions of \eqref{eq:sde_invar_1} and \eqref{eq:sde_invar_2}, say $\theta$ and $\hat{\theta}$. In other words, we set our learning algorithm such that it gives a random sample from the stationary distribution of the SDE determined by the dataset, i.e., $\mathcal{A}_{\text{cont}}((X,y)) = \theta$, and $\mathcal{A}_{\text{cont}}((\hat{X},\hat{y}))= \hat{\theta}$, where $\mathcal{A}_{\text{cont}}$ denotes the \emph{continuous-time} heavy-tailed SGD algorithm. In the rest of this section, we will derive stability bounds for this learning algorithm.

 
\vspace{-2pt} 
\subsection{Warm-up: the need for the surrogate loss}
\vspace{-2pt}

To motivate our analysis technique, let us first consider the following simple setting, where we set $d=1$, so that we have $f(\theta,(x,y)) = (x\theta-y)^2$. In this specific case, when $\alpha>1$, we can compute the stationary distributions of \eqref{eq:sde_invar_1} and \eqref{eq:sde_invar_2} in an explicit form. With a slight abuse of notation, the distribution of $\theta_t$ converges to a symmetric stable law: $(\delta/s) + \mathcal{S}\alpha\mathcal{S}((\alpha s )^{-1/\alpha} )$, where $s = (1/n)\sum_{i=1}^n x_i^2$ and $\delta = (1/n) \sum_{i=1}^n x_i y_i$ with $\delta/s$ being the mean (and the mode) of the stationary distribution, which coincides with the ordinary least-squares solution. 
Similarly, the distribution of $\hat{\theta}_t$ converges to $(\hat{\delta}/\hat{s}) + \mathcal{S}\alpha\mathcal{S}((\alpha \hat{s} )^{-1/\alpha} )$, where $\hat{\delta}$ and $\hat{s}$ are defined analogously. 

As a first observation, assume that we have a sample from the stationary distribution of $\theta_t$, such that $\theta \sim (\delta/s) + \mathcal{S}\alpha\mathcal{S}((\alpha s )^{-1/\alpha} )$. Considering this scheme as the algorithm, i.e., $\mathcal{A}_{\text{cont}}((X,y)) = \theta$, a simple calculation shows that 
\begin{align*}
    \mathbb{E}_{ \mathcal{A}_{\text{cont}} (X,y)}\left[f\left(\mathcal{A}_{\text{cont}}((X,y)),(X,y)\right)\right] 
    = \mathbb{E}_{\theta, (X,y)}\left[(1/n)\sum\nolimits_{i=1}^n(x_i\theta-y_i)^2\right] = +\infty,
\end{align*}
since the variance of $\mathcal{S}\alpha\mathcal{S}((\alpha s )^{-1/\alpha})$ is infinite whenever $\alpha<2$. Therefore, it is clear that we cannot expect any algorithmic stability in this scheme, as long as the stability is measured with respect to the squared loss. However, as we will show in the sequel, in turns out that if we instead measure the stability with respect to a surrogate loss function, which in this case would be $|x\theta-y|^p$ for some $p \in [1,\alpha)$, the algorithm becomes stable, even though it is based on a distribution that concentrates near the optimum for squared loss.

On the other hand, we notice that the means of the stationary distributions, i.e., $\delta/s$ and $\hat{\delta}/\hat{s}$ do not interact with the tail exponent $\alpha$. Since our main goal is to investigate the interplay between the tail behavior and algorithmic stability, we will ignore this term and assume that $y_i=0$ almost surely for all $i$ (otherwise non-zero $y_i$ will only introduce terms in the stability that do not depend on $\alpha$). This way, we fall back to the SDE given in \eqref{eqn:htou}. 

In the light of these two observations, for the general case where $d\geq 1$, we will use the following surrogate loss function to measure stability: 
\begin{align}
f(x) := f(\theta,x) := |\theta^\top x|^p, \quad \text{for some} \quad  p \in [1,2],    \label{eqn:surrogate}
\end{align}
which generalizes the original loss function. Note that, from now on we will drop the notation $\tilde{f}$ for denoting surrogate losses for simplicity and use a single notation for the loss function. 


\vspace{-2pt}
\subsection{Algorithmic stability analysis in the Fourier domain}
\vspace{-2pt}

For $d\geq 2$, unfortunately we cannot identify the stationary distributions of \eqref{eq:sde_invar_1} and \eqref{eq:sde_invar_2} in an explicit form. However, by using the theory of the characterization of the stationary distribution
for an Ornstein-Uhlenbeck process driven by a L\'{e}vy process in the literature (see \citet{Sato1984,Masuda2004} and the background review in the Appendix), in the next lemma, we show that we can
%
characterize the stationary distribution
of the Ornstein-Uhlenbeck process driven by a rotationally symmetric $\alpha$-stable L\'{e}vy process in a semi-explicit way:
\begin{equation}\label{char:stat:dist}
\rmd \theta_t=-A \theta_t\rmd t+\Sigma \rmd \Lm_t,
\end{equation}
where $A$ and $\Sigma$ are $d\times d$ real matrices.
%
\begin{lemma}\label{lem:char}
Assume that $A$ is a real symmetric matrix with all the eigenvalues being positive.
Then \eqref{char:stat:dist} admits a unique stationary distribution 
$\pi$ whose characteristic function is given by
\begin{equation}
\int_{\mathbb{R}^{d}}e^{\imagi  u^\top x}\pi(\rmd x)
=\exp\left(-\int_{0}^{\infty}\left\Vert\Sigma^{\top}e^{-sA}u\right\Vert_{2}^{\alpha}\rmd s\right).
\end{equation}
\end{lemma}


While Lemma~\ref{lem:char} provides us information about the stationary distributions of the SDEs \eqref{eq:sde_invar_1} and \eqref{eq:sde_invar_2}, it considers the Fourier transforms of these distributions, which makes this setting not amenable to conventional algorithmic stability analysis tools. 

As a remedy, we perform the stability analysis directly in the Fourier domain and use the Fourier inversion theorem to compute stability bounds for continuous-time heavy-tailed SGD. 
Our main approach is based on the following observation. Let $g:\rset^d \mapsto \rset$ be a function, and $P$, $Q$ be random variables in $\rset^d$ with respective characteristic functions $\psi_P$ and $\psi_Q$. If the Fourier inversion theorem holds on $g$, then $g$ is the inverse Fourier transform of $\mathcal{F}g(\cdot)$. Hence, 
\begin{align}
   \mathbb{E}\left[g(P) - g(Q)\right] &=  \frac{1}{(2\pi)^d} \mathbb{E}\left[\int_{\mathbb{R}^d} \left(e^{\imagi  u^\top P} - e^{\imagi  u^\top Q}\right) \mathcal{F}g(u) ~\rmd u\right] \notag   
    \\
    &= \frac{1}{(2\pi)^d} \int_{\mathbb{R}^d} \mathbb{E}\left[e^{\imagi  u^\top P} - e^{\imagi  u^\top Q}\right] \mathcal{F}g(u) ~\rmd u  \notag \\
    &= \frac{1}{(2\pi)^d} \int_{\mathbb{R}^d}  \left( \psi_P(u) - \psi_Q(u)\right) \mathcal{F}g(u) ~\rmd u  \notag
    \\
    &\leq \frac{1}{(2\pi)^d} \int_{\mathbb{R}^d} |\psi_P(u) - \psi_Q(u)| |\mathcal{F}g(u)|~\rmd u. \label{eq:char_formulation}
\end{align}
Hence, \eqref{eq:char_formulation} enables us to utilize the result given in Lemma~\ref{lem:char} and hence gives us a way to perform stability analysis (as given in Definition~\ref{def:stability}) in the Fourier domain.



\paragraph{Algorithmic stability via characteristic function. }
By setting $y = \hat{y} = 0$ and invoking Lemma~\ref{lem:char}, 
the characteristic functions of stationary distributions corresponding to the SDEs \eqref{eq:sde_invar_1} and \eqref{eq:sde_invar_2} are respectively given as follows:
\begin{align}
    \psi_{\theta}(u) &= \exp\left( - \int_{0}^{\infty} \left\|\Sigma^\top e^{-s \frac{1}{n}({X}^\top {X})}u \right\|_2^\alpha \rmd s\right), \label{eq:char_sde_1}\\
    \psi_{\hat{\theta}}(u) &= \exp\left( - \int_{0}^{\infty} \left\|\Sigma^\top e^{-s \frac{1}{n}(\hat{X}^\top \hat{X})}u \right\|_2^\alpha \rmd s\right). \label{eq:char_sde_2}
\end{align}

From the Definition~\ref{def:stability} and from \eqref{eq:char_formulation} and \eqref{eqn:surrogate}, we have
\begin{align}
    \varepsilon_{\text{stab}} (\mathcal{A}_{\text{cont}}) &= \sup_{X\cong \hat{X}}\sup_{x \in \mathcal{X}} \mathbb{E}\left[\left|\theta^\top x\right|^p - \left|\hat{\theta}^\top x\right|^p  \right] \notag \\
    &= \sup_{X\cong \hat{X}}\sup_{x \in \mathcal{X}}  \frac{1}{(2\pi)^d} \int_{\mathbb{R}^d} |\psi_{\theta}(u) - \psi_{\hat{\theta}}(u)|\cdot \left|\mathcal{F}\left[|x^\top \cdot|^p\right](u)\right|~\rmd u.  \label{eq:stab_char}
\end{align}
In the remainder of this section, we will consider $\Sigma = I$ for convenience with $I$ being the identity matrix. However, we provide bounds showing the effect of $\Sigma$ in the Appendix. 

\paragraph{One-dimensional case ($d=1$). } 
We first discuss the case where $d=1$ and report it as a separate result since its proof is simpler and more instructive. Following \eqref{eq:char_formulation}, as a first step, we prove a lemma, which relates the characteristic functions of the stationary distributions by upper-bounding $|\psi_{\theta}(u) - \psi_{\hat{\theta}}(u)|$. 
For the sake of brevity, we present this result in the Appendix (Lemma~\ref{lem:1d_char_func}). 
%
%
%
%
By using this intermediate result, we next prove upper- and lower-bounds on the stability of the continuous time heavy-tailed SGD algorithm and discuss its behavior with respect to $\alpha$ and $p$.
\begin{theorem}\label{thm:1d_main}
Consider the one-dimensional loss function $f(x)  =|\theta x|^p$.  For any $x\sim P_X$, if we have $|x| > R $ with probability $\delta_1$ and for any $X$ sampled uniformly at random from the set $\mathcal{X}_n$, if we have $\|X\|_2^2 \leq \sigma^2 n$  with probability $\delta_2$. Then,
\begin{enumerate}
    \item[(i)] For $\alpha \in [1,2)$, the algorithm is not stable when $p \in [\alpha,2]$ i.e. $\varepsilon_{\text{stab}(\mathcal{A}_{\text{cont}})}$ diverges. When $\alpha = p =2$ then $ \varepsilon_{\text{stab}} (\mathcal{A}_{\text{cont}}) \leq \frac{R^4}{\pi \sigma^4 n}$ with probability  at least $1- \delta_1 -2\delta_2$.
    \item[(ii)] For $p\in [1,\alpha)$, we have the following upper bound for the algorithmic stability,
    \begin{align*}
         \varepsilon_{\text{stab}} (\mathcal{A}_{\text{cont}})
    &\leq    \frac{2R^{p+2}}{\pi\sigma^2 n }\Gamma(p+1) \cos\left( \frac{(p-1)\pi}{2}\right) \frac{1}{\alpha} \left( \frac{1}{\alpha \sigma^2 }\right)^{\frac{p}{\alpha}} \Gamma \left( 1 - \frac{p}{\alpha}\right)=: c(\alpha),
    \end{align*}
    which holds with probability at least $1- \delta_1-2\delta_2$. Furthermore, for some $\alpha_0 > 1$, if we have 
    \begin{align}
        \label{eqn:sigma_bound}
        \sigma^2 \geq \exp\left( 1 + \frac{2}{p} - \log \alpha_0 -  \phi\left( 1 - \frac{p}{\alpha_0}\right) \right),
    \end{align}
    where $\phi$ is the digamma function, then the map $\alpha\mapsto c(\alpha)$ is increasing for $\alpha \in [\alpha_0,2)$.
    \item[(iii)]  The stability bound is tight in $\alpha$.  
\end{enumerate}
\end{theorem}
%
Informally, this result illuminates the following facts: (i) When subject to heavy tails, i.e., $\alpha<2$, the algorithm is stable only when a surrogate loss is used with $p<\alpha$. (ii) For $1\leq p<\alpha<2$, the stability level $\varepsilon_{\text{stab}}$ is upper-bounded by a function of $\alpha$, $p$, and the variance of the data distribution $\sigma^2$. Furthermore (and perhaps more surprisingly), for a given \emph{heavy-tailedness threshold} $\alpha_0 \in (1,2)$, if the data variance is sufficiently large as in \eqref{eqn:sigma_bound}, the stability bound becomes monotonically increasing for $\alpha \in [\alpha_0, 2)$, which indicates that as the algorithm becomes more stable it gets heavier-tailed. However, this relation holds as long as the heaviness of the tails does not exceed the threshold $\alpha_0$. (iii) We further show that, there exists a data distribution $P_X$ such that $\varepsilon_{\text{stab}}$ is lower-bounded by a function, which also depends on $\alpha$, $p$, and $\sigma^2$. In the proved lower-bound, the terms depending on $\alpha$ have the same order as of the ones given in the upper-bound of Theorem~\ref{thm:1d_main}. Hence our stability bound is tight in $\alpha$.  
Combined with point (ii), this result suggests that the generalization error might not be globally monotonic with respect to the heaviness of the tails under our modeling strategy. On the other hand, for a fixed data distribution where $\sigma^2$ is given, \eqref{eqn:sigma_bound} provides a `guideline' for choosing the optimal tail index $\alpha$ in the sense of algorithmic stability.

\paragraph{Multi-dimensional case ($d\geq 2$).  } 
Now we will focus our attention to the case of $d$ dimensions. We follow the same route as in Theorem~\ref{thm:1d_main}, where we first relate the characteristic functions of the stationary distributions. We also present this result in the Appendix (Lemma~\ref{lem:dd_char_func}).  
Based on Lemma~\ref{lem:dd_char_func}, we next provide stability bounds for the $d$-dimensional case. 
%
%
\begin{theorem}\label{thm:dd_main}
Consider $f(x) =|\theta^\top x|^p$ such that $\theta, x \in \mathbb{R}^d$. Assume that for almost all $x\sim P_X$, we have $\|x\|_2 \leq R $, for any $X$ sampled uniformly at random from the set $\mathcal{X}_n$, we have $\frac{1}{n} \| X^\top X u\|_2 \geq  \sigma_{\min} \|u\|_2$ for all $u\in \mathbb{R}^d$ and for any two $X\cong \hat{X}$ sampled from $\mathcal{X}_n$ generating two stochastic process given by SDEs in equations~\eqref{eq:sde_invar_1} and \eqref{eq:sde_invar_2}, we have $\|x_i x_i^\top - \tilde{x}_i \tilde{x}_i^\top \|_2  \leq 2\sigma$ holds with high probability. Then,
\begin{enumerate}
    \item[(i)] For $\alpha \in (1,2)$, the algorithm is not stable when $p \in [\alpha,2]$ i.e. $\varepsilon_{\text{stab}(\mathcal{A}_{\text{cont}})}$ diverges. When $\alpha = p =2$ then with high probability $ \varepsilon_{\text{stab}} (\mathcal{A}_{\text{cont}}) \leq \frac{2R^2}{\pi}  \frac{\sigma}{n   \sigma_{\min}^2}$.
    \item[(ii)] For $p\in [1,\alpha)$, we have the following upper bound for the algorithmic stability,
    \begin{align*}
         \varepsilon_{\text{stab}} (\mathcal{A}_{\text{cont}})&=  \frac{8R^p}{\pi} \frac{\sigma}{n\alpha^2 \sigma_{\min} } \Gamma(p+1) \cos\left( \frac{(p-1)\pi}{2}\right)  \left(\frac{1}{\alpha \sigma_{\min}}\right)^{\frac{p}{\alpha}} \Gamma\left(1-\frac{p}{\alpha}\right)= c(\alpha),
    \end{align*}
    which holds with high probability. Furthermore, for some $\alpha_0 > 1$, if we have $$\sigma_{\min} \geq \exp\left( 1 + \frac{4}{p} - \log \alpha_0 -  \phi\left( 1 - \frac{p}{\alpha_0}\right) \right),$$
    where $\phi$ is the digamma function, then the map $\alpha\rightarrow c(\alpha)$ is increasing for $\alpha \in [\alpha_0,2)$.
    \item[iii] The stability bound is tight in $\alpha$.
    \end{enumerate}
\end{theorem}
The conclusions of Theorem~\ref{thm:dd_main} are almost identical to the ones of Theorem~\ref{thm:1d_main}, though its proof requires a more careful analysis, especially for the lower-bound in (iii). The main differences here are that, we need the smallest eigenvalue of the covariance matrix of $P_X$, i.e., $\sigma_{\min}$ to be large enough, and we need a different condition on the second moment $\sigma$ of the distribution. Under these conditions, we obtain very similar stability and monotonicity properties. 

As a final remark, we note that our results do not require any non-trivial topological or statistical assumptions in comparison with \citet{simsekli2020hausdorff} and \citet{barsbey2021heavy} that suggested a globally monotonic relation for the generalization error and the tail exponent $\alpha$. On the other hand, the rate $1/n$ in our bounds are in line with the existing stability literature \citep{hardt2016train,maurer2005algorithmic}.

\paragraph{Finite time bound.} The result presented in Theorem~\ref{thm:dd_main} is for the case when $t\rightarrow \infty$ i.e. $\theta$ is sampled from the stationary distribution of the stochastic process corresponding to the SDE in equation~\eqref{eqn:htou}. However, in the Appendix~\ref{sec:finite:time:continuous}, we characterize the finite time distribution of a Lévy-driven OU process. We show that the characteristic function of the probability density corresponding to the SDE in equation~\eqref{eqn:htou} is given as,
\begin{align*}
    \psi_{\theta}(t,u) &= \exp\left( - \int_{0}^{t} \left\|\Sigma^\top e^{-s \frac{1}{n}({X}^\top {X})}u \right\|_2^\alpha \rmd s\right).
\end{align*}
If we observe carefully, we can follow the similar procedure to get the stability bound for finite time case as we did to obtain for $t\rightarrow \infty$. In particular, in Remark~\ref{rem:fintie_OU} in the appendix, we show that whenever $t = O(\frac{1}{\alpha \sigma_{\min}})$, the same monotonicity conclusions of Theorem~\ref{thm:dd_main} still hold. See Remark~\ref{rem:fintie_OU} for more details. 

\paragraph{Algorithmic stability for the Euler discretization.} Previously, we have provided results for the continuous-time case which can not be implemented in practice. Now, we derive a stability bound for the Euler discretization of the SDE \eqref{eqn:htou}. We consider the following scheme:
\begin{align}
    \theta_{k+1} &= \theta_{k}- \frac{\eta}{n} \left(X^\top X \theta_{k} - X^\top y \right) +\eta^{1/\alpha} \Sigma S_{k+1}, \label{eq:sde_invar_1:discrete_main}  \\
    \hat{\theta}_{k+1} &=\hat{\theta}_{k} -\frac{\eta}{n} \left(\hat{X}^\top \hat{X} \hat{\theta}_{k} - \hat{X}^\top \hat{y} \right) + \eta^{1/\alpha}\Sigma S_{k+1}. \label{eq:sde_invar_2:discrete_main}
\end{align}
To provide algorithmic stability guarantees for the discretization, 
we first identify the characteristic function of the stationary distribution of the discretization in Appendix~\ref{discrete:least:square} (Lemma~\ref{lem:discrete:time}). We then provide a stability bound based on these characteristic functions. We only present the result for $k\rightarrow \infty$ here, however, the result for any finite $k$ follows the same procedure as we have provided stability bound for characteristic function for any finite $k$ in Lemma~\ref{lem:discretized_SDE_char}. 
\begin{theorem}\label{thm:discrete_time_bound}
Consider $f(x) =|\theta^\top x|^p$ such that $\theta, x \in \mathbb{R}^d$. Assume that for almost all $x\sim P_X$, we have $\|x\|_2 \leq R $, for any $X$ sampled uniformly at random from the set $\mathcal{X}_n$, it holds that  $\frac{1}{n} \| X^\top X u\|_2 \geq  \sigma_{\min} \|u\|_2$ for all $u\in \mathbb{R}^d$ and for any two $X\cong \hat{X}$ sampled from $\mathcal{X}_n$ generating the stochastic process given in \eqref{eq:sde_invar_1:discrete_main} and \eqref{eq:sde_invar_2:discrete_main}, we have that  $\frac{1}{n}\hat{X}^\top\hat{X}$, $\|x_i x_i^\top - \tilde{x}_i \tilde{x}_i^\top \|_2  \leq 2\sigma$ holds  with high probability. Further assume that $\eta \leq \frac{1}{L}$ where $L$ is the maximum of largest eigenvalues of $\frac{1}{n}X^\top X$. Then, for $p\in [1,\alpha)$, we have 
\begin{align*}
    \varepsilon_{\text{stab}} \leq \frac{2R^p}{\pi} \Gamma(p+1) \cos\left( \frac{(p-1)\pi}{2}\right) \frac{\sigma \eta^{1+\frac{p}{\alpha}} (1-\eta \sigma_{\min})^{\alpha-1}}{n\alpha(1-(1-\eta\sigma_{\min})^{\alpha})^{1+\frac{p}{\alpha}}}\Gamma\left(1-\frac{p}{\alpha}\right)
\end{align*}
with high probability.
\end{theorem}
This theorem shows that the monotonicity behavior of algorithmic stability with respect to $\alpha$ can be more complicated. However, if the step-size $\eta$ is chosen small enough such that we can consider the approximation $(1 - \eta \sigma_{\min})^{\alpha} \approx 1- \eta \alpha \sigma_{\min}$ then the result from above Theorem~\ref{thm:discrete_time_bound} can be written as,
$\varepsilon_{\text{stab}} \leq \frac{2R^p}{\pi} \Gamma(p+1) \cos\left( \frac{(p-1)\pi}{2}\right) \frac{\sigma }{n\alpha^2\sigma_{\min}} \left(\frac{1}{\alpha \sigma_{\min}}\right)^{\frac{p}{\alpha}}\Gamma\left(1-\frac{p}{\alpha}\right)$,
with high probability. This expression is almost the same as the bound given in Theorem~\ref{thm:dd_main}. Hence, we get a similar behavior of $\varepsilon_{\text{stab}}$ with respect to $\alpha$ for the discretized SDEs when $\eta$ is small enough. Finally, this result is independent of $\eta$, which is perhaps not surprising as similar results also exist for SGD with strong convex losses \cite[Theorem 3.9]{hardt2016train}.

\begin{figure}[t]
    \centering
    \includegraphics[width=0.99\textwidth]{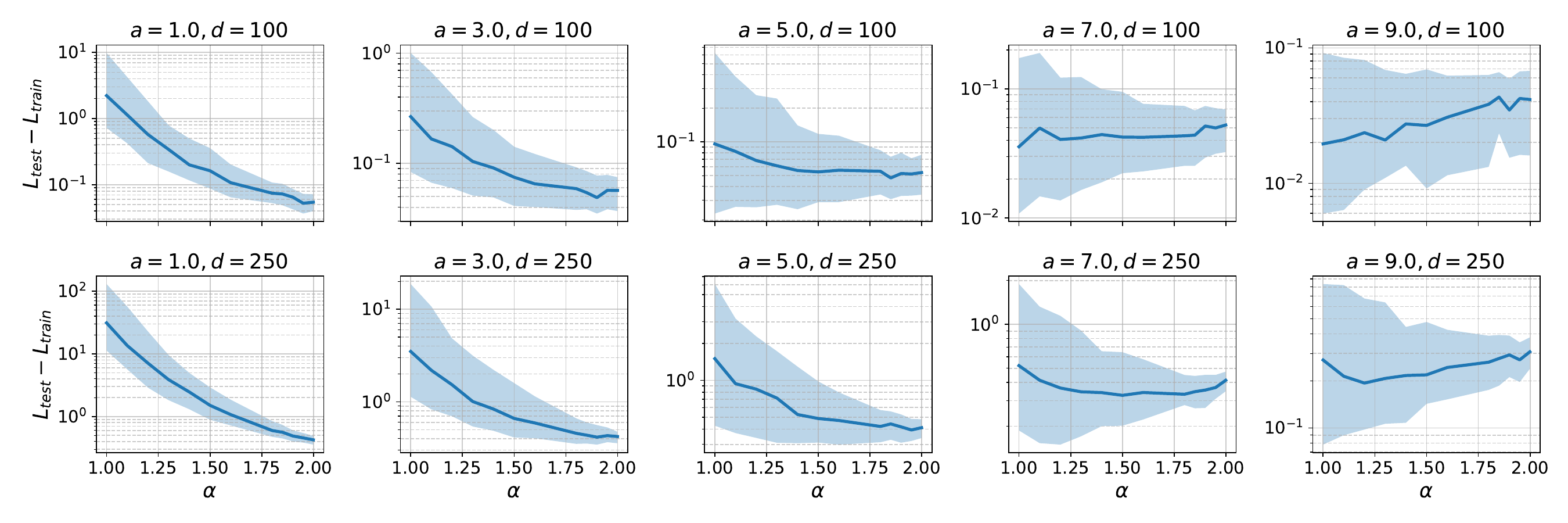}
    \vspace{-15pt}
    \caption{Results of the synthetic data experiments with varying $a$, $\alpha$, and $d$. Each experiment was repeated $200$ times with $n=1000$. The lines correspond to the median, and the shaded areas are the interquartile ranges.}
    \label{fig:synth_data}
\end{figure}

\vspace{-3pt}
\section{Experiments} \label{sec:expt}
\vspace{-3pt}
\textbf{Synthetic data. }
We first test the implications of the theoretical findings presented above with synthetic data experiments. We assume that $y=0$ and $P_X$ is a scaled uniform distribution $\mathcal{U}(-0.5a, 0.5a)$, with $a$ determining the range of the distribution. 
We simulate the SDE presented in \eqref{eq:sde_invar_1} by using the Euler-Maruyama discretization, which yields the following recursion:
\begin{align}
    \theta_{k+1} &= \theta_k - \eta \frac{1}{n} \left(X^\top X\right) \theta_k - \eta^{1/\alpha} E_{k+1}, \label{eq:gradient_descent} 
\end{align}
where $\eta$ is the learning-rate and each $E_{k} \in \mathbb{R}^d$ is a rotationally symmetric $\alpha$-stable random vector. 

In the experiments, we systematically vary $a$ as well as the tail-index of the additive noise, $\alpha$. For all experiments we set $\eta = 0.1$ and ran the algorithm for $3000$ iterations. We set $n=1000$ and varied $d$ to be $100$ or $250$. The order of the loss function $f$ was selected to be $p = 1$. For each experimental setting, we repeated the experiment $200$ times, where after sampling a population $N=100000$ observations, for each replication we sampled $n = 1000$ with replacement from within this population. The generalization error was computed to be the difference between loss computed on the replication sample of size $n$ and the population of size $N$. To prevent numerical issues, the noise was scaled with a constant of $0.1$ in all experiments, which corresponds to choosing $\Sigma = (1/10) I$.

The results are presented in Figure~\ref{fig:synth_data}, and corroborate the trend predicted by Theorem~\ref{thm:dd_main}. As $a$ grows, the variance of the input increases, leading the map $\alpha \mapsto c(\alpha)$ to become increasing for $\alpha \in [\alpha_0, 2)$ for some $\alpha_0$. Since $c(\alpha)$ is the upper bound for stability $\varepsilon_{\text{stab}}$, this leads to the observed `V-shaped' trend in generalization error for higher $a$ values, where the inflection point corresponds to $\alpha_0$ for a given experiment setting.\footnote{We note that the rather large error bars in Figure~\ref{fig:synth_data} are caused by the randomness coming from the heavy-tails (i.e., not by the randomness due to the choice of datasets). As we are essentially trying to compute the expectation of a heavy-tailed random variable by using a finite number of samples, these errors bars are not surprising as the task is notoriously difficult \citep{lugosi2019mean}.}


\textbf{Experiments on image data.} 
In our second set of experiments, we consider a real image classification task, where we use plain SGD \eqref{eqn:sgd} \emph{without} adding explicit heavy-tailed noise and monitor the effect of the heavy-tails that are \emph{inherently} introduced by SGD, as shown in \citet{gurbuzbalaban2020heavy,hodgkinson2021multiplicative}. In this context, we will view the SDE \eqref{eqn:htou} as a proxy to the original SGD recursion near a local minimum, so that a quadratic approximation would be pertinent.

Here, we train two fully connected neural networks (FCN) of different depths (4 vs.\ 6) as well as a convolutional neural network (CNN) on the \textsc{MNIST}, \textsc{CIFAR-10}, and \textsc{CIFAR-100} datasets \citep{mnist2010, cifar102009}. We train these models under different, constant learning rates ($\eta$) and with batch sizes ($b$) of $50$ or $100$, producing models trained under a wide range of $\eta/b$ values. The models are trained until convergence, where the convergence criteria for \textsc{MNIST} and \textsc{CIFAR-10} is a training negative log-likelihood (NLL) of $<5 \times 10^{-5}$ and a training accuracy of $100\%$, and for \textsc{CIFAR-100} these are a NLL of $<1 \times 10^{-2}$ and a training accuracy of $>99\%$.

For the estimation of the trained networks' tail indices, we used the multivariate estimator proposed in \cite[Corollary 2.4]{mohammadi2015estimating}\footnote{We note that this estimator has been shown to be consistent; yet, we do not have an non-asymptotic understanding of the esimates \citep{mohammadi2015estimating}.}, which is previously used in various related neural network research \citep{simsekli2020hausdorff,gurbuzbalaban2020heavy,zhou2020towards,barsbey2021heavy}. Since this estimator assumes a stable distribution, after convergence we obtained $1000$ iterations of SGD and computed the average to be used in this estimation, based on the generalized central limit theorem \cite[Corollary 11]{gurbuzbalaban2020heavy}, which demonstrated that the ergodic averages of heavy-tailed SGD iterates converge to a multivariate stable distribution. Before estimating the parameters, we centered the parameters with median values. Each layer's tail-index estimation was conducted separately, which were in turn averaged to produce a single tail-index for every model, as in \citet{barsbey2021heavy}. See the Appendix for further details.
\begin{wrapfigure}{r}{0.60\textwidth}
    \vspace{10pt}
    \includegraphics[width=0.60\textwidth]{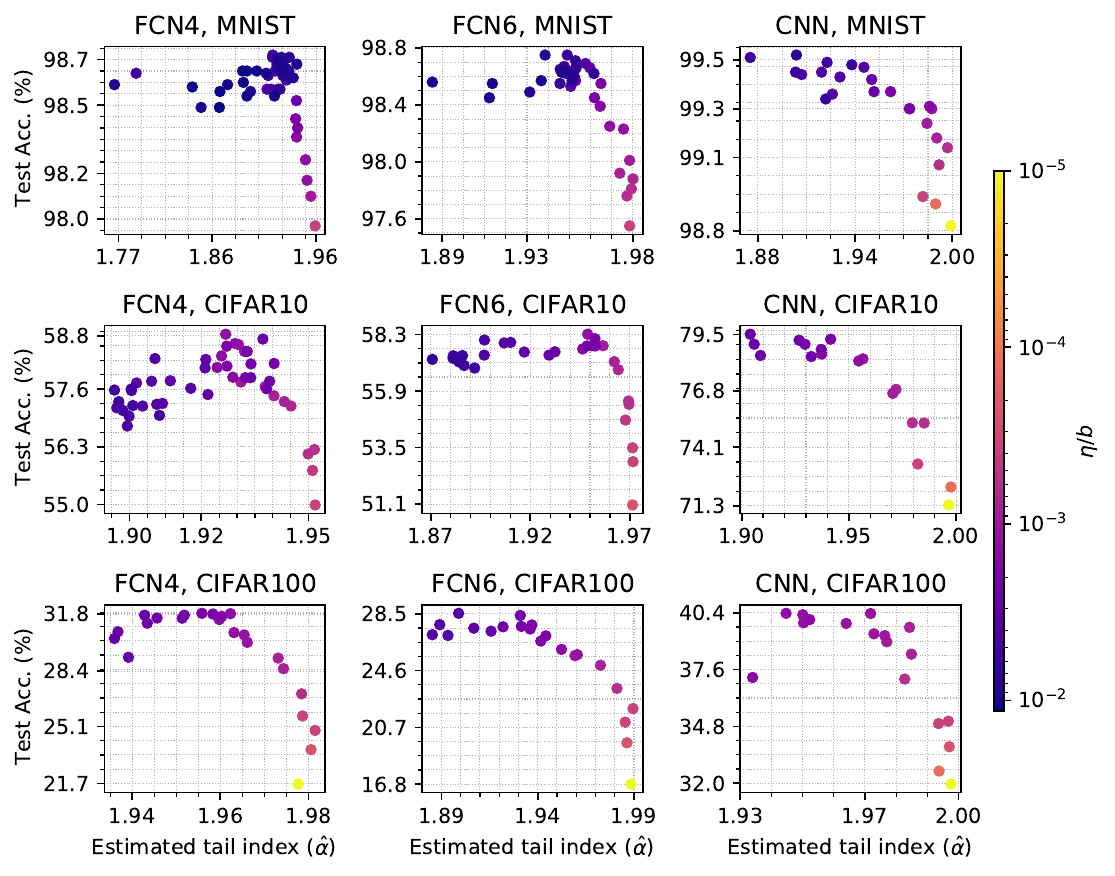}
    \vspace{-30pt}
    \caption{Test accuracy vs. mean estimated tail-index ($\hat{\alpha}$) for each model. Color: training $\eta/b$ ratio.}
    \vspace{-10pt}
    \label{fig:mnist_cifar10_cifar100}
  \end{wrapfigure}

Previous literature demonstrated that (i) training neural networks with larger $\eta/b$ values lead to heavy-tailed parameters \citep{gurbuzbalaban2020heavy} and (ii) networks with heavier-tailed parameters are more likely to generalize \citet{simsekli2020hausdorff, barsbey2021heavy}. Here, Figure~\ref{fig:mnist_cifar10_cifar100} demonstrates that networks with highest $\alpha$ (light-tails) consistently perform worst in terms of generalization and the performance improves as the $\alpha$ decreases until some \emph{threshold}. This outcome is in line with the predictions of our theoretical results, which suggest a `V-shaped' behavior for the relation between generalization and $\alpha$, as opposed to \citet{simsekli2020hausdorff, barsbey2021heavy}. As a final remark, here the values of $\alpha$ are larger compared with the synthetic experiments; however, we shall emphasize that such values for $\alpha$ still indicate strong heavy tails.

\section{Conclusion}

We established novel links between the tail behavior and generalization properties of SGD building on the notion of algorithmic stability. We focused on quadratic optimization and considered a heavy-tailed SDE previously proposed as a proxy to SGD dynamics. We then proved uniform stability bounds which uncover several phenomena about the effect of the heaviness of the tails on the generalization. We also established lower bounds which show that our stability bounds are tight in terms of the heaviness of the tails. We then extended our results to the finite-time, and to the discrete-time cases and showed that similar results hold. We finally supported our theory on a variety of experiments. Future work includes extending our work to explore the relation between distributional robustness and heavy-tails \cite{das2021heavy}.


\section*{Acknowledgment}
{A.R is supported by the a Marie Sklodowska-Curie
Fellowship (project NN-OVEROPT 101030817). M.G.'s research is supported in part by the grants Office of Naval Research Award Number N00014-21-1-2244, National Science Foundation (NSF) CCF-1814888 and NSF DMS-2053485. U.\c{S}.'s research is supported by the French government under management of Agence Nationale de la Recherche as part of the ``Investissements d'avenir'' program, reference ANR-19-P3IA-0001 (PRAIRIE 3IA Institute) and the European Research Council Starting Grant DYNASTY – 101039676.
L.Z. is grateful to the support from a Simons Foundation Collaboration Grant and the grants NSF DMS-2053454, NSF DMS-2208303 from the National Science Foundation.}

\bibliographystyle{plainnat}
\bibliography{heavy,levy,opt-ml}
\normalsize
\newpage
\onecolumn
\appendix
\begin{center}
\Large \bf Algorithmic Stability of Heavy-Tailed Stochastic Gradient Descent on Least Squares \vspace{3pt}\\ {\normalsize Appendix}
\end{center}
The Appendix is organized as follows:
\begin{itemize}
    \item In Section~\ref{sec:tech:background:OU}, we provide the background details about characterizing the stationary distribution of a L\'{e}vy-driven OU process. 
    \item  In Section~\ref{sec:finite:time:continuous}, we characterize the finite-time distribution of a L\'{e}vy-driven OU process.
    \item  In Section~\ref{discrete:least:square}, we characterize the distributions of a discrete-time L\'{e}vy-driven OU process. 
    \item In Section~\ref{ap:proof_1d}, we provide the proofs for the $1$-dimensional case. 
    \item In Section~\ref{ap:proof_d_dim}, we prove the results for least square in $d$-dimension.
    \item  In Section~\ref{ap:discretized_SDE}, we provide theory for the discretized SDE. 
    \item In Section~\ref{ap:general_sigma}, we extend the $d$-dimensional result for general preconditioner PSD $\Sigma$.
    \item In Section~\ref{ap:useful_results}, we discuss useful results which we utilize in proving our results for $d$-dimensional case. 
    \item In Section~\ref{ap:further_expts}, we provide further details about our experimental setup. 
\end{itemize}

\section{Characterizing the Stationary Distribution of a L\'{e}vy-Driven OU Process}\label{sec:tech:background:OU}

In this section, we review the technical background
of characterizing the stationary distribution of an Ornstein-Uhlenbeck process
driven by a general L\'{e}vy process.
Consider an Ornstein-Uhlenbeck process
driven by a general L\'{e}vy process
\begin{equation}\label{eqn:general:2}
\rmd \theta_{t}=-A\theta_{t}\rmd t+\rmd \mathrm{Z}_{t},
\end{equation}
where $\mathrm{Z}_{t}$ is a general L\'{e}vy process. 
One particular example is $\mathrm{Z}_{t}=\Sigma \Lm_{t}$
so that
\begin{equation}\label{eqn:2:2}
\rmd \theta_{t}=-A\theta_{t}\rmd t+\Sigma\rmd\Lm_{t}.
\end{equation}
Under some regularity conditions on $A$ and the L\'{e}vy measure of $Z$,
$\theta_{t}$ in \eqref{eqn:general:2} admits a unique invariant distribution $\pi$,
and the class of all possible $\pi$'s forms the class of all $A$-self-decomposable
distributions; see \cite{Masuda2004} and the references therein.

Let $d\in\mathbb{N}$ and $\mathrm{Z}_{t}$ is a $d$-dimensional L\'{e}vy process such that
$Z_{0}=0$ a.s. and $\mathrm{Z}_{t}$ admits the generating triplet $(b,C,\nu)$,
that is, $b\in\mathbb{R}^{d}$, $C$ is a $d\times d$ symmetric non-negative definite
matrix and $\nu$ is a $\sigma$-finite measure on $\mathbb{R}^{d}$ satisfying $\nu(\{0\})=0$
and $\int_{\mathbb{R}^{d}}\min(1,\Vert z\Vert^{2})\nu(dz)<\infty$
for which $\mathrm{Z}_{t}$ has the characteristic function
\begin{equation}
\varphi_{t}(u):=\mathbb{E}\left[e^{\imagi \langle u,\mathrm{Z}_{t}\rangle}\right]
=\exp\left(t\left\{\imagi u^{\top}b-\frac{1}{2}u^{\top}Cu+\int_{\mathbb{R}^{d}}
\left(e^{\imagi u^{\top}z}-1-\imagi u^{\top}z1_{\Vert z\Vert\leq 1}\right)\nu(dz)\right\}\right)
\end{equation}
for any $u\in\mathbb{R}^{d}$ and $t>0$.

\begin{lemma}[Theorems 4.1. and 4.2. in \cite{Sato1984}]\label{lem:general}
Assume that $A$ is a $d\times d$ matrix such that the real parts
of all its eigenvalues are positive. 
Moreover, assume that 
\begin{equation*}
\int_{\Vert z\Vert>1}\log\Vert z\Vert\nu(\rmd z)<\infty.
\end{equation*}
Then, $\theta_{t}$ in \eqref{eqn:general:2}
admits a unique invariant distribution $\pi$ whose characteristic function is given by
\begin{equation}
\int_{\mathbb{R}^{d}}e^{\imagi \langle u,x\rangle}\pi(\rmd x)=\exp\left(\int_{0}^{\infty}\log\varphi_{1}\left(e^{-sA^{\top}}u\right)\rmd s\right),
\end{equation}
for any $u\in\mathbb{R}^{d}$. In particular, the generating triplet of the limiting
distribution is $(b_{\infty},C_{\infty},\nu_{\infty})$, where
\begin{align}
&b_{\infty}=A^{-1}b+\int_{\mathbb{R}^{d}}\int_{0}^{\infty}e^{-sA}z\left(1_{\Vert e^{-sA}z\Vert\leq 1}-1_{\Vert z\Vert\leq 1}\right)\rmd s\nu(\rmd z),\label{b:eqn}
\\
&C_{\infty}=\int_{0}^{\infty}e^{-sA}Ce^{-sA^{\top}}\rmd s,\label{C:eqn}
\\
&\nu_{\infty}(E)=\int_{0}^{\infty}\nu\left(e^{sA}E\right)\rmd s,\qquad\text{for any $E\in\mathcal{B}(\mathbb{R}^{d})$}.\label{nu:eqn}
\end{align}
\end{lemma}

By using Lemma~\ref{lem:general}, we can easily obtain the following result.

\begin{lemma*}[Restatement of Lemma~\ref{lem:char}]
Assume that $A$ is a real symmetric matrix with all the eigenvalues being positive.
Then \eqref{eqn:2:2} admits a unique stationary distribution
\begin{equation}\label{compute:2}
\int_{\mathbb{R}^{d}}e^{\imagi \langle u,x\rangle}\pi(\rmd x)
=\exp\left(-\int_{0}^{\infty}\left\Vert\Sigma^{\top}e^{-sA}u\right\Vert_{2}^{\alpha}\rmd s\right),
\qquad\text{for any $u\in\mathbb{R}^{d}$}.
\end{equation}
\end{lemma*}
\begin{proof}
In our \eqref{eqn:2:2}, $A$ is a real symmetric matrix
with positive eigenvalues and $\mathrm{Z}_{t}=\Sigma \Lm_{t}$,
where $\Lm_{t}$ is a rotationally symmetric $\alpha$-stable L\'{e}vy process
and moreover, $\alpha$-stable L\'{e}vy measure satisfies
the condition $\int_{\Vert z\Vert>1}\log\Vert z\Vert\nu(dz)<\infty$,
so that the condition in Lemma~\ref{lem:general} is satisfied,
and we conclude that \eqref{eqn:2:2} admits
a unique stationary distribution say $\pi$.

Moreover, in our case, the characteristic function of $\mathrm{Z}_{t}=\Sigma \Lm_{t}$ is given by
\begin{equation}
\varphi_{1}(u)=\mathbb{E}\left[e^{\imagi \left\langle u,\Sigma \Lm_{1}\right\rangle}\right]
=\mathbb{E}\left[e^{\imagi \left\langle\Sigma^{\top}u,\Lm_{1}\right\rangle}\right]
e^{-\Vert\Sigma^{\top}u\Vert_{2}^{\alpha}}.
\end{equation}
Therefore, by Lemma~\ref{lem:general}, the unique invariant distribution $\pi$ of \eqref{eqn:2:2} has the following characteristic function:
\begin{equation}
\int_{\mathbb{R}^{d}}e^{\imagi \langle u,x\rangle}\pi(\rmd x)
=\exp\left(-\int_{0}^{\infty}\left\Vert\Sigma^{\top}e^{-sA}u\right\Vert_{2}^{\alpha}\rmd s\right).
\end{equation}
This completes the proof.
\end{proof}

\begin{remark}
(i) We can also characterize the generating triplet 
for the limiting distribution $\pi$ in the above lemma according to Lemma~\ref{lem:general}.
In our case, $b=0$ in \eqref{b:eqn}, $C=0$ in \eqref{C:eqn}, and $\nu$ is the L\'{e}vy measure
for $\Sigma \Lm_{t}$ in \eqref{nu:eqn}.

(ii) In general, it is not possible to further
simplify the expression \eqref{compute:2} except for some special cases.
For example, when $A=I$. 
For any $s\geq 0$, $e^{-sI}=\sum_{k=0}^{\infty}\frac{(-sI)^{k}}{k!}=\sum_{k=0}^{\infty}\frac{(-s)^{k}}{k!}I=e^{-s}I$. 
Therefore, when $A=I$, we can compute from \eqref{compute:2} that
\begin{equation}
\int_{\mathbb{R}^{d}}e^{\imagi \langle u,x\rangle}\pi(\rmd x)
=\exp\left(-\int_{0}^{\infty}e^{-s\alpha}\left\Vert\Sigma^{\top}u\right\Vert_{2}^{\alpha}\rmd s\right)
=e^{-\frac{1}{\alpha}\left\Vert\Sigma^{\top}u\right\Vert_{2}^{\alpha}}.
\end{equation}
Hence, in this special case, the limiting distribution is $\alpha^{\frac{-1}{\alpha}}\Sigma\Lm_{1}$. 
\end{remark}

\section{Characterizing the Finite-Time Distribution of a L\'{e}vy-Driven OU Process}\label{sec:finite:time:continuous}

In this section, we derive the characteristic function for the finite-time distribution
of a L\'{e}vy-driven OU process. 
We recall from equation~\eqref{eqn:2:2}
\begin{equation}\label{eqn:2:2:1}
\rmd \theta_{t}=-A\theta_{t}\rmd t+\Sigma\rmd\Lm_{t}.
\end{equation}
We have the following technical lemma that computes
the characteristic function of $\theta_{t}$ at any finite time $t>0$.

\begin{lemma}\label{lem:char:finite:time}
For any $t>0$ and $u\in\mathbb{R}^{d}$, we have 
\begin{align*}
\mathbb{E}\left[e^{\imagi u^{\top}\theta_{t}}\right]
=e^{\imagi u^{\top}e^{-At}\theta_{0}}e^{-\int_{0}^{t}\left\Vert\Sigma^{\top}e^{-sA}u\right\Vert_{2}^{\alpha}\rmd s}.
\end{align*}
\end{lemma}

\begin{proof}
We can solve the L\'{e}vy-driven SDE \eqref{eqn:2:2:1} and obtain
\begin{equation}
\theta_{t}=e^{-At}\theta_{0}+\int_{0}^{t}e^{-A(t-s)}\Sigma\rmd\Lm_{s},
\end{equation}
such that
for any $u\in\mathbb{R}^{d}$, we have 
\begin{align*}
\mathbb{E}\left[e^{\imagi u^{\top}\theta_{t}}\right]
&=e^{\imagi u^{\top}e^{-At}\theta_{0}}\mathbb{E}\left[e^{\imagi\int_{0}^{t}u^{\top}e^{-A(t-s)}\Sigma\rmd\Lm_{s}}\right]
\nonumber
\\
&=e^{\imagi u^{\top}e^{-At}\theta_{0}}e^{-\int_{0}^{t}\left\Vert\Sigma^{\top}e^{-(t-s)A}u\right\Vert_{2}^{\alpha}\rmd s}
=e^{\imagi u^{\top}e^{-At}\theta_{0}}e^{-\int_{0}^{t}\left\Vert\Sigma^{\top}e^{-sA}u\right\Vert_{2}^{\alpha}\rmd s}.
\end{align*}
This completes the proof.
\end{proof}

We recall from \eqref{eq:sde_invar_1}-\eqref{eq:sde_invar_2} that
\begin{align}
    \rmd \theta_t &= - \frac{1}{n} \left(X^\top X \theta_t - X^\top y \right)\rmd t + \Sigma \rmd \Lm_t,  \\
    \rmd \hat{\theta}_t &= -\frac{1}{n} \left(\hat{X}^\top \hat{X} \hat{\theta}_t - \hat{X}^\top \hat{y} \right) \rmd t + \Sigma \rmd \Lm_t.
\end{align}
For the sake of simplicity, 
we take $y=\hat{y}=0$ and $\theta_{0}=\hat{\theta}_{0}=0$. 
We denote $\psi_{\theta}(t,u):=\mathbb{E}\left[e^{\langle\imagi u,\theta_{t}\rangle}\right]$ and $\psi_{\hat{\theta}}(t,u):=\mathbb{E}\left[e^{\langle\imagi u,\hat{\theta}_{t}\rangle}\right]$.
Hence, we obtain from Lemma~\ref{lem:char:finite:time} that
\begin{align}
    \psi_{\theta}(t,u) &= \exp\left( - \int_{0}^{t} \left\|\Sigma^\top e^{-s \frac{1}{n}({X}^\top {X})}u \right\|_2^\alpha \rmd s\right), \\
    \psi_{\hat{\theta}}(t,u) &= \exp\left( - \int_{0}^{t} \left\|\Sigma^\top e^{-s \frac{1}{n}(\hat{X}^\top \hat{X})}u \right\|_2^\alpha \rmd s\right). 
\end{align}



\section{Heavy-Tailed Discretized SDE on Least Squares Regression}\label{discrete:least:square}

{ In this section, we introduce heavy-tailed discretized SGD for the least square regression. 

In the context of algorithmic stability, we assume that we have two training datasets $(X,y)$ and $(\hat{X},\hat{y})$ that differ in only one data point. Without loss of generality, we have 
\begin{equation*}
\hat{X}=\left[x_1^\top , x_2^\top,\dots, \tilde{x}_i^\top,\dots,x_n^\top\right] \in \mathbb{R}^{n\times d}
\end{equation*}
and 
\begin{equation*}
\hat{y} = [y_1, y_2, \dots, \tilde{y}_i, \dots, y_n]\in\mathbb{R}^{n}.
\end{equation*}
We consider the 
following discretized heavy-tailed SDE
for the ERM problem as defined in \eqref{eqn:risks}:
\begin{align}
    \theta_{k+1} &= \theta_{k}- \frac{\eta}{n} \left(X^\top X \theta_{k} - X^\top y \right) +\eta^{1/\alpha} \Sigma S_{k+1}, \label{eq:sde_invar_1:discrete}  \\
    \hat{\theta}_{k+1} &=\hat{\theta}_{k} -\frac{\eta}{n} \left(\hat{X}^\top \hat{X} \hat{\theta}_{k} - \hat{X}^\top \hat{y} \right) + \eta^{1/\alpha}\Sigma S_{k+1}, \label{eq:sde_invar_2:discrete}
\end{align}
where $\eta>0$ is the stepsize and $\Sigma \in \mathbb{R}^{d\times d}$ is a real-valued matrix and 
$S_{k}$ are i.i.d. alpha-stable random vectors with the characteristic function:
\begin{equation}
\mathbb{E}\left[e^{\imagi\langle u,S_{k}\rangle}\right]=e^{-\Vert u\Vert_{2}^{\alpha}},
\qquad\text{for any $u\in\mathbb{R}^{d}$}.
\end{equation}

We denote $\psi_{\theta}(k,u):=\mathbb{E}\left[e^{\langle\imagi u,\theta_{k}\rangle}\right]$ and $\psi_{\hat{\theta}}(k,u):=\mathbb{E}\left[e^{\langle\imagi u,\hat{\theta}_{k}\rangle}\right]$.
For the sake of simplicity, we take $y=\hat{y}=0$
and $\theta_{0}=\hat{\theta}_{0}=0$.
By Lemma~\ref{lem:discrete:time}, we obtain 
\begin{align}
    \psi_{\theta}(k,u) &= \exp\left( - \eta\sum_{j=0}^{k-1}\left\|\Sigma^\top\left(I-\frac{\eta}{n}\left({X}^\top {X}\right)\right)^{j}u \right\|_2^\alpha \right), \label{eq:discrete_char_1}\\
    \psi_{\hat{\theta}}(k,u) &= \exp\left( - \eta\sum_{j=0}^{k-1} \left\|\Sigma^\top\left(I-\frac{\eta}{n}\left(\hat{X}^\top \hat{X}\right)\right)^{j}u \right\|_2^\alpha \right), \label{eq:discrete_char_2}
\end{align}
which will be the key ingredients to obtain the algorithmic stability results.
}


{ 
Here below, we derive the characteristic function for distributions
of a discrete-time L\'{e}vy-driven OU process:
\begin{equation}\label{discrete:SDE}
\theta_{k+1}=\theta_{k}-\eta A\theta_{k}+\eta^{1/\alpha}\Sigma S_{k+1},
\end{equation}
where $A$ is a real symmetric matrix and $S_{k}$ are i.i.d. alpha-stable random vectors with the characteristic function:
\begin{equation}
\mathbb{E}\left[e^{\imagi\langle u,S_{k}\rangle}\right]=e^{-\Vert u\Vert_{2}^{\alpha}},
\qquad\text{for any $u\in\mathbb{R}^{d}$}.
\end{equation}
We can compute the characteristic function
of the finite-time distribution of the discrete-time L\'{e}vy-driven OU process \eqref{discrete:SDE}
as follows. 

\begin{lemma}\label{lem:discrete:time}
Assume that $A$ is a real symmetric matrix.
For any $k\in\mathbb{N}$ and for any $u\in\mathbb{R}^{d}$, 
\begin{equation}
\mathbb{E}\left[e^{\imagi\langle u,\theta_{k}\rangle}\right]
=e^{\imagi\langle(I-\eta A)^{k}u,\theta_{0}\rangle}
e^{-\eta\sum_{j=0}^{k-1}\left\Vert\Sigma^{\top}(I-\eta A)^{j}u\right\Vert_{2}^{\alpha}}.
\end{equation}
\end{lemma}

\begin{proof}
We can compute from \eqref{discrete:SDE} that 
for any $u\in\mathbb{R}^{d}$, 
\begin{align}
\mathbb{E}\left[e^{\imagi\langle u,\theta_{k+1}\rangle}\right]
&=\mathbb{E}\left[e^{\imagi\langle u,(I-\eta A)\theta_{k}\rangle}\right]
\mathbb{E}\left[e^{\imagi\langle u,\eta^{1/\alpha}\Sigma S_{k}\rangle}\right]
\nonumber
\\
&=\mathbb{E}\left[e^{\imagi\langle(I-\eta A)u,\theta_{k}\rangle}\right]
\mathbb{E}\left[e^{\imagi\langle\eta^{1/\alpha}\Sigma^{\top}u,S_{k}\rangle}\right]
=\mathbb{E}\left[e^{\imagi\langle(I-\eta A)u,\theta_{k}\rangle}\right]
e^{-\eta\Vert\Sigma^{\top}u\Vert_{2}^{\alpha}},
\end{align}
and we can further compute that
\begin{equation}
\mathbb{E}\left[e^{\imagi\langle(I-\eta A)u,\theta_{k}\rangle}\right]
=\mathbb{E}\left[e^{\imagi\langle(I-\eta A)^{2}u,\theta_{k-1}\rangle}\right]
e^{-\eta\Vert\Sigma^{\top}(I-\eta A)u\Vert_{2}^{\alpha}}.
\end{equation}
Hence, iteratively, we obtain
\begin{equation}
\mathbb{E}\left[e^{\imagi\langle u,\theta_{k}\rangle}\right]
=e^{\imagi\langle(I-\eta A)^{k}u,\theta_{0}\rangle}
e^{-\eta\sum_{j=0}^{k-1}\left\Vert\Sigma^{\top}(I-\eta A)^{j}u\right\Vert_{2}^{\alpha}}.
\end{equation}
This completes the proof.
\end{proof}

We can derive from Lemma~\ref{lem:discrete:time} the characteristic function
of the stationary distribution of the discrete-time L\'{e}vy-driven OU process \eqref{discrete:SDE}
as follows.

\begin{corollary}\label{cor:discrete:SDE}
Assume that $A$ is a real symmetric matrix
with all the eigenvalues being positive and less than $1/\eta$.
Then, for any $u\in\mathbb{R}^{d}$, 
\begin{equation}
\mathbb{E}\left[e^{\imagi\langle u,\theta_{\infty}\rangle}\right]
=e^{-\eta\sum_{j=0}^{\infty}\Vert\Sigma^{\top}(I-\eta A)^{j}u\Vert_{2}^{\alpha}}.
\end{equation}
\end{corollary}

\begin{proof}
When $A$ is a real symmetric matrix
with all the eigenvalues being positive and less than $1/\eta$, we have $\Vert I-\eta A\Vert<1$
and it follows that
\begin{equation}
\left|\left\langle(I-\eta A)^{k}u,\theta_{0}\right\rangle\right|
\leq\Vert I-\eta A\Vert^{k}\cdot\Vert u\Vert\cdot\Vert\theta_{0}\Vert\rightarrow 0, 
\end{equation}
as $k\rightarrow\infty$, and moreover, 
\begin{equation}
\left\Vert\Sigma^{\top}(I-\eta A)^{j}u\right\Vert_{2}^{\alpha}
\leq
\left\Vert\Sigma^{\top}\right\Vert\cdot\Vert I-\eta A\Vert^{j}\cdot\Vert u\Vert,
\end{equation}
which is summable over $j$ and hence the result follows from Lemma~\ref{lem:discrete:time}.
The proof is complete.
\end{proof}
}



\section{Proofs for the 1-Dimensional Case} \label{ap:proof_1d}

In this section, we provide the proofs for the one-dimensional case.
In the next lemma, we first bound the difference between the characteristic functions of the stationary distributions. 

\begin{lemma}
\label{lem:1d_char_func}
For two matrices $X \in \mathbb{R}^{n}$ and $\hat{X} \in \mathbb{R}^n$ as defined earlier, the absolute value of difference between the characteristic function for stationary distribution at any $u \in \mathbb{R}$ corresponding to one-dimensional rotation invariant  processes in equations~\eqref{eq:char_sde_1} and \eqref{eq:char_sde_2} ($d=1$) is bounded as 
\begin{align*}
    \left|\psi_{\theta}(u) - \psi_{\hat{\theta}}(u) \right| &\leq \frac{x_i^2 - \tilde{x}_i^2}{\|\hat{X}\|_2^2} \left(|u|^\alpha \frac{n}{\alpha \|X\|_2^2} \exp\left(- |u|^\alpha \frac{n}{\alpha \|X\|_2^2} \right)\right).
\end{align*}
\end{lemma}

\begin{proof}
We can compute that
\begin{align}
    &\left|\psi_\theta(u) - \psi_{\hat{\theta}}(u) \right|\nonumber
    \\
    &= \left|\exp\left(- \int_{0}^{\infty} \left|e^{-s\frac{1}{n}\|X\|_2^2 }u\right|^{\alpha} \rmd s\right) - \exp\left(- \int_{0}^{\infty} \left|e^{-s\frac{1}{n}\|\hat{X}\|_2^2 }u\right|^{\alpha} \rmd s\right) \right| \notag \\
    &= \left|\exp\left(- \int_{0}^{\infty} \left|e^{-s\frac{1}{n}\|X\|_2^2 }u\right|^{\alpha} \rmd s\right)\left(1 - \exp\left( \int_{0}^{\infty} \left|e^{-s\frac{1}{n}\|{X}\|_2^2 }u\right|^{\alpha} \rmd s \right. \right. \right. \notag \\
    & \qquad \qquad \qquad \qquad \qquad \qquad  \qquad \qquad  \left.\left.\left.- \int_{0}^{\infty} \left|e^{-s\frac{1}{n}\|\hat{X}\|_2^2 }u\right|^{\alpha} \rmd s\right)  \right) \right| \label{eq:mid_1d_char}\\
    &\leq \exp\left(- \int_{0}^{\infty} \left|e^{-s\frac{1}{n}\|X\|_2^2 }u\right|^{\alpha} \rmd s\right) \left| \int_{0}^{\infty} \left|e^{-s\frac{1}{n}\|X\|_2^2 }u\right|^{\alpha} \rmd s - \int_{0}^{\infty} \left|e^{-s\frac{1}{n}\|\hat{X}\|_2^2 }u\right|^{\alpha} \rmd s \right| \notag\\
    &= \exp\left(- |u|^\alpha \int_{0}^{\infty} e^{-s\frac{\alpha}{n}\|X\|_2^2 } \rmd s\right) |u|^\alpha \left| \int_{0}^{\infty} e^{-s\frac{\alpha}{n}\|X\|_2^2 } \rmd s - \int_{0}^{\infty} e^{-s\frac{\alpha}{n}\|\hat{X}\|_2^2 } \rmd s \right| \notag\\
    &= |u|^\alpha \exp\left(- |u|^\alpha \int_{0}^{\infty} e^{-s\frac{\alpha}{n}\|X\|_2^2 } \rmd s\right)  \left| \frac{n}{\alpha \|X\|_2^2} - \frac{n}{\alpha \|\hat{X}\|_2^2}\right| \notag\\
    &= |u|^\alpha \exp\left(- |u|^\alpha \frac{n}{\alpha \|X\|_2^2} \right)  \left| \frac{n}{\alpha \|X\|_2^2} - \frac{n}{\alpha \|\hat{X}\|_2^2}\right| \notag\\
    &= |u|^\alpha \frac{n|x_i^2 - \tilde{x}_i^2|}{\alpha \|X\|_2^2\|\hat{X}\|_2^2} \exp\left(- |u|^\alpha \frac{n}{\alpha \|X\|_2^2} \right)\notag\\
    &= \frac{|x_i^2 - \tilde{x}_i^2|}{\|\hat{X}\|_2^2} \left(|u|^\alpha \frac{n}{\alpha \|X\|_2^2} \exp\left(- |u|^\alpha \frac{n}{\alpha \|X\|_2^2} \right)\right), \notag
\end{align}
which completes the proof.
\end{proof}

\begin{theorem*}[Restatement of Theorem~\ref{thm:1d_main}]
Consider the one-dimensional loss function $f(x)  =|\theta x|^p$.  For any $x\sim P_X$, if we have $|x| > R $ with probability $\delta_1$ and for any $X$ sampled uniformly at random from the set $\mathcal{X}_n$, if we have $\|X\|_2^2 \leq \sigma^2 n$  with probability $\delta_2$. Then,
\begin{enumerate}
    \item[(i)] For $\alpha \in [1,2)$, the algorithm is not stable when $p \in [\alpha,2]$ i.e. $\varepsilon_{\text{stab}(\mathcal{A}_{\text{cont}})}$ diverges. When $\alpha = p =2$ then $ \varepsilon_{\text{stab}} (\mathcal{A}_{\text{cont}}) \leq \frac{R^4}{\pi \sigma^4 n}$ with probability  at least $1- \delta_1 -2\delta_2$.
    \item[(ii)] For $p\in [1,\alpha)$, we have the following upper bound for the algorithmic stability,
    \begin{align*}
         \varepsilon_{\text{stab}} (\mathcal{A}_{\text{cont}})
    &\leq    \frac{2R^{p+2}}{\pi\sigma^2 n }\Gamma(p+1) \cos\left( \frac{(p-1)\pi}{2}\right) \frac{1}{\alpha} \left( \frac{1}{\alpha \sigma^2 }\right)^{\frac{p}{\alpha}} \Gamma \left( 1 - \frac{p}{\alpha}\right)=: c(\alpha),
    \end{align*}
    which holds with probability at least $1- \delta_1-2\delta_2$. Furthermore, for some $\alpha_0 > 1$, if we have 
    \begin{align}
        \sigma^2 \geq \exp\left( 1 + \frac{2}{p} - \log \alpha_0 -  \phi\left( 1 - \frac{p}{\alpha_0}\right) \right),
    \end{align}
    where $\phi$ is the digamma function, then the map $\alpha\mapsto c(\alpha)$ is increasing for $\alpha \in [\alpha_0,2)$.
    \item[(iii)]  The stability bound is tight in $\alpha$.  
\end{enumerate}
\end{theorem*}

\begin{proof}
Closed-form expression for the Fourier transform of function $f(\theta) = |\theta x|^p$ has been given in \citet{gelfand1969generalized}. However, we provide here the result for the sake of completeness.  Let us compute the Fourier transform of the function $\theta \to |\theta x|^p$ for $p \in [1,2)$.
\begin{align}
    \int_{-\infty}^{\infty} |\theta x|^p e^{-\imagi  u\theta} \rmd \theta =&  |x|^p \int_{-\infty}^{\infty} |\theta|^p e^{- \imagi u\theta}~\rmd \theta \notag \\
    =& |x|^p \int_{0}^{\infty} (e^{\imagi u\theta} + e^{-\imagi u\theta}) \theta^p ~\rmd \theta  \notag  \\
    =&  |x|^p \left[ \int_{0}^{\infty} e^{\imagi u\theta} \theta^p ~\rmd \theta  + \int_{0}^{\infty} e^{-\imagi u\theta} \theta^p ~\rmd \theta \right] \notag \\
    =& 2|x|^p  \Gamma(p+1) \cos\left( \frac{(p+1)\pi}{2}\right) \frac{1}{|u|^{p+1}}. \label{eq:1d_fourier}
\end{align}
From \citet{gelfand1969generalized}, 
\begin{align*}
    \int_{-\infty}^{\infty} |\theta x|^2 e^{-\imagi  u\theta} \rmd \theta = \frac{-2x^2}{u^2} \delta(u),
\end{align*}
where $\delta(u)$ is the Dirac-delta function. 

First, we get the result for $p\in [1,2)$. We utilize the result from Lemma~\ref{lem:1d_char_func} and equation~\eqref{eq:1d_fourier} to get, 
\begin{align*}
    \varepsilon_{\text{stab}} (\mathcal{A}_{\text{cont}})
    &= \sup_{X\cong \hat{X}}\sup_{x \in \mathcal{X}}  \left[\frac{1}{(2\pi)} \int_{\mathbb{R}} |\psi_{\theta}(u) - \psi_{\hat{\theta}}(u)| \left|\int_{-\infty}^{\infty}|\theta x|^p e^{-\imagi  u \theta}~\rmd \theta \right|~\rmd u \right] \\
    &= \sup_{X\cong \hat{X}}\sup_{x \in \mathcal{X}}  \frac{|x|^p}{\pi }\Gamma(p+1) \cos\left( \frac{(p-1)\pi}{2}\right)\frac{|x_i^2 - \tilde{x}_i^2|}{\|\hat{X}\|_2^2}  \frac{n}{\alpha \|X\|_2^2} \\
    & \qquad \qquad \qquad \qquad \qquad \qquad \qquad  \cdot \int_{\mathbb{R}}  \left(|u|^\alpha \exp\left(- |u|^\alpha \frac{n}{\alpha \|X\|_2^2} \right)\right)   \frac{1}{|u|^{p+1}}  ~\rmd u  \\
    &\leq \sup_{X\cong \hat{X}}\sup_{x \in \mathcal{X}} \frac{2|x|^{p+2}}{\pi\|\hat{X}\|_2^2}\Gamma(p+1) \cos\left( \frac{(p-1)\pi}{2}\right)  \\
    & \qquad \qquad \qquad   \qquad \qquad \qquad \cdot\int_{0}^{\infty} u^{\alpha - p -1} \frac{n}{\alpha \|X\|_2^2} \exp\left(- |u|^\alpha \frac{n}{\alpha \|X\|_2^2} \right) ~\rmd u.
\end{align*}
In the above integral, by substituting $u^\alpha \frac{n}{\alpha \|X\|^2}$ with $t$ so that
\begin{align}
    \rmd t =  u^{\alpha-1} \frac{n}{\|X\|_2^2}\rmd u,~ \text{and } \frac{1}{u^p} = \left( \frac{n}{\alpha \| X\|_2^2}\right)^{\frac{p}{\alpha}} t^{-p/\alpha},
\end{align}
we have,
\begin{align}
   \varepsilon_{\text{stab}} (\mathcal{A}_{\text{cont}})
    &\leq  \sup_{X\cong \hat{X}}\sup_{x \in \mathcal{X}} \frac{2|x|^{p+2}}{\pi\alpha \|\hat{X}\|_2^2}\Gamma(p+1) \cos\left( \frac{(p-1)\pi}{2}\right)  \left( \frac{n}{\alpha \| X\|_2^2}\right)^{\frac{p}{\alpha}} \int_{0}^{\infty}t^{-p/\alpha} e^{-t}~\rmd t. \label{eq:p_1_2}
\end{align}
It is clear that, the above integral diverges for $p\geq \alpha$, hence the algorithm is not stable for $p \in [1,2)$. Now, we check the case for $p=2$. For $p=2$, we have,
\begin{align*}
    \varepsilon_{\text{stab}} (\mathcal{A}_{\text{cont}})&\leq \sup_{X\cong \hat{X}}\sup_{x \in \mathcal{X}} \frac{|x|^{4}}{\pi\|\hat{X}\|_2^2} \frac{n}{  \|X\|_2^2} \int_{0}^{\infty} u^{\alpha-2}\exp\left( - u^{\alpha} \frac{n }{\|X\|_2^2}\right) \delta(u) ~\rmd u.
\end{align*}
The above integral clearly diverges for $\alpha < 2$. However, when $\alpha =2$, then
\begin{align*}
    \varepsilon_{\text{stab}} (\mathcal{A}_{\text{cont}})&\leq \sup_{X\cong \hat{X}}\sup_{x \in \mathcal{X}} \frac{|x|^{4}}{\pi\|\hat{X}\|_2^2} \frac{n}{  \|X\|_2^2}. 
\end{align*}
If we have $|x| > R $ for any $x\sim P_X$ with probability $\delta_1$ and $\|X\|_2^2 \leq \sigma^2 n$ for any $X$ sampled uniformly from the set $\mathcal{X}_n$ with probability $\delta_2$, then for $\alpha=2$ and $p=2$,
\begin{align}
    \varepsilon_{\text{stab}} (\mathcal{A}_{\text{cont}}) &\leq \frac{R^4}{\pi \sigma^4 n},
\end{align}
with probability at least $1- \delta_1 -2\delta_2$.
This proves the part (i) of our result. 

Next, let us prove the part (ii). 
From equation~\eqref{eq:p_1_2}, for $p < \alpha$, we have
\begin{align*}
    \varepsilon_{\text{stab}} (\mathcal{A}_{\text{cont}})
    &\leq  \sup_{X\cong \hat{X}}\sup_{x \in \mathcal{X}} \frac{2|x|^{p+2}}{\pi\alpha \|\hat{X}\|_2^2}\Gamma(p+1) \cos\left( \frac{(p-1)\pi}{2}\right)  \left( \frac{n}{\alpha \| X\|_2^2}\right)^{\frac{p}{\alpha}} \int_{0}^{\infty}t^{-p/\alpha} e^{-t}~\rmd t \\
    &=  \sup_{X\cong \hat{X}}\sup_{x \in \mathcal{X}} \frac{2|x|^{p+2}}{\pi \alpha \|\hat{X}\|_2^2}\Gamma(p+1) \cos\left( \frac{(p-1)\pi}{2}\right)  \left( \frac{n}{\alpha \| X\|_2^2}\right)^{\frac{p}{\alpha}} \Gamma \left( 1 - \frac{p}{\alpha}\right).
\end{align*}
If we have $|x| > R $ for any $x\sim P_X$ with probability $\delta_1$ and $\|X\|_2^2 \leq \sigma^2 n$ for any $X$ sampled uniformly at random from the set $\mathcal{X}_n$ with probability $\delta_2$, then with probability at least $1- \delta_1 - 2\delta_2$, the following holds:
\begin{align*}
  \varepsilon_{\text{stab}} (\mathcal{A}_{\text{cont}})
    &\leq    \frac{2R^{p+2}}{\pi\sigma^2 n }\Gamma(p+1) \cos\left( \frac{(p-1)\pi}{2}\right) \frac{1}{\alpha} \left( \frac{1}{\alpha \sigma^2 }\right)^{\frac{p}{\alpha}} \Gamma \left( 1 - \frac{p}{\alpha}\right)= c(\alpha).
\end{align*}
Now, consider the function,
\begin{align*}
    \Lambda(\alpha ) =  \frac{1}{\alpha} \left( \frac{1}{\alpha \sigma^2 }\right)^{\frac{p}{\alpha}} \Gamma \left( 1 - \frac{p}{\alpha}\right).
\end{align*}
We can compute that
\begin{align*}
    \partial_{\alpha} \log \Lambda(\alpha) = \frac{p}{\alpha^2} \left[ \log \alpha  + \log \sigma^2 -1 - \frac{\alpha}{p} +  \phi\left( 1 - \frac{p}{\alpha}\right) \right], 
\end{align*}
where $\phi$ is the digamma function. For any arbitrary $\alpha_0$, if we choose $$\sigma^2 \geq \exp\left( 1 + \frac{2}{p} - \log \alpha_0 -  \phi\left( 1 - \frac{p}{\alpha_0}\right) \right),$$ then $\partial_{\alpha} \log \Lambda(\alpha) >0 $ for  $\alpha \in [\alpha_0,2)$. Hence, for all $\alpha_1,\alpha_2 \in [\alpha_0,2)$,  $\alpha_1 < \alpha_2 \Rightarrow \Lambda(\alpha_1) \leq \Lambda(\alpha_2)$. This proves that $c(\alpha)$ is an increasing map in $\alpha$. \\

(iii). Now we show that the bound on the stability is tight for some appropriately chosen $P_X$. Note that we have,
\begin{align*}
    \varepsilon_{\text{stab}} (\mathcal{A}_{\text{cont}})
    &= \sup_{X\cong \hat{X}}\sup_{x \in \mathcal{X}}  \left[\frac{1}{(2\pi)} \int_{\mathbb{R}} \left|\psi_{\theta}(u) - \psi_{\hat{\theta}}(u)\right| \left|\int_{-\infty}^{\infty}|\theta x|^p e^{-\imagi  u \theta}~\rmd \theta \right|~\rmd u \right] \\
    &=\sup_{X\cong \hat{X}}\sup_{x \in \mathcal{X}}  \underbrace{\left[\frac{1}{\pi} \int_{0}^\infty \left|\psi_{\theta}(u) - \psi_{\hat{\theta}}(u)\right| \left|\int_{-\infty}^{\infty}|\theta x|^p e^{-\imagi  u \theta}~\rmd \theta \right|~\rmd u \right]}_{:=\Phi_{x,X,\hat{X}}(\mathcal{A}_{\text{cont}})}.
\end{align*}

Hence, let us consider $u\geq 0$.
From equation~\eqref{eq:mid_1d_char}, we have
\begin{align*}
    \left|\psi_{\theta}(u) - \psi_{\hat{\theta}}(u)\right| &= \Bigg|\exp\left(- \int_{0}^{\infty} \left|e^{-s\frac{1}{n}\|X\|_2^2 }u\right|^{\alpha} \rmd s\right)
    \\
    &\qquad \qquad \cdot
    \left(1 - \exp\left( \int_{0}^{\infty} \left|e^{-s\frac{1}{n}\|{X}\|_2^2 }u\right|^{\alpha} \rmd s 
    - \int_{0}^{\infty} \left|e^{-s\frac{1}{n}\|\hat{X}\|_2^2 }u\right|^{\alpha} \rmd s\right)  \right) \Bigg| \\
    &= \exp\left(- \int_{0}^{\infty} \left|e^{-s\frac{1}{n}\|X\|_2^2 }u\right|^{\alpha} \rmd s\right)\left(1 - \exp\left( - |u|^{\alpha}\left[ \frac{n}{\alpha \|\hat{X}\|_2^2 }   - \frac{n}{\alpha \|{X}\|_2^2 } \right] \right)  \right)  \\
    &=\exp\left(- |u|^\alpha \frac{n}{\alpha \|X\|_2^2}\right)\left(1 - \exp\left( - |u|^{\alpha}\left[ \frac{n \overbrace{(x_i^2 - \tilde{x}_i^2)}^{:=\delta}}{\alpha \|X\|_2^2 \|\hat{X}\|_2^2 }    \right] \right)  \right) \\
    &=\exp\left(- |u|^\alpha \frac{n}{\alpha \|X\|_2^2}\right) \left[ \sum_{k=1}^{\infty} \frac{(-1)^{k+1}}{k!} |u|^{k \alpha}\left( \frac{ n \delta}{\alpha \|X \|_2^2 \| \hat{X}\|_2^2} \right)^{k} \right].
\end{align*}
Hence,
\begin{align}
    \Phi_{x,X,\hat{X}}(\mathcal{A}_{\text{cont}}) &=    \left[\frac{1}{\pi} \int_{0}^\infty |\psi_{\theta}(u) - \psi_{\hat{\theta}}(u)| \left|\int_{-\infty}^{\infty}|\theta x|^p e^{-\imagi  u \theta}~\rmd \theta \right|~\rmd u \right] \notag \\
    &=   \frac{2|x|^p}{\pi}  \Gamma(p+1) \cos\left( \frac{(p+1)\pi}{2}\right) \int_{0}^{\infty} \exp\left(- u^\alpha \frac{n}{\alpha \|X\|_2^2}\right) \notag \\
    & \qquad \qquad \qquad \qquad \qquad \cdot \left[ \sum_{k=1}^{\infty} \frac{(-1)^{k+1}}{k!} u^{k \alpha-p - 1}\left( \frac{ n \delta}{\alpha \|X \|_2^2 \| \hat{X}\|_2^2} \right)^{k} \right] \rmd u. \label{eq:intermediate_lower_1_d}
\end{align}
To simplify the above term, we do need to compute the integral: 
\begin{align*}
    \int_{0}^{\infty} \exp\left( - \frac{u^\alpha n}{\alpha \| X\|_2^2} \right) \frac{u^{k\alpha}}{u^{p+1}} ~\rmd u.
\end{align*}
Let us use the substitution:
\begin{align*}
    \frac{u^\alpha n}{\alpha \| X\|_2^2} = t
\end{align*}
such that
\begin{align*}
\rmd t =   \frac{u^{\alpha - 1} n}{  \|X\|_2^2} \rmd u \text{~~~~and ~} u = \left(\frac{\alpha \| X\|_2^2}{n}\right)^{1/\alpha} t^{1/\alpha}.
\end{align*}
Hence,
\begin{align*}
    \int_{0}^{\infty} \exp\left( - \frac{u^\alpha n}{\alpha \|X\|_2^2} \right) \frac{u^{k\alpha}}{u^{p+1}} ~\rmd u &= \int_{0}^{\infty} e^{-t} u^{\alpha(k-1) - p} \frac{\|X\|_2^2}{n} ~\rmd t \\
    &= \frac{\|X\|_2^2}{n} \left(\frac{\alpha \|X \|_2^2}{n}\right)^{k-1 - \frac{p}{\alpha}} \int_{0}^{\infty} e^{-t} t^{k-1 - \frac{p}{\alpha}} \rmd t \\
    & = \frac{\|X\|_2^2}{n} \left(\frac{\alpha \|X \|_2^2}{n}\right)^{k-1 - \frac{p}{\alpha}} \Gamma\left( k - \frac{p}{\alpha}\right).
\end{align*}
This implies, 
\begin{align*}
    \Phi_{x,X,\hat{X}}(\mathcal{A}_{\text{cont}}) &=  \frac{2|x|^p}{\pi}  \Gamma(p+1) \cos\left( \frac{(p+1)\pi}{2}\right)\int_{0}^{\infty}  \exp\left( - \frac{|u|^\alpha n}{\alpha \| X\|_2^2} \right)  \\ 
    &\qquad \qquad \qquad \qquad \qquad \qquad \cdot\sum_{k=1}^{\infty} \left[ \frac{(-1)^{k+1}}{k!} u^{k \alpha - p - 1}\left( \frac{ n \delta}{\alpha \|X \|_2^2 \| \hat{X}\|_2^2} \right)^{k} \right] ~\rmd u \\
    &=  \frac{2|x|^p}{\pi}  \Gamma(p+1) \cos\left( \frac{(p+1)\pi}{2}\right) 
    \\
    &\qquad \qquad \qquad \qquad \qquad \qquad \cdot\sum_{k=1}^{\infty}\left[\frac{(-1)^{k+1}}{k!} \frac{\delta^k}{\alpha \|\hat{X}\|_2^{2k}} \left( \frac{n}{\alpha \|X\|_2^2} \right)^{\frac{p}{\alpha}} \Gamma\left(k - \frac{p}{\alpha}\right)\right] \\
    &=   \frac{2|x|^p}{\pi}  \Gamma(p+1) \cos\left( \frac{(p+1)\pi}{2}\right) \sum_{k=1}^{\infty} (-1)^{k+1} \gamma_k,
\end{align*}
where $\gamma_k = \frac{1}{k!}\frac{\delta^k}{\alpha \|\hat{X}\|_2^{2k}} \left( \frac{n}{\alpha \|X\|_2^2} \right)^{\frac{p}{\alpha}} \Gamma\left(k - \frac{p}{\alpha}\right)$.
Let us now compute $\frac{\gamma_{k+1}}{\gamma_k}$. We have  
\begin{align*}
   \frac{\gamma_{k+1}}{\gamma_k} = \frac{\delta}{\|\hat{X}\|_2^2} \frac{\Gamma\left(k+1 - \frac{p}{\alpha}\right)}{(k+1)\Gamma\left(k - \frac{p}{\alpha}\right)} = \frac{\delta}{\|\hat{X}\|_2^2} \frac{k-\frac{p}{\alpha}}{k+1} = \frac{\delta}{\|\hat{X}\|_2^2} \left( 1 - \frac{1+\frac{p}{\alpha}}{k+1}\right).
\end{align*}
Hence,   we have the following,
\begin{align}
    \Phi_{x,X,\hat{X}}(\mathcal{A}_{\text{cont}}) &=   \frac{2|x|^p}{\pi}  \Gamma(p+1) \cos\left( \frac{(p+1)\pi}{2}\right) \left[ \gamma_1\sum_{k=1}^{\infty} \left(-\frac{\delta}{\|\hat{X}\|_2^2}\right)^{k-1}  \right. \notag \\
   &\qquad \qquad \qquad \left.- \gamma_1\left(1+\frac{p}{\alpha}\right)\sum_{k=1}^{\infty} \left(-\frac{\delta}{\|\hat{X}\|_2^2}\right)^{k-1}\prod_{j=1}^{k-1} \left(\frac{1}{j+1} \right)\right] \notag \\
   &\geq   \frac{2|x|^p}{\pi}  \Gamma(p+1) \cos\left( \frac{(p+1)\pi}{2}\right)   \gamma_1\sum_{k=1}^{\infty} \left(-\frac{\delta}{\|\hat{X}\|_2^2}\right)^{k-1} \notag \\
   &=  \frac{2|x|^p}{\pi} \Gamma(p+1) \cos\left( \frac{(p+1)\pi}{2}\right) \frac{\gamma_1}{1+ \frac{\delta}{\|\hat{X}\|_2^2}} \notag \\
   &\geq   \frac{2|x|^p}{\pi}  \Gamma(p+1) \cos\left( \frac{(p+1)\pi}{2}\right) \frac{\delta}{\alpha \|\hat{X}\|_2^{2}} \left( \frac{n}{\alpha \|X\|_2^2} \right)^{\frac{p}{\alpha}} \Gamma\left(1 - \frac{p}{\alpha}\right). \label{eq:just_before_final_lower_bo} 
\end{align}
Now, let us assume that $P_X$ is a distribution with discrete support in range $\sigma^2$ to $R$ with $C$ number of support points equally spaced. Hence, with probability $(1 - 1/C)$, $\delta \geq c$ for some positive constant $c$. Hence, with with high probability,
\begin{align*}
    \Phi_{x,X,\hat{X}}(\mathcal{A}_{\text{cont}}) \geq  \frac{2\sigma^{2p}c}{R^2 \pi  n}  \Gamma(p+1) \cos\left( \frac{(p+1)\pi}{2}\right) \frac{1}{\alpha }\left( \frac{1}{\alpha R^2} \right)^{\frac{p}{\alpha}} \Gamma\left(1 - \frac{p}{\alpha}\right).
\end{align*}
This completes the proof. 
\end{proof}

\section{Proofs for Least-Square in $d$-Dimension} \label{ap:proof_d_dim}

In this section, we provide the proofs for least-square in $d$-dimension.
We start by proving the following lemma, relating the characteristic functions of the two distributions.

\begin{lemma}
\label{lem:dd_char_func}
For two matrices $X \in \mathbb{R}^{n\times d}$ and $\hat{X} \in \mathbb{R}^{n \times d}$ as defined earlier, the absolute value of difference between the characteristic functions of the stationary distributions at any $u \in \mathbb{R}^d$ corresponding to $d$-dimensional rotation invariant  processes in equations~\eqref{eq:char_sde_1} and \eqref{eq:char_sde_2}  is bounded as 
   \begin{align*}
    \left|\psi_{\theta}(u) - \psi_{\hat{\theta}}(u) \right| &\leq \frac{2(\sigma_1+\sigma_2)\|u\|_{2}^{\alpha}}{n \alpha^2 \sigma_{\min}^2}  \exp\left(- \frac{\| u\|_{2}^{\alpha}}{\alpha \sigma_{\min}} \right),
\end{align*}
where $\sigma_{\min}$ is the smaller of the smallest of singular values of the matrices $\frac{1}{n}X^\top X$ and $\frac{1}{n}\hat{X}^\top \hat{X}$, and $x_i x_i^\top - \tilde{x}_i {\tilde{x}_i}^\top = \sigma_1 v_1 v_1^\top + \sigma_2 v_2 v_2^\top$ where $v_1$ and $v_2$ are orthogonal vectors.
\end{lemma}

\begin{proof}
We can compute that
\begin{align}
    &\left|\psi_{\theta}(u)- \psi_{\hat{\theta}}(u)\right| 
    \notag\\
    &= \left| \exp\left(- \int_{0}^{\infty} \left\|e^{-s\frac{1}{n} X^\top X }u\right\|_{2}^{\alpha} \rmd s\right) - \exp\left(- \int_{0}^{\infty} \left\|e^{-s\frac{1}{n} \hat{X}^\top\hat{X}  }u\right\|_{2}^{\alpha} \rmd s\right) \right| \notag \\
    &\leq \underbrace{\exp\left(- \int_{0}^{\infty} \left\|e^{-s\frac{1}{n} X^\top X }u\right\|_{2}^{\alpha} \rmd s\right)}_{:=B} \underbrace{\left| \int_{0}^{\infty} \left\|e^{-s\frac{1}{n} X^\top X }u\right\|_{2}^{\alpha} \rmd s - \int_{0}^{\infty} \left\|e^{-s\frac{1}{n} \hat{X}^\top \hat{X}  }u\right\|_{2}^{\alpha} \rmd s \right|}_{:=C}. \notag
\end{align}
We first consider the term $C$ in the above equation. From Lemma~\ref{lem:ab_alpha}, we have for two positive numbers $a$ and $b$, and for some $1\leq \alpha \leq 2$, we have
\begin{align*}
    |a^\alpha - b^\alpha | \leq |a-b|(a^{\alpha-1}+b^{\alpha-1}).
\end{align*}
Now,
\begin{align*}
    C &= \left| \int_{0}^{\infty} \left\|e^{-s\frac{1}{n} X^\top X }u\right\|_{2}^{\alpha} \rmd s - \int_{0}^{\infty} \left\|e^{-s\frac{1}{n} \hat{X}^\top \hat{X} }u\right\|_{2}^{\alpha} \rmd s \right| \\
    &=\left| \int_{0}^{\infty} \left( \left\|e^{-s\frac{1}{n} X^\top X }u\right\|_{2}^{\alpha} - \left\|e^{-s\frac{1}{n} \hat{X}^\top \hat{X} }u\right\|_{2}^{\alpha} \right) \rmd s \right| \\
    &\leq \int_{0}^{\infty} \left| \left \|e^{-s\frac{1}{n} X^\top X }u\right\|_{2} - \left\|e^{-s\frac{1}{n} \hat{X}^\top \hat{X} }u\right\|_2 \right| \left( \left \|e^{-s\frac{1}{n} X^\top X }u\right\|_{2}^{\alpha-1} + \left\|e^{-s\frac{1}{n} \hat{X}^\top \hat{X} }u\right\|_2^{\alpha-1} \right) \rmd s \\
    &\leq \int_{0}^{\infty}  \left \|e^{-s\frac{1}{n} X^\top X }u - e^{-s\frac{1}{n} \hat{X}^\top \hat{X} }u\right\|_2  \left( \left \|e^{-s\frac{1}{n} X^\top X }u\right\|_{2}^{\alpha-1} + \left\|e^{-s\frac{1}{n} \hat{X}^\top \hat{X}}u\right\|_2^{\alpha-1} \right) \rmd s \\
    &= \int_{0}^{\infty}  \left \|e^{-s\frac{1}{n} X^\top X }\left(I - e^{s\frac{1}{n} X^\top X - s\frac{1}{n} \hat{X}^\top \hat{X} }\right)u\right\|_2  
    \\
    &\qquad\qquad\qquad\qquad\cdot\left( \left \|e^{-s\frac{1}{n} X^\top X }u\right\|_{2}^{\alpha-1} + \left\|e^{-s\frac{1}{n} \hat{X}^\top \hat{X}}u\right\|_2^{\alpha-1} \right) \rmd s.
\end{align*}
Now, we have from the definitions, 
\begin{align*}
   X^\top X - \hat{X}^\top \hat{X} = x_i x_i^\top - \tilde{x}_i \tilde{x}_i^\top.
\end{align*}
Hence,
\begin{align*}
    C &\leq \int_{0}^{\infty}  \left \|e^{-s\frac{1}{n} X^\top X }\left(I - e^{s\frac{1}{n} \left(x_i x_i^\top - \tilde{x}_i \tilde{x}_i^\top\right) }\right)u\right\|_2  \left( \left \|e^{-s\frac{1}{n} X^\top X }u\right\|_{2}^{\alpha-1} + \left\|e^{-s\frac{1}{n} \hat{X}^\top \hat{X}}u\right\|_2^{\alpha-1} \right) \rmd s .
\end{align*}
We recall that the 2-norm  $\|\cdot\|_{2}$ for a matrix $D \in \mathbb{R}^{d\times d}$ is defined as follows:
\begin{align*}
    \|D\|_{2} = \sup_{u\in\mathbb{R}^{d}: \|u\|_{2}=1} \|Du\|_2
\end{align*}
for $u \in \mathbb{R}^d$.
We notice that $$ \min\left(\frac{1}{n}\left\|\left(X^\top X\right) u \right\|_2, \frac{1}{n}\left\|\left(\hat{X}^\top \hat{X}\right) u \right\|_2 \right) \geq \sigma_{\min}\|u\|_{2}.$$ 
Hence,
\begin{align}
    C &\leq \int_{0}^{\infty}  \left \|e^{-s\frac{1}{n} X^\top X }u\right\|_{2}\left\|I - e^{-s\frac{1}{n} \left(\tilde{x}_i \tilde{x}_i^\top -  {x}_i x_{i}^{\top}\right) }\right\|_{2}  
    \\
    &\qquad\qquad\qquad\qquad\cdot\left( \left \|e^{-s\frac{1}{n} X^\top X }u\right\|_{2}^{\alpha-1} + \left\|e^{-s\frac{1}{n} \hat{X}^\top \hat{X}}u\right\|_2^{\alpha-1} \right) \rmd s \notag \\
    &=\int_{0}^{\infty} \underbrace{\left\|I - e^{s\frac{1}{n} \left(x_i x_i^\top - \tilde{x_i}\tilde{x_i}^\top\right) }\right\|_{2} }_{:=D} \underbrace{\left( \left \|e^{-s\frac{1}{n} X^\top X }u\right\|_{2}^{\alpha} +  \left \|e^{-s\frac{1}{n} X^\top X }u\right\|_{2} \left\|e^{-s\frac{1}{n} \hat{X}^\top \hat{X}}u\right\|_2^{\alpha-1} \right)}_{:=E} \rmd s. \label{eq:upper_bnd_C}
\end{align}

Let us consider the term $D$ in \eqref{eq:upper_bnd_C} first. $\tilde{x}_i {\tilde{a}_i}^\top  - x_i x_i^\top$ is a rank 2 matrix. Consider the two non-zero eigenvalues of this matrix are $\sigma_1$ and $\sigma_2$. Hence, $\tilde{x}_i {\tilde{x}_i}^\top  - x_i x_i^\top = \sigma_1 v_1 v_1^\top + \sigma_2 v_2 v_2^\top$ where $v_1$ and $v_2$ are the eigenvectors.  Then,
\begin{align*}
   I - e^{-s\frac{1}{n} \left(\tilde{x}_i {\tilde{a}_i}^\top  - x_i x_i^\top\right) } = \left(1 - e^{\frac{-s\sigma_1}{n}}\right) v_1 v_1^\top  + \left(1 - e^{\frac{-s\sigma_2}{n}}\right) v_2v_2^\top.
\end{align*}
Hence,
\begin{align*}
    D = \left\|I - e^{-s\frac{1}{n} \left(\tilde{x}_i {\tilde{a}_i}^\top  - x_i x_i^\top\right) }\right\|_{2} &\leq \left\|\left(1 - e^{\frac{-s\sigma_1}{n}}\right) v_1 v_1^\top \right\|_{2} + \left\|\left(1 - e^{\frac{-s\sigma_2}{n}}\right) v_2v_2^\top\right\|_{2} \\
    &\leq \frac{s\sigma_1}{n}  + \frac{s\sigma_2}{n}, 
\end{align*}
where $v_1$ and $v_2$ are orthogonal vectors with $\|v_1\|_2= \|v_2\|_2 = 1$ and $v_1^\top v_2 = 0$. By definition, we have 
\begin{align*}
    \frac{1}{n} \left\| X^\top X u\right\|_2 \geq  \sigma_{\min}\| u\|_2, \quad\text{and}\quad
    \frac{1}{n} \left\| \hat{X}^\top \hat{X} u\right\|_2 \geq  \sigma_{\min}\| u\|_2.
\end{align*}
This gives,
\begin{align*}
    \left\| e^{-s\frac{1}{n}X^\top X}u \right\|_2 \leq e^{-s \sigma_{\min}},\quad \text{and}\quad
    \left\| e^{-s\frac{1}{n}\hat{X}^\top \hat{X}}u \right\|_2 \leq e^{-s \sigma_{\min}},
\end{align*}
which implies that the $E$ term in \eqref{eq:upper_bnd_C} can be bounded as:
\begin{align}
    E \leq 2 \| u\|_2^{\alpha} e^{-s \alpha \sigma_{\min}}.
\end{align}
Therefore,
\begin{align}
    C &\leq \frac{2(\sigma_1+\sigma_2)\|u\|_{2}^{\alpha}}{n} \int_{0}^{\infty}  s e^{-s \alpha \sigma_{\min}} \rmd s = \frac{2(\sigma_1+\sigma_2)\|u\|_{2}^{\alpha}}{n \alpha^2 \sigma_{\min}^2}. \label{eq:final_C_bound}
\end{align}

Hence, we have
\begin{align}
    |\psi_1(u) - \psi_2(u)| &\leq \frac{2(\sigma_1+\sigma_2)\|u\|_{2}^{\alpha}}{n \alpha^2 \sigma_{\min}^2}  \exp\left(- \int_{0}^{\infty} \left\|e^{-s\frac{1}{n} X^\top X }u\right\|_{2}^{\alpha} \rmd s\right) \notag  \\
    &\leq \frac{2(\sigma_1+\sigma_2)\|u\|_{2}^{\alpha}}{n \alpha^2 \sigma_{\min}^2}  \exp\left(- \| u\|_{2}^{\alpha} \int_{0}^{\infty} e^{-s\alpha \sigma_{\min} }  \rmd s\right), \label{eq:bound_B}
\end{align}
where the last inequality is due to the definition of $\sigma_{\min}$. Hence we conclude that 
\begin{align}
    |\psi_1(u) - \psi_2(u)| &\leq \frac{2(\sigma_1+\sigma_2)\|u\|_{2}^{\alpha}}{n \alpha^2 \sigma_{\min}^2}  \exp\left(- \frac{\| u\|_{2}^{\alpha}}{\alpha \sigma_{\min}} \right),
\end{align}
which completes the proof.
\end{proof}

\begin{theorem*}[Restatement of Theorem~\ref{thm:dd_main}] Consider  the $d$-dimensional loss function $f(x) =|\theta^\top x|^p$ such that $\theta, x \in \mathbb{R}^d$.  For any $x\sim P_X$ if $\|x\|_2 \leq R $,  for any $X$ sampled uniformly at random from the set $\mathcal{X}_n$, if  $\frac{1}{n} \| X^\top X u\|_2 \geq  \sigma_{\min} \|u\|_2$ for $u\in \mathbb{R}^d$ and for any two $X\cong \hat{X}$ sampled from $\mathcal{X}_n$   generating two stochastic process given by SDEs in equations~\eqref{eq:sde_invar_1} and \eqref{eq:sde_invar_2}, $\|x_i x_i^\top - \tilde{x}_i \tilde{x}_i^\top \|_2  \leq 2\sigma$ holds  with high probability. Then,
\begin{enumerate}
   \item[(i)] For $\alpha \in (1,2)$, the algorithm is not stable when $p \in [\alpha,2]$ i.e. $\varepsilon_{\text{stab}(\mathcal{A}_{\text{cont}})}$ diverges. When $\alpha = p =2$ then with high probability $ \varepsilon_{\text{stab}} (\mathcal{A}_{\text{cont}}) \leq \frac{2R^2}{\pi}  \frac{\sigma}{n   \sigma_{\min}^2}$.
    \item[(ii)] For $p\in [1,\alpha)$, we have the following upper bound for the algorithmic stability,
    \begin{align*}
         \varepsilon_{\text{stab}} (\mathcal{A}_{\text{cont}})&=  \frac{8R^p}{\pi} \frac{\sigma}{n\alpha^2 \sigma_{\min} } \Gamma(p+1) \cos\left( \frac{(p-1)\pi}{2}\right)  \left(\frac{1}{\alpha \sigma_{\min}}\right)^{\frac{p}{\alpha}} \Gamma\left(1-\frac{p}{\alpha}\right)= c(\alpha),
    \end{align*}
    which holds with high probability. Furthermore, for some $\alpha_0 > 1$, if we have $$\sigma_{\min} \geq \exp\left( 1 + \frac{4}{p} - \log \alpha_0 -  \phi\left( 1 - \frac{p}{\alpha_0}\right) \right),$$
    where $\phi$ is the digamma function, then the map $\alpha\rightarrow c(\alpha)$ is increasing for $\alpha \in [\alpha_0,2)$.
    \item[iii] The stability bound is tight in $\alpha$. 
    \end{enumerate}
\end{theorem*}

\begin{proof}
We have $d$-dimensional loss function for an $x \in \mathbb{R}^d$ sampled uniformly at random from $P_X$, $f(\theta) = |\theta^\top x|^p $ Let us denote the Fourier transform of $f$, $\mathcal{F}f(u)$ as $h(u)$. For an orthogonal matrix $A$ such that $Ae_1 = \frac{x}{\|x\|_2}$, we have from the results in Lemma~\ref{lem:ft_theta_x},  
\begin{align}
    h(Au) =  2 \|x\|_2^p (2\pi)^{d-1} \delta(u_2,\cdots,u_d) \Gamma(p+1) \cos\left( \frac{(p+1)\pi}{2}\right) \frac{1}{|u_1|^{p+1}} \quad\text{for~} p\in[1,2), \label{eq:fou_1_d}
\end{align}
and
\begin{align}
    h(Au) = 2\|x\|_2^p (2\pi)^{d-1} \delta(u_1,u_2,\cdots,u_d) \frac{2}{u_1^2} \quad\text{for~} p=2, \label{eq:fou_2_d}
\end{align}
where $\delta$ is the Dirac-delta function. Let us first consider the case when $p \in [1,2)$.  From equation~\eqref{eq:stab_char}, 
\begin{align*}
    \varepsilon_{\text{stab}} (\mathcal{A}_{\text{cont}})
    &= \sup_{X\cong \hat{X}}\sup_{x \in \mathcal{X}}  \frac{1}{(2\pi)^d} \int_{\mathbb{R}^d} \left|\psi_{\theta}(u) - \psi_{\hat{\theta}}(u)\right| |h(u)|~\rmd u  \\
    & = \sup_{X\cong \hat{X}}\sup_{x \in \mathcal{X}} \frac{1}{(2\pi)^d} \int_{\mathbb{R}^d} \left|\psi_{\theta}(u) - \psi_{\hat{\theta}}(u)\right| \left| \int_{\mathbb{R}^d} |\theta^\top x|^p e^{\imagi  u^\top \theta}~\rmd \theta\right|~\rmd u \\
    &= \sup_{X\cong \hat{X}}\sup_{x \in \mathcal{X}} \frac{1}{(2\pi)^d} \int_{\mathbb{R}^d} \frac{2(\sigma_1+\sigma_2)\|u\|_{2}^{\alpha}}{n \alpha^2 \sigma_{\min}^2}  
 \exp\left(- \frac{\| u\|_{2}^{\alpha}}{\alpha \sigma_{\min}} \right) \left| \int_{\mathbb{R}^d} |\theta^\top x|^p e^{\imagi  u^\top \theta}~\rmd \theta \right|~\rmd u.
\end{align*}
In the above equation, let us apply the change of variable $u=Av$ and use result from Lemma~\ref{lem:ft_theta_x} (equations~\eqref{eq:fou_1_d}) and we get the following, 
\begin{align*}
  \varepsilon_{\text{stab}} (\mathcal{A}_{\text{cont}})&=  \sup_{X\cong \hat{X}}\sup_{x \in \mathcal{X}} \frac{1}{(2\pi)^d} \int_{\mathbb{R}^d} \frac{2(\sigma_1+\sigma_2)\|Av\|_{2}^{\alpha}}{n \alpha^2 \sigma_{\min}^2}  \exp\left(- \frac{\| Av\|_{2}^{\alpha}}{\alpha \sigma_{\min}} \right) \left| \int_{\mathbb{R}^d} |\theta^\top x|^p e^{\imagi  (Av)^\top \theta}~\rmd \theta \right|~\rmd v \\
  &=  \sup_{X\cong \hat{X}}\sup_{x \in \mathcal{X}} \frac{1}{(2\pi)^d} \int_{\mathbb{R}^d} \frac{2(\sigma_1+\sigma_2)\|v\|_{2}^{\alpha}}{n \alpha^2 \sigma_{\min}^2}  \exp\left(- \frac{\| v\|_{2}^{\alpha}}{\alpha \sigma_{\min}} \right) \left| h(Av) \right|~\rmd v \\
  &= \sup_{X\cong \hat{X}}\sup_{x \in \mathcal{X}} \frac{1}{(2\pi)^d} \int_{\mathbb{R}^d} \left[ \left( \frac{2(\sigma_1+\sigma_2)\|v\|_{2}^{\alpha}}{n \alpha^2 \sigma_{\min}^2}  \exp\left(- \frac{\| v\|_{2}^{\alpha}}{\alpha \sigma_{\min}} \right) \right) \right. ~ \\
  &\qquad  \left.  \cdot\left(\left|  2 \|x\|_2^p (2\pi)^{d-1} \delta(v_2,\cdots,v_d) \Gamma(p+1) \cos\left( \frac{(p+1)\pi}{2}\right) \frac{1}{|v_1|^{p+1}} \right| \right)\right]~\rmd v \\
  &= \sup_{X\cong \hat{X}}\sup_{x \in \mathcal{X}} \frac{2\|x\|_2^p}{\pi} \Gamma(p+1) \cos\left( \frac{(p-1)\pi}{2}\right) \frac{(\sigma_1+\sigma_2)}{n\alpha^2 \sigma_{\min}^2} \\
  &\qquad \qquad \qquad \qquad \qquad \qquad \qquad \qquad\cdot\int_{-\infty}^{\infty}   |v_1|^\alpha \exp\left(  \frac{-|v_1|^\alpha}{\alpha \sigma_{\min}}\right)  \frac{1}{|v_1|^{p+1}} ~\rmd v_1 \\
  & = \sup_{X\cong \hat{X}}\sup_{x \in \mathcal{X}} \frac{4\|x\|_2^p}{\pi} \frac{(\sigma_1+\sigma_2)}{n\alpha^2 \sigma_{\min}^2} \Gamma(p+1) \cos\left( \frac{(p-1)\pi}{2}\right)\\
  &\qquad \qquad \qquad \qquad \qquad \qquad \qquad \qquad   \cdot\int_{0}^{\infty} v_1^{\alpha-p-1} \exp\left(  \frac{-v_1^\alpha}{\alpha \sigma_{\min}}\right)~\rmd v_1.
\end{align*}
In the above integral, by substituting $ \frac{v_1^\alpha}{\alpha \sigma_{\min}}$ with $t$ so that
\begin{align}
    \rmd t =  v_1^{\alpha-1} \frac{1}{\sigma_{\min}}\rmd v_1,~ \text{and } \frac{1}{v_1^p} = \left( \frac{1}{\alpha \sigma_{\min}}\right)^{\frac{p}{\alpha}} t^{-p/\alpha},
\end{align}
we have,
\begin{align}
 \varepsilon_{\text{stab}} (\mathcal{A}_{\text{cont}})&= \sup_{X\cong \hat{X}}\sup_{x \in \mathcal{X}} \frac{4\|x\|_2^p}{\pi} \frac{(\sigma_1+\sigma_2)}{n\alpha^2 \sigma_{\min}} \nonumber
 \\
 &\qquad\qquad\qquad\cdot\Gamma(p+1) \cos\left( \frac{(p-1)\pi}{2}\right)  \left(\frac{1}{\alpha \sigma_{\min}}\right)^{\frac{p}{\alpha}} \int_{0}^{\infty} t^{-p/\alpha} e^{-t}~\rmd t. \label{eq:p_1_2_d}
\end{align}
It is clear that, the above integral diverge for $p\geq \alpha$, hence the algorithm is not stable for $p \in [1,2)$. Now, we check the case for $p=2$. For $p=2$, we have,
\begin{align*}
    \varepsilon_{\text{stab}} (\mathcal{A}_{\text{cont}})
    &= \sup_{X\cong \hat{X}}\sup_{x \in \mathcal{X}} \frac{1}{(2\pi)^d} \int_{\mathbb{R}^d} \frac{2(\sigma_1+\sigma_2)\|u\|_{2}^{\alpha}}{n \alpha^2 \sigma_{\min}^2}  \exp\left(- \frac{\| u\|_{2}^{\alpha}}{\alpha \sigma_{\min}} \right) \left| \int_{\mathbb{R}^d} |\theta^\top x|^2 e^{\imagi  u^\top \theta}~\rmd \theta \right|~\rmd u.
\end{align*}
In the above equation, we make change of variable $u=Av$ and use result from Lemma~\ref{lem:ft_theta_x} (equations~\eqref{eq:fou_2_d}), we get the following,
\begin{align*}
  \varepsilon_{\text{stab}} (\mathcal{A}_{\text{cont}})&=  \sup_{X\cong \hat{X}}\sup_{x \in \mathcal{X}} \frac{1}{(2\pi)^d} \int_{\mathbb{R}^d} \frac{2(\sigma_1+\sigma_2)\|Av\|_{2}^{\alpha}}{n \alpha^2 \sigma_{\min}^2}  \\
  &\qquad\qquad\qquad\qquad\qquad\qquad\cdot\exp\left(- \frac{\| Av\|_{2}^{\alpha}}{\alpha \sigma_{\min}} \right) \left| \int_{\mathbb{R}^d} |\theta^\top x|^2 e^{\imagi  (Av)^\top \theta}~\rmd \theta \right|~\rmd v \\
  &=  \sup_{X\cong \hat{X}}\sup_{x \in \mathcal{X}} \frac{1}{(2\pi)^d} \int_{\mathbb{R}^d} \frac{2(\sigma_1+\sigma_2)\|v\|_{2}^{\alpha}}{n \alpha^2 \sigma_{\min}^2}  \exp\left(- \frac{\| v\|_{2}^{\alpha}}{\alpha \sigma_{\min}} \right) \left| h(Av) \right|~\rmd v  \\
  &= \sup_{X\cong \hat{X}}\sup_{x \in \mathcal{X}} \frac{2}{\pi} \int_{\mathbb{R}^d} \frac{(\sigma_1+\sigma_2)\|v\|_{2}^{\alpha}}{n \alpha^2 \sigma_{\min}^2}  \exp\left(- \frac{\| v\|_{2}^{\alpha}}{\alpha \sigma_{\min}} \right) \|x\|_2^2  \delta(v_1,v_2,\cdots,v_d) \frac{2}{v_1^2} ~\rmd v.
\end{align*}
The above integral clearly diverges for $\alpha < 2$. However, when $\alpha =2$, then
\begin{align*}
    \varepsilon_{\text{stab}} (\mathcal{A}_{\text{cont}})&\leq \frac{\|x\|_2^2}{\pi}  \frac{(\sigma_1+\sigma_2)}{n   \sigma_{\min}^2} . 
\end{align*}
Now, if $\sigma$ is the upper bound on $\sigma_1$ and $\sigma_2$ for all $X\cong \hat{X} \in \mathcal{X}_n$ and $\|x\|_2 \leq R$ for $x\sim P_X$ with high probability, then, 
\begin{align*}
    \varepsilon_{\text{stab}} (\mathcal{A}_{\text{cont}})&\leq \frac{2R^2}{\pi}  \frac{\sigma}{n   \sigma_{\min}^2},
\end{align*}
holds with high probability. This proves the part (i) of our claim. 

Next, we will prove part (ii) when $p< \alpha$. We have from equation~\eqref{eq:p_1_2_d},
\begin{align*}
 \varepsilon_{\text{stab}} (\mathcal{A}_{\text{cont}})&= \sup_{X\cong \hat{X}}\sup_{x \in \mathcal{X}} \frac{4\|x\|_2^p}{\pi} \frac{(\sigma_1+\sigma_2)}{n\alpha^2 \sigma_{\min} } \Gamma(p+1) \cos\left( \frac{(p-1)\pi}{2}\right)  \left(\frac{1}{\alpha \sigma_{\min}}\right)^{\frac{p}{\alpha}} \int_{0}^{\infty} t^{-p/\alpha} e^{-t}~\rmd t \\
 &=\sup_{X\cong \hat{X}}\sup_{x \in \mathcal{X}} \frac{4\|x\|_2^p}{\pi} \frac{(\sigma_1+\sigma_2)}{n\alpha^2 \sigma_{\min} } \Gamma(p+1) \cos\left( \frac{(p-1)\pi}{2}\right)  \left(\frac{1}{\alpha \sigma_{\min}}\right)^{\frac{p}{\alpha}} \Gamma\left(1-\frac{p}{\alpha}\right).
\end{align*}
Now, if $\sigma$ is the upper bound on $\sigma_1$ and $\sigma_2$ for all $X\cong \hat{X} \in \mathcal{X}_n$ and $\|x\|_2 \leq R$ for $x\sim P_X$ with high probability then,
\begin{align*}
 \varepsilon_{\text{stab}} (\mathcal{A}_{\text{cont}})&=  \frac{8R^p}{\pi} \frac{\sigma}{n\alpha^2 \sigma_{\min} } \Gamma(p+1) \cos\left( \frac{(p-1)\pi}{2}\right)  \left(\frac{1}{\alpha \sigma_{\min}}\right)^{\frac{p}{\alpha}} \Gamma\left(1-\frac{p}{\alpha}\right)
\end{align*}
holds with high probability.  Now, consider the function,
\begin{align*}
    \Lambda(\alpha ) =  \frac{1}{\alpha^2} \left( \frac{1}{\alpha \sigma_{\min} }\right)^{\frac{p}{\alpha}} \Gamma \left( 1 - \frac{p}{\alpha}\right).
\end{align*}
We can compute that
\begin{align*}
    \partial_{\alpha} \log \Lambda(\alpha) = \frac{p}{\alpha^2} \left[ \log \alpha  + \log \sigma_{\min} -1 - \frac{2\alpha}{p} +  \phi\left( 1 - \frac{p}{\alpha}\right) \right], 
\end{align*}
where $\phi$ is the digamma function. For any arbitrary $\alpha_0$, if we choose $$\sigma_{\min} \geq \exp\left( 1 + \frac{4}{p} - \log \alpha_0 -  \phi\left( 1 - \frac{p}{\alpha_0}\right) \right),$$ then $\partial_{\alpha} \log \Lambda(\alpha) >0 $ for  $\alpha \in [\alpha_0,2)$. Hence, for all $\alpha_1,\alpha_2 \in [\alpha_0,2)$,  $\alpha_1 < \alpha_2 \Rightarrow \Lambda(\alpha_1) \leq \Lambda(\alpha_2)$. This proves that $c(\alpha)$ is an increasing map in $\alpha$.

This completes the proof till part (ii). Now, we will prove tightness result in $\alpha$. Let us have the following construction. Consider a one-dimensional distribution $P_{X}$ supported in a ring such that the density function $\int _{A}p(x)\rmd x \leq \eta$ such that $A = \{x: |x| \geq \sqrt{\sigma_{\min}   d \log d } \text{~or~}   |x|\leq R\}$. The empirical covariance matrix $X^\top X$ is a  diagonal matrix. Hence, from the results in \citet{flatto2019dixie}. with high probability $1- \delta$, we have
\begin{align*}
    \frac{1}{n} \left\|X^\top X u \right\|_2\geq \sigma_{\min} \|u\|_2 .
\end{align*}
Exact expression for $\delta$ is given in \citet{flatto2019dixie}. Similarly, for the dataset $\hat{X}$, the similar condition holds,
\begin{align*}
    \frac{1}{n} \left\|\hat{X}^\top \hat{X} u \right\|_2\geq \sigma_{\min} \|u\|_2^2
\end{align*}
with high probability $1- \delta$.
 We have,
\begin{align}
    &\left|\psi_{\theta}(u)- \psi_{\hat{\theta}}(u)\right| 
    \notag\\
    &= \left| \exp\left(- \int_{0}^{\infty} \left\|e^{-s\frac{1}{n} X^\top X }u\right\|_{2}^{\alpha} \rmd s\right) - \exp\left(- \int_{0}^{\infty} \left\|e^{-s\frac{1}{n} \hat{X}^\top\hat{X}  }u\right\|_{2}^{\alpha} \rmd s\right) \right| \notag \\
    &=\exp\left(- \int_{0}^{\infty} \left\|e^{-s\frac{1}{n} X^\top X }u\right\|_{2}^{\alpha} \rmd s\right) \left| 1 - \exp\left(- \int_{0}^{\infty} \left|\left\|e^{-s\frac{1}{n} \hat{X}^\top \hat{X} }u\right\|_{2}^{\alpha} - \left\|e^{-s\frac{1}{n} {X}^\top {X} }u\right\|_{2}^{\alpha} \right| \rmd s\right) \right|.
\end{align}
From equation~\eqref{eq:stab_char}, 
\begin{align*}
    \varepsilon_{\text{stab}} (\mathcal{A}_{\text{cont}})
    &= \sup_{X\cong \hat{X}}\sup_{x \in \mathcal{X}}  \frac{1}{(2\pi)^d} \int_{\mathbb{R}^d} \left|\psi_{\theta}(u) - \psi_{\hat{\theta}}(u)\right| |h(u)|~\rmd u  \\
    & = \sup_{X\cong \hat{X}}\sup_{x \in \mathcal{X}} \frac{1}{(2\pi)^d} \int_{\mathbb{R}^d} \left|\psi_{\theta}(u) - \psi_{\hat{\theta}}(u)\right| \left| \int_{\mathbb{R}^d} |\theta^\top x|^p e^{\imagi  u^\top \theta}~\rmd \theta\right|~\rmd u.
\end{align*}
In the above equation, let us apply the change of variable $u=Av$ where $A$ is the orthogonal matrix defined earlier and we get the following, 
\begin{align*}
 \varepsilon_{\text{stab}} (\mathcal{A}_{\text{cont}})
    &=    \sup_{X\cong \hat{X}}\sup_{x \in \mathcal{X}}  \frac{1}{(2\pi)^d} \int_{\mathbb{R}^d} \left|\psi_{\theta}(Av) - \psi_{\hat{\theta}}(Av)\right| \left| \int_{\mathbb{R}^d} |\theta^\top x|^p e^{\imagi  (Av)^\top \theta}~\rmd \theta \right|~\rmd v \\
    &=\sup_{X\cong \hat{X}}\sup_{x \in \mathcal{X}}  \frac{1}{(2\pi)^d} \int_{\mathbb{R}^d} \left|\psi_{\theta}(Av) - \psi_{\hat{\theta}}(Av)\right| \left| h(Av) \right|~\rmd v 
\end{align*}
Since, $A$ is orthogonal matrix, we can see that 
\begin{align*}
     \left|\psi_{\theta}(Av) - \psi_{\hat{\theta}}(Av)\right| =  \left|\psi_{\theta}(v) - \psi_{\hat{\theta}}(v)\right|.
\end{align*}
Hence,
\begin{align*}
   \varepsilon_{\text{stab}} (\mathcal{A}_{\text{cont}})
    &=  \sup_{X\cong \hat{X}}\sup_{x \in \mathcal{X}}  \frac{1}{(2\pi)^d} \int_{\mathbb{R}^d} \left|\psi_{\theta}(v) - \psi_{\hat{\theta}}(v)\right| \left| h(Av) \right|~\rmd v \\
    &=\sup_{X\cong \hat{X}}\sup_{x \in \mathcal{X}}  \frac{1}{(2\pi)^d} \int_{\mathbb{R}^d} \left|\psi_{\theta}(v) - \psi_{\hat{\theta}}(v)\right|   \Bigg(\Bigg|  2 \|x\|_2^p (2\pi)^{d-1} \delta(v_2,\cdots,v_d) 
    \\
    &\qquad\qquad\qquad\qquad\qquad\qquad\qquad\qquad\cdot\Gamma(p+1) \cos\left( \frac{(p+1)\pi}{2}\right) \frac{1}{|v_1|^{p+1}} \Bigg| \Bigg) \rmd v.
\end{align*}
Let us denote
\begin{align*}
    \Phi_{x,X,\hat{X}}(\mathcal{A}_{\text{cont}}) &:=\frac{1}{(2\pi)^d} \int_{\mathbb{R}^d} \left|\psi_{\theta}(v) - \psi_{\hat{\theta}}(v)\right|   \Bigg(\Bigg|  2 \|x\|_2^p (2\pi)^{d-1} \delta(v_2,\cdots,v_d)  \\
    &\qquad\qquad\qquad\qquad\qquad\qquad\qquad\qquad\cdot\Gamma(p+1) \cos\left( \frac{(p+1)\pi}{2}\right) \frac{1}{|v_1|^{p+1}} \Bigg| \Bigg) \rmd v.
\end{align*}
Now, we use the property of Dirac-delta function. 
From our construction, $X^\top X$ and $\hat{X}^\top \hat{X}$ are diagonal matrices. Let us denote $X^\top X = \text{diag}(a_1, a_2, \cdots , a_d)$. Similarly, we denote $\hat{X}^\top \hat{X} = \text{diag}(\hat{a}_1, \hat{a}_2, \cdots , \hat{a}_d)$. 
 Hence,
we have
\begin{align*}
   \exp\left(- \int_{0}^{\infty} \left\|e^{-s\frac{1}{n} X^\top X }v \right\|_{2}^{\alpha} \rmd s\right) &= \exp\left( - \int_{0}^{\infty} \left(\sum_{i=1}^d e^{-2 (s/n) a_i} v_i^2 \right)^{\frac{\alpha}{2}}\rmd s \right).
\end{align*}

From the construction, the matrix $X^\top X$ and $\hat{X}^\top X$ are both diagonal and differ at two diagonal elements with probability $(1-1/d)$. They differ at one diagonal element with probability $1/d$. Let's assume that $x_i$ has non-zero element at dimension 1 and $\tilde{x}_i$ has non-zero element either at dimension 1 or at 2 (without loss of generality). Hence, with high probability,
\begin{align*}
   & \left| 1 - \exp\left(- \int_{0}^{\infty} \left|\left\|e^{-s\frac{1}{n} \hat{X}^\top \hat{X} }u\right\|_{2}^{\alpha} - \left\|e^{-s\frac{1}{n} {X}^\top {X} }u\right\|_{2}^{\alpha} \right| \rmd s\right) \right| \\
    &= 1 - \exp\left( - \int_{0}^{\infty}\left| \left(\sum_{i=1}^d e^{-2(s/n)a_i}v_i^2\right)^{\alpha/2} - \left(\sum_{i=1}^d e^{-2(s/n)\hat{a}_i}v_i^2\right)^{\alpha/2} \right| \rmd s\right)  
\end{align*}

Combining everything together and using the property of Dirac-delta function we get,
\begin{align*}
    \Phi_{x,X,\hat{X}}(\mathcal{A}_{\text{cont}})
    &=  \frac{1}{2\pi } \int_{-\infty}^{\infty}\Bigg[ \exp\left( - \int_{0}^{\infty}    e^{- (s\alpha/n)  a_1} |v_1|^\alpha  \rmd s \right) 
    \\
    &\qquad\qquad\qquad\cdot\left[ 1 - \exp\left(- \int_{0}^{\infty} \left(e^{-(s\alpha/n) a_1} - e^{-(s\alpha/n) \tilde{a}_1}\right) |v_1|^\alpha \rmd s \right) \right]
    \\
    &\qquad\qquad\qquad\qquad\qquad\cdot  \Bigg(\Bigg|  2 \|x\|_2^p     \Gamma(p+1) \cos\left( \frac{(p+1)\pi}{2}\right) \frac{1}{|v_1|^{p+1}} \Bigg| \Bigg)\Bigg] \rmd v_1 \\
    &=   \frac{2\|x\|_2^p}{\pi } \Gamma(p+1) \cos\left( \frac{(p+1)\pi}{2}\right) \int_{0}^{\infty}\Bigg[ \exp\left( - \int_{0}^{\infty}    e^{- (s\alpha/n)  a_1} v_1^\alpha  \rmd s \right) 
    \\
    &\qquad\qquad\qquad\cdot \frac{1}{v_1^{p+1}} \left[ 1 - \exp\left(- \int_{0}^{\infty} \left(e^{-(s\alpha/n) a_1} - e^{-(s\alpha/n) \tilde{a}_1}\right) v_1^\alpha \rmd s \right) \right] \Bigg] \rmd v_1 \\
    &=   \frac{2\|x\|_2^p}{\pi } \Gamma(p+1) \cos\left( \frac{(p+1)\pi}{2}\right) \int_{0}^{\infty} \Bigg[ \exp \left(-v_1^\alpha\frac{n}{\alpha a_1}\right)  \\
    &\qquad \qquad \cdot \frac{1}{v_1^{p+1}} \left[ 1 - \exp\left(-v_1^{\alpha} \Big[ \frac{n}{\alpha a_1} - \frac{n}{\alpha \hat{a}_1} \Big] \right) \right] \Bigg] \rmd v_1 \\
    &=   \frac{2\|x\|_2^p}{\pi } \Gamma(p+1) \cos\left( \frac{(p+1)\pi}{2}\right) \int_{0}^{\infty}\Bigg[\frac{1}{v_1^{p+1}}\exp \left(-v_1^\alpha\frac{n}{\alpha a_1}\right)  \\
    &\qquad \qquad \qquad \qquad \cdot \left[1 - \exp\left(\frac{n (\|x_i\|_2^2 - \|\tilde{x}_i\|_2^2)}{\alpha a_1 \hat{a}_1}\right)\right]\Bigg] \rmd v_1.
\end{align*}
Let us denote $\delta := \| x_i\|_2^2 - \| \tilde{x}_i\|_2^2$. Hence,
\begin{align*}
    \Phi_{x,X,\hat{X}}(\mathcal{A}_{\text{cont}})
    &=   \frac{2\|x\|_2^p}{\pi } \Gamma(p+1) \cos\left( \frac{(p+1)\pi}{2}\right) \\
    & \qquad \qquad \qquad \qquad \cdot\int_{0}^{\infty} \Bigg[\frac{1}{v_1^{p+1}}\exp \left(-v_1^\alpha\frac{n}{\alpha a_1}\right) \left[1 - \exp\left(\frac{n \delta }{\alpha a_1 \hat{a}_1}\right)\right]\Bigg] \rmd v_1 \\
    &=   \frac{2\|x\|_2^p}{\pi \alpha} \Gamma(p+1) \cos\left( \frac{(p+1)\pi}{2}\right) \int_{0}^{\infty} \exp\left(- v_1^\alpha \frac{n}{\alpha a_1}\right) \notag \\
    & \qquad \qquad \qquad \qquad \qquad  \cdot\left[ \sum_{k=1}^{\infty} \frac{(-1)^{k+1}}{k!} v_1^{k \alpha-p - 1}\left( \frac{ n \delta}{\alpha a_1 \hat{a}_1} \right)^{k} \right] \rmd v_1.
\end{align*}
The above equation is just reduction to the computation of one-dimensional case which we did in equation~\eqref{eq:intermediate_lower_1_d}. We apply similar argument that we did apply in computing the lower bound in equation~\eqref{eq:intermediate_lower_1_d}. Hence, with high probability, we get (equation~\eqref{eq:just_before_final_lower_bo})
\begin{align*}
 \Phi_{x,X,\hat{X}}(\mathcal{A}_{\text{cont}}) \geq   \frac{2\|x\|_2^p}{\pi}  \Gamma(p+1) \cos\left( \frac{(p+1)\pi}{2}\right) \frac{\delta}{\alpha^2 \hat{a}_1} \left( \frac{n}{\alpha a_1} \right)^{\frac{p}{\alpha}} \Gamma\left(1 - \frac{p}{\alpha}\right). 
\end{align*}
Here, we also assume that $P_X$ is a distribution with discrete support in range $\sigma^2$ to $R$ with $C$ number of support points equally spaced. Hence, with probability $(1 - 1/C)$, $\delta \geq c$ for some positive constant $c$. 

By construction and the result from \citet{flatto2019dixie}, we know that $n C d\log d \geq  a_1 \geq n \sigma_{\min}$ for some positive constant $C$ with high probability. This also holds for $\hat{a}_1$. Hence, for some positive constant $C_1$ and $C_2$ ($C_1$ and $C_2$ has dependence on the dimension) , with high probability
\begin{align*}
    \Phi_{x,X,\hat{X}}(\mathcal{A}_{\text{cont}}) \geq \frac{C_1}{n\alpha^2 } \Gamma(p+1) \cos\left( \frac{(p+1)\pi}{2}\right)   \left( \frac{1}{\alpha C_2} \right)^{\frac{p}{\alpha}} \Gamma\left(1 - \frac{p}{\alpha}\right).
\end{align*}
This completes the proof. 
\end{proof}

\begin{remark} \label{rem:fintie_OU}
As we have characterized the finite-time distribution of a Lévy-driven OU process in Appendix~\ref{sec:finite:time:continuous}, it is clear to see that for any finite time $t$, if $\psi_i^{(t)}(u)$ denotes the characteristic function at that time then following the same procedure as that in Lemma~\ref{lem:dd_char_func}, 
\begin{align*}
    &\left|\psi_1^{(t)}(u) - \psi_2^{(t)}(u)\right| 
    \\
    &\leq \frac{2(\sigma_1+\sigma_2)\|u\|_{2}^{\alpha}}{n \alpha^2 \sigma_{\min}^2}\left(1 - (\alpha \sigma_{\min}t+1)e^{-\alpha \sigma_{\min}t} \right)  \exp\left(- \frac{\| u\|_{2}^{\alpha}}{\alpha \sigma_{\min}}\left(1 - e^{-\alpha \sigma_{\min}t}\right) \right).
\end{align*}
And hence, the algorithmic stability can be calculated in the similar way as that given in Theorem~\ref{thm:dd_main} for any time instance $t$. From here, it is hard to analyze the monotonic behavior of algorithmic stability for all time instance $t$. Here, we consider two interesting cases to discuss the monotone behavior:
\begin{itemize}

    \item When $t = O(\frac{1}{\alpha \sigma_{\min}})$ or higher but finite.  In this case, 
    \begin{align*}
        \left|\psi_1^{(t)}(u) - \psi_2^{(t)}(u)\right| \leq \frac{2(\sigma_1+\sigma_2)\|u\|_{2}^{\alpha}}{n \alpha^2 \sigma_{\min}^2}  \exp\left(- \frac{\| u\|_{2}^{\alpha}}{\alpha \sigma_{\min}}\left(1 - \frac{1}{e}\right) \right). 
    \end{align*}
    The above expression differs from the result in Lemma~\ref{lem:dd_char_func} only by a constant factor in the exponential. Hence, the stability bound will have similar monotonic behavious as that for $t \rightarrow \infty$.
    \item When $t$ is very small i.e. $t \ll \frac{1}{\alpha \sigma_{\min}} $ such that $1 - e^{-\alpha \sigma_{\min}t}\approx \alpha \sigma_{\min}t$. Then,
    \begin{align*}
       \left|\psi_1^{(t)}(u) - \psi_2^{(t)}(u)\right| \leq \frac{2(\sigma_1+\sigma_2)\|u\|_{2}^{\alpha} t}{n \alpha \sigma_{\min}}  \exp\left(- {\| u\|_{2}^{\alpha}t}  \right).
    \end{align*}
    
    In that case, we can easily see that under similar conditions in Theorem~\ref{thm:dd_main}, 
    \begin{align*}
        \varepsilon_{\text{stab}} (\mathcal{A}_{\text{cont}}) \leq  \frac{8R^p}{\pi} \frac{\sigma t^{\frac{p}{\alpha}}}{n\alpha^2 \sigma_{\min} } \Gamma(p+1) \cos\left( \frac{(p-1)\pi}{2}\right)   \Gamma\left(1-\frac{p}{\alpha}\right) = c(\alpha).
    \end{align*}
    We can similarly show here that there exist some $\alpha_0$ corresponding to every $t$ when  $c(\alpha)$ is monotonic in $[\alpha_0,2)$.
\end{itemize}
\end{remark}

\section{Theory and Proofs for the Discretized SDE}\label{ap:discretized_SDE}

In this section, we provide theoretical results
and their proofs of the discretized SDE \eqref{eq:sde_invar_1:discrete}-\eqref{eq:sde_invar_2:discrete}.

\begin{lemma}\label{lem:discretized_SDE_char}
For two matrices $X \in \mathbb{R}^{n\times d}$ and $\hat{X} \in \mathbb{R}^{n \times d}$ as defined earlier, the absolute value of difference between the characteristic functions of the anytime distributions for $\eta \leq \frac{1}{L}$ where $L$ is the maximum of largest eigenvalues of $\frac{1}{n}X^\top X$ and $\frac{1}{n}\hat{X}^\top\hat{X}$, at any $u \in \mathbb{R}^d$ corresponding to $d$-dimensional rotation invariant  processes in equations~\eqref{eq:discrete_char_1} and \eqref{eq:discrete_char_2}  is bounded as 
   \begin{align*}
    &\left|\psi_{\theta}(k,u)- \psi_{\hat{\theta}}(k,u)\right| 
    \\
    &\leq\frac{\eta^2 (\sigma_1 + \sigma_2)}{n (1 - \eta \sigma_{\min})}\frac{(k-1)(1-\eta \sigma_{\min})^{\alpha(k+1)}-k(1-\eta \sigma_{\min})^{\alpha k}+(1-\eta \sigma_{\min})^{\alpha}}{(1-(1-\eta\sigma_{\min})^{\alpha})^{2}}
    \\
    &\qquad\qquad\qquad\qquad\qquad\cdot
    \| u\|_2^{\alpha} \exp\left( - \frac{\eta(1-(1-\eta\sigma_{\min})^{k\alpha})}{1-(1-\eta\sigma_{\min})^{\alpha}}\|u  \|_2^\alpha \right),
\end{align*}
for any $k>0$, where $\sigma_{\min}$ is the smaller of the smallest of singular values of the matrices $\frac{1}{n}X^\top X$ and $\frac{1}{n}\hat{X}^\top \hat{X}$, and $x_i x_i^\top - \tilde{x}_i {\tilde{x}_i}^\top = \sigma_1 v_1 v_1^\top + \sigma_2 v_2 v_2^\top$ where $v_1$ and $v_2$ are orthogonal vectors.
\end{lemma}

\begin{proof}
For simplicity, we consider $\Sigma = I$ here. For general PSD sigma, similar steps can be followed as in Appendix~\ref{ap:general_sigma}. We can compute that
\begin{align}
    &\left|\psi_{\theta}(k,u)- \psi_{\hat{\theta}}(k,u)\right| 
    \notag\\
    &= \left| \exp\left( - \eta\sum_{j=0}^{k-1}\left\| \left(I-\frac{\eta}{n}\left({X}^\top {X}\right)\right)^{j}u \right\|_2^\alpha \right) - \exp\left( - \eta\sum_{j=0}^{k-1} \left\| \left(I-\frac{\eta}{n}\left(\hat{X}^\top \hat{X}\right)\right)^{j}u \right\|_2^\alpha \right) \right| \notag \\
    &\leq \underbrace{\exp\left( - \eta\sum_{j=0}^{k-1}\left\| \left(I-\frac{\eta}{n}\left({X}^\top {X}\right)\right)^{j}u \right\|_2^\alpha \right)}_{:=B} \notag \\
    &  \qquad \qquad \qquad \qquad \cdot\underbrace{\left| \eta\sum_{j=0}^{k-1} \left\| \left(I-\frac{\eta}{n}\left({X}^\top {X}\right)\right)^{j}u \right\|_2^\alpha - \eta\sum_{j=0}^{k-1} \left\| \left(I-\frac{\eta}{n}\left(\hat{X}^\top \hat{X}\right)\right)^{j}u \right\|_2^\alpha \right|}_{:=C}. \label{B:C:two:terms}
\end{align}
We first consider bounding the term $C$ in equation~\eqref{B:C:two:terms}. From Lemma~\ref{lem:ab_alpha}, we have for two positive numbers $a$ and $b$, and for some $1\leq \alpha \leq 2$, we have
\begin{align*}
    |a^\alpha - b^\alpha | \leq |a-b|(a^{\alpha-1}+b^{\alpha-1}).
\end{align*}
Utilizing the above result and triangle inequality, we have
\begin{align*}
   &\left| \eta\sum_{j=0}^{k-1} \left\| \left(I-\frac{\eta}{n}\left({X}^\top {X}\right)\right)^{j}u \right\|_2^\alpha - \eta\sum_{j=0}^{k-1} \left\| \left(I-\frac{\eta}{n}\left(\hat{X}^\top \hat{X}\right)\right)^{j}u \right\|_2^\alpha \right|  \\
    = & \eta \left|\sum_{j=0}^{k-1} \left[\left(\left\| \left(I-\frac{\eta}{n}\left({X}^\top {X}\right)\right)^{j}u - \left(I-\frac{\eta}{n}\left(\hat{X}^\top \hat{X}\right)\right)^{j}u \right\|_2\right) \right. \right. \\
    & \qquad \qquad \qquad \qquad \qquad  \left.\left. \cdot\left(  \left\| \left(I-\frac{\eta}{n}\left({X}^\top {X}\right)\right)^{j}u \right\|_2^{\alpha-1} +  \left\|\left(I-\frac{\eta}{n}\left(\hat{X}^\top \hat{X}\right)\right)^{j}u \right\|_2^{\alpha-1} \right) \right]  \right|.
\end{align*}
By definition, we have 
\begin{align*}
    \frac{1}{n} \left\| X^\top X u\right\|_2 \geq  \sigma_{\min}\| u\|_2, \quad\text{and}\quad
    \frac{1}{n} \left\| \hat{X}^\top \hat{X} u\right\|_2 \geq  \sigma_{\min}\| u\|_2.
\end{align*}
If $\eta \leq \frac{1}{L}$, where $L$ is the maximum of largest eigenvalues of $\frac{1}{n}X^\top X$ and $\frac{1}{n}\hat{X}^\top\hat{X}$, then,
\begin{align*}
    \left\| \left(I-\frac{\eta}{n}\left({X}^\top {X}\right)\right)^{j}u \right\|_2  \leq (1- \eta \sigma_{\min} )^j \| u\|_2.
\end{align*}
Similarly, 
\begin{align*}
    \left\| \left(I-\frac{\eta}{n}\left(\hat{X}^\top \hat{X}\right)\right)^{j}u \right\|_2  \leq (1- \eta \sigma_{\min} )^j \| u\|_2.
\end{align*}
For any two symmetric matrices $A$ and $B$ with $AB=BA$, we have $A^j - B^j = (A-B)(A^{j-1} + AB^{j-2} \cdots +  B^{j-1})$. 
It follows from the definitions that 
\begin{align*}
   X^\top X - \hat{X}^\top \hat{X} = x_i x_i^\top - \tilde{x}_i \tilde{x}_i^\top.
\end{align*}

By using similar argument as before, we have
\begin{align*}
    \left\| \left(I-\frac{\eta}{n}\left({X}^\top {X}\right)\right)^{j}u - \left(I-\frac{\eta}{n}\left(\hat{X}^\top \hat{X}\right)\right)^{j}u \right\|_2 \leq \frac{\eta j}{n} \left\|x_ix_i^\top - \tilde{x}_i \tilde{x}_i^\top\right\|_2 (1 - \eta \sigma_{\min})^{j-1}\|u\|_2.
\end{align*}
Note that $\tilde{x}_i {\tilde{a}_i}^\top  - x_i x_i^\top$ is a rank 2 matrix. Consider the two non-zero eigenvalues of this matrix are $\sigma_1$ and $\sigma_2$. Hence, $\tilde{x}_i {\tilde{x}_i}^\top  - x_i x_i^\top = \sigma_1 v_1 v_1^\top + \sigma_2 v_2 v_2^\top$ where $v_1$ and $v_2$ are the eigenvectors. Hence, we have obtained a bound on the term $C$ in equation~\eqref{B:C:two:terms} such that 
\begin{align*}
    C \leq \frac{\eta^2 (\sigma_1 + \sigma_2)}{n (1 - \eta \sigma_{\min})}\| u\|_2^{\alpha} \sum_{j=0}^{k-1}j(1 - \eta \sigma_{\min})^{j\alpha}.
\end{align*}
Now, let us consider bounding the term $B$ in equation~\eqref{B:C:two:terms}. Using previous arguments, 
\begin{align*}
    \exp\left( - \eta\sum_{j=0}^{k-1}\left\| \left(I-\frac{\eta}{n}\left({X}^\top {X}\right)\right)^{j}u \right\|_2^\alpha \right) \leq \exp\left( - \eta\sum_{j=0}^{k-1}  (1-\eta \sigma_{\min})^{j\alpha}\|u  \|_2^\alpha \right).
\end{align*}
Hence, we get, 
\begin{align*}
    &\left|\psi_{\theta}(k,u)- \psi_{\hat{\theta}}(k,u)\right| 
    \\
    &\leq \frac{\eta^2 (\sigma_1 + \sigma_2)}{n (1 - \eta \sigma_{\min})}\sum_{j=0}^{k-1}[j(1 - \eta \sigma_{\min})^{j\alpha}] \| u\|_2^{\alpha} \exp\left( - \eta\sum_{j=0}^{k-1}  (1-\eta \sigma_{\min})^{j\alpha}\|u  \|_2^\alpha \right)
    \\
    &=\frac{\eta^2 (\sigma_1 + \sigma_2)}{n (1 - \eta \sigma_{\min})}\frac{(k-1)(1-\eta \sigma_{\min})^{\alpha(k+1)}-k(1-\eta \sigma_{\min})^{\alpha k}+(1-\eta \sigma_{\min})^{\alpha}}{(1-(1-\eta\sigma_{\min})^{\alpha})^{2}}
    \\
    &\qquad\qquad\qquad\qquad\qquad\cdot
    \| u\|_2^{\alpha} \exp\left( - \frac{\eta(1-(1-\eta\sigma_{\min})^{k\alpha})}{1-(1-\eta\sigma_{\min})^{\alpha}}\|u  \|_2^\alpha \right),
\end{align*}
where we applied Lemma~\ref{lem:mean} and the proof is complete.
\end{proof}

In particular, by letting $k\rightarrow\infty$ in Lemma~\ref{lem:discretized_SDE_char}, we obtain the following corollary that concerns the stability
of the characteristic functions for the stationary distributions.
By denoting $\psi_{\theta}(u):=\psi_{\theta}(\infty,u)$
and $\psi_{\hat{\theta}}(u):=\psi_{\hat{\theta}}(\infty,u)$, 
we have the following result.

\begin{corollary}\label{cor:discretized_SDE_char}
Under the settings in Lemma~\ref{lem:discretized_SDE_char}, we have
\begin{align*}
    \left|\psi_{\theta}(u)- \psi_{\hat{\theta}}(u)\right| 
    \leq\frac{\eta^2 (\sigma_1 + \sigma_2)(1-\eta \sigma_{\min})^{\alpha-1}}{n(1-(1-\eta\sigma_{\min})^{\alpha})^{2}}
   \cdot
    \| u\|_2^{\alpha} \exp\left( - \frac{\eta}{1-(1-\eta\sigma_{\min})^{\alpha}}\|u  \|_2^\alpha \right).
\end{align*}
\end{corollary}

\begin{proof}
The results directly follows from Lemma~\ref{lem:discretized_SDE_char}
by letting $k\rightarrow\infty$ and using the results for sum of geometric series.
\end{proof}

\begin{theorem*}[Restatement of Theorem~\ref{thm:discrete_time_bound}] Consider  the $d$-dimensional loss function $f(x) =|\theta^\top x|^p$ such that $\theta, x \in \mathbb{R}^d$.  For any $x\sim P_X$ if $\|x\|_2 \leq R $,  for any $X$ sampled uniformly at random from the set $\mathcal{X}_n$, if  $\frac{1}{n} \| X^\top X u\|_2 \geq  \sigma_{\min} \|u\|_2$ for $u\in \mathbb{R}^d$ and for any two $X\cong \hat{X}$ sampled from $\mathcal{X}_n$   generating two stochastic process given by SDEs in equations~\eqref{eq:sde_invar_1:discrete_main} and \eqref{eq:sde_invar_2:discrete_main} for $\eta \leq \frac{1}{L}$ where $L$ is the maximum of largest eigenvalues of $\frac{1}{n}X^\top X$ and $\frac{1}{n}\hat{X}^\top\hat{X}$, $\|x_i x_i^\top - \tilde{x}_i \tilde{x}_i^\top \|_2  \leq 2\sigma$ holds  with high probability. Then, for $p\in [1,\alpha)$, we have
\begin{align*}
    \varepsilon_{\text{stab}} \leq \frac{2R^p}{\pi} \Gamma(p+1) \cos\left( \frac{(p-1)\pi}{2}\right) \frac{\sigma \eta^{1+\frac{p}{\alpha}} (1-\eta \sigma_{\min})^{\alpha-1}}{n\alpha(1-(1-\eta\sigma_{\min})^{\alpha})^{1+\frac{p}{\alpha}}}\Gamma\left(1-\frac{p}{\alpha}\right)
\end{align*}
with high probability.

\end{theorem*}
\begin{proof}
We have $d$-dimensional loss function for an $x \in \mathbb{R}^d$ sampled uniformly at random from $P_X$, $f(\theta) = |\theta^\top x|^p $ Let us denote the Fourier transform of $f$, $\mathcal{F}f(u)$ as $h(u)$. For an orthogonal matrix $A$ such that $Ae_1 = \frac{x}{\|x\|_2}$, we have from the results in Lemma~\ref{lem:ft_theta_x},  
\begin{align}
    h(Au) =  2 \|x\|_2^p (2\pi)^{d-1} \delta(u_2,\cdots,u_d) \Gamma(p+1) \cos\left( \frac{(p+1)\pi}{2}\right) \frac{1}{|u_1|^{p+1}} \quad\text{for~} p\in[1,2), \label{eq:fou_1_d:2}
\end{align}
and
\begin{align}
    h(Au) = 2\|x\|_2^p (2\pi)^{d-1} \delta(u_1,u_2,\cdots,u_d) \frac{2}{u_1^2} \quad\text{for~} p=2, \label{eq:fou_2_d:2}
\end{align}
where $\delta$ is the Dirac-delta function. Let us first consider the case when $p \in [1,2)$.  From equation~\eqref{eq:stab_char},
\begin{align*}
    &\varepsilon_{\text{stab}} (\mathcal{A}_{\text{cont}})
    \\
    &= \sup_{X\cong \hat{X}}\sup_{x \in \mathcal{X}}  \frac{1}{(2\pi)^d} \int_{\mathbb{R}^d} \left|\psi_{\theta}(u) - \psi_{\hat{\theta}}(u)\right| |h(u)|~\rmd u  \\
    & = \sup_{X\cong \hat{X}}\sup_{x \in \mathcal{X}} \frac{1}{(2\pi)^d} \int_{\mathbb{R}^d} \left|\psi_{\theta}(u) - \psi_{\hat{\theta}}(u)\right| \left| \int_{\mathbb{R}^d} |\theta^\top x|^p e^{\imagi  u^\top \theta}~\rmd \theta\right|~\rmd u \\
    &= \sup_{X\cong \hat{X}}\sup_{x \in \mathcal{X}} \frac{1}{(2\pi)^d} \int_{\mathbb{R}^d} \frac{\eta^2 (\sigma_1 + \sigma_2)(1-\eta \sigma_{\min})^{\alpha-1}}{n(1-(1-\eta\sigma_{\min})^{\alpha})^{2}}
   \cdot
    \| u\|_2^{\alpha} \exp\left( - \frac{\eta}{1-(1-\eta\sigma_{\min})^{\alpha}}\|u  \|_2^\alpha \right) \\
 & \qquad \qquad \qquad \qquad \qquad \qquad \qquad \cdot \left| \int_{\mathbb{R}^d} |\theta^\top x|^p e^{\imagi  u^\top \theta}~\rmd \theta \right|~\rmd u.
\end{align*}
In the above equation, let us apply the change of variable $u=Av$ and use result from Lemma~\ref{lem:ft_theta_x} (equations~\eqref{eq:fou_1_d:2}) and we get the following, 
\begin{align*}
  &\varepsilon_{\text{stab}}(\mathcal{A}_{\text{cont}})
  \\
  &=  \sup_{X\cong \hat{X}}\sup_{x \in \mathcal{X}} \frac{1}{(2\pi)^d} \int_{\mathbb{R}^d} \frac{\eta^2 (\sigma_1 + \sigma_2)(1-\eta \sigma_{\min})^{\alpha-1}}{n(1-(1-\eta\sigma_{\min})^{\alpha})^{2}}
   \cdot
    \| Av\|_2^{\alpha} \exp\left( - \frac{\eta \|Av  \|_2^\alpha}{1-(1-\eta\sigma_{\min})^{\alpha}} \right) \\
 & \qquad \qquad \qquad \qquad \qquad \qquad \qquad \cdot \left| \int_{\mathbb{R}^d} |\theta^\top x|^p e^{\imagi  (Av)^\top \theta}~\rmd \theta \right|~\rmd v \\
  &=  \sup_{X\cong \hat{X}}\sup_{x \in \mathcal{X}} \frac{1}{(2\pi)^d} \int_{\mathbb{R}^d} \frac{\eta^2 (\sigma_1 + \sigma_2)(1-\eta \sigma_{\min})^{\alpha-1}}{n(1-(1-\eta\sigma_{\min})^{\alpha})^{2}}
   \\
   &\qquad\qquad\qquad\qquad\qquad\qquad\cdot
    \| v\|_2^{\alpha} \exp\left( - \frac{\eta \|v  \|_2^\alpha}{1-(1-\eta\sigma_{\min})^{\alpha}} \right) \left| h(Av) \right|~\rmd v \\
  &= \sup_{X\cong \hat{X}}\sup_{x \in \mathcal{X}} \frac{1}{(2\pi)^d} \int_{\mathbb{R}^d} \left[ \frac{\eta^2 (\sigma_1 + \sigma_2)(1-\eta \sigma_{\min})^{\alpha-1}}{n(1-(1-\eta\sigma_{\min})^{\alpha})^{2}}
   \cdot
    \| v\|_2^{\alpha} \exp\left( - \frac{\eta \|v  \|_2^\alpha}{1-(1-\eta\sigma_{\min})^{\alpha}} \right) \right. ~ \\
  &\qquad \qquad \qquad \left.  \cdot\left(\left|  2 \|x\|_2^p (2\pi)^{d-1} \delta(v_2,\cdots,v_d) \Gamma(p+1) \cos\left( \frac{(p+1)\pi}{2}\right) \frac{1}{|v_1|^{p+1}} \right| \right)\right]~\rmd v \\
  &= \sup_{X\cong \hat{X}}\sup_{x \in \mathcal{X}} \frac{\|x\|_2^p}{\pi} \Gamma(p+1) \cos\left( \frac{(p-1)\pi}{2}\right) \frac{\eta^2 (\sigma_1 + \sigma_2)(1-\eta \sigma_{\min})^{\alpha-1}}{n(1-(1-\eta\sigma_{\min})^{\alpha})^{2}} \\
  &\qquad \qquad \qquad \qquad \qquad \qquad \qquad \qquad\cdot\int_{-\infty}^{\infty}   |v_1|^\alpha \exp\left(  \frac{-\eta|v_1|^\alpha}{1-(1-\eta\sigma_{\min})^{\alpha}}\right)  \frac{1}{|v_1|^{p+1}} ~\rmd v_1 \\
  & = \sup_{X\cong \hat{X}}\sup_{x \in \mathcal{X}} \frac{2\|x\|_2^p}{\pi} \Gamma(p+1) \cos\left( \frac{(p-1)\pi}{2}\right) \frac{\eta^2 (\sigma_1 + \sigma_2)(1-\eta \sigma_{\min})^{\alpha-1}}{n(1-(1-\eta\sigma_{\min})^{\alpha})^{2}} \\
  &\qquad \qquad \qquad \qquad \qquad \qquad \qquad \qquad\cdot\int_{0}^{\infty}   |v_1|^{\alpha-p-1} \exp\left(  \frac{-\eta|v_1|^{\alpha}}{1-(1-\eta\sigma_{\min})^{\alpha}}\right) ~\rmd v_1.
\end{align*}

In the above integral, by substituting $ \frac{\eta v_1^\alpha}{ 1- (1 - \eta \sigma_{\min})^{\alpha}}$ with $t$ so that
\begin{align}
    \rmd t =  v^{\alpha-1} \frac{\eta \alpha}{1-(1 - \eta \sigma_{\min})^{\alpha}}\rmd v_1,~ \text{and } \frac{1}{v^p} = \left( \frac{\eta}{1-(1-\eta \sigma_{\min})^{\alpha}}\right)^{\frac{p}{\alpha}} t^{-p/\alpha},
\end{align}
we have,
\begin{align}
 &\varepsilon_{\text{stab}} (\mathcal{A}_{\text{cont}})= \sup_{X\cong \hat{X}}\sup_{x \in \mathcal{X}} \frac{2\|x\|_2^p}{\pi} \Gamma(p+1) \cos\left( \frac{(p-1)\pi}{2}\right) \frac{\eta^2 (\sigma_1 + \sigma_2)(1-\eta \sigma_{\min})^{\alpha-1}}{n(1-(1-\eta\sigma_{\min})^{\alpha})^{2}}  \nonumber
 \\
 &\qquad\qquad\qquad\cdot \frac{1-(1-\eta \sigma_{\min})^{\alpha}}{\eta \alpha} \left( \frac{\eta}{1-(1-\eta \sigma_{\min})^{\alpha}}\right)^{\frac{p}{\alpha}} \int_{0}^{\infty} t^{-p/\alpha} e^{-t}~\rmd t \notag \\
 &= \sup_{X\cong \hat{X}}\sup_{x \in \mathcal{X}} \frac{2\|x\|_2^p}{\pi} \Gamma(p+1) \cos\left( \frac{(p-1)\pi}{2}\right) \frac{\eta^{1+\frac{p}{\alpha}} (\sigma_1 + \sigma_2)(1-\eta \sigma_{\min})^{\alpha-1}}{n\alpha(1-(1-\eta\sigma_{\min})^{\alpha})^{1+\frac{p}{\alpha}}}\Gamma\left(1-\frac{p}{\alpha}\right).
\end{align}
Now, if $\sigma$ is the upper bound on $\sigma_1$ and $\sigma_2$ for all $X\cong \hat{X} \in \mathcal{X}_n$ and $\|x\|_2 \leq R$ for $x\sim P_X$ with high probability then,
\begin{align}
    \varepsilon_{\text{stab}} \leq \frac{2R^p}{\pi} \Gamma(p+1) \cos\left( \frac{(p-1)\pi}{2}\right) \frac{\sigma \eta^{1+\frac{p}{\alpha}} (1-\eta \sigma_{\min})^{\alpha-1}}{n\alpha(1-(1-\eta\sigma_{\min})^{\alpha})^{1+\frac{p}{\alpha}}}\Gamma\left(1-\frac{p}{\alpha}\right).
\end{align}
This completes the proof.
\end{proof}

\section{Case for General P.S.D $\Sigma$ (Preconditioning)} \label{ap:general_sigma}
In this section, we would discuss the effect of general positive semidefinite matrix $\Sigma$. As in equations~\eqref{eq:sigma_theta_1} and \eqref{eq:sigma_theta_2}, 
we consider two SDEs corresponding to a rotationally symmetric $\alpha$-stable L\'{e}vy process $\Lm_{t}$ in $\mathbb{R}^d$,
\begin{align}
    \rmd \theta_t &= - \frac{1}{n} \left(X^\top X\right) \theta_t \rmd t + \Sigma \rmd \Lm_{t}, \label{eq:sigma_theta_1}   \\
    \rmd\hat{\theta}_t &= -\frac{1}{n} \left(\hat{X}^\top \hat{X}\right) \hat{\theta}_t \rmd t + \Sigma \rmd \Lm_{t}, \label{eq:sigma_theta_2}
\end{align}
where $\Sigma \in \mathbb{R}^{d\times d}$ is a real valued P.S.D matrix.  The corresponding characteristic functions are given by as in equations~\eqref{eq:char_sde_1_sigma} and \eqref{eq:char_sde_2_sigma} (see Lemma~\ref{lem:char}),
\begin{align}
    \psi_{\theta}(u) &= \exp\left( - \int_{0}^{\infty} \left\|\Sigma^\top e^{-s \frac{1}{n}({X}^\top {X})}u \right\|_2^\alpha \rmd s\right), \label{eq:char_sde_1_sigma} \\
    \psi_{\hat{\theta}}(u) &= \exp\left( - \int_{0}^{\infty} \left\|\Sigma^\top e^{-s \frac{1}{n}(\hat{X}^\top \hat{X})}u \right\|_2^\alpha \rmd s\right)\label{eq:char_sde_2_sigma}. 
\end{align}
We assume that the largest and smallest eigenvalues of the matrix $\Sigma$ is $\lambda_{\max}$ and $\lambda_{\min}$.

\begin{lemma} \label{lem:Sigma_Char_bound}
For two matrices $X \in \mathbb{R}^{n\times d}$ and $\hat{X} \in \mathbb{R}^{n \times d}$ as defined earlier, the absolute value of difference between the characteristic functions of the stationary distributions at any $u \in \mathbb{R}^d$ corresponding to $d$-dimensional rotation invariant  processes in equations~\eqref{eq:char_sde_1_sigma} and \eqref{eq:char_sde_2_sigma}  is bounded as 
   \begin{align*}
   |\psi_{\theta}(u)- \psi_{\hat{\theta}}(u)|  \leq \lambda_{\max}^{\alpha}\frac{2(\sigma_1+\sigma_2)\|u\|_{2}^{\alpha}}{n \alpha \sigma_{\min}}\exp\left(- \frac{\lambda_{\min}^{\alpha}\| u\|_{2}^{\alpha}}{\alpha^2 \sigma_{\min}^2} \right),
\end{align*}
where $\sigma_{\min}$ is the smaller of the smallest of singular values of the matrices $\frac{1}{n}X^\top X$ and $\frac{1}{n}\hat{X}^\top \hat{X}$, and $x_i x_i^\top - \tilde{x}_i {\tilde{x}_i}^\top = \sigma_1 v_1 v_1^\top + \sigma_2 v_2 v_2^\top$ where $v_1$ and $v_2$ are orthogonal vectors.
\end{lemma}

\begin{proof}
We can compute that
\begin{align}
    &|\psi_{\theta}(u)- \psi_{\hat{\theta}}(u)| 
    \notag\\
    &= \left| \exp\left(- \int_{0}^{\infty} \left\|\Sigma^\top e^{-s\frac{1}{n} X^\top X }u\right\|_{2}^{\alpha} \rmd s\right) - \exp\left(- \int_{0}^{\infty} \left\|\Sigma^\top  e^{-s\frac{1}{n} \hat{X}^\top\hat{X}  }u\right\|_{2}^{\alpha} \rmd s\right) \right| \notag \\
    &\leq \underbrace{\exp\left(- \int_{0}^{\infty} \left\|\Sigma^\top e^{-s\frac{1}{n} X^\top X }u\right\|_{2}^{\alpha} \rmd s\right)}_{:=B} \underbrace{\left| \int_{0}^{\infty} \left\|\Sigma^\top  e^{-s\frac{1}{n} X^\top X }u\right\|_{2}^{\alpha} \rmd s - \int_{0}^{\infty} \left\|\Sigma^\top e^{-s\frac{1}{n} \hat{X}^\top \hat{X}  }u\right\|_{2}^{\alpha} \rmd s \right|}_{:=C}. \notag
\end{align}
We first consider the term $C$ in the above equation. From Lemma~\ref{lem:ab_alpha}, we have for two positive numbers $a$ and $b$, and for some $1\leq \alpha \leq 2$, we have
\begin{align*}
    |a^\alpha - b^\alpha | \leq |a-b|(a^{\alpha-1}+b^{\alpha-1}).
\end{align*}
Now,
\begin{align*}
    C &= \left| \int_{0}^{\infty} \left\|\Sigma^\top e^{-s\frac{1}{n} X^\top X }u\right\|_{2}^{\alpha} \rmd s - \int_{0}^{\infty} \left\|\Sigma^\top e^{-s\frac{1}{n} \hat{X}^\top \hat{X} }u\right\|_{2}^{\alpha} \rmd s \right| \\
    &=\left| \int_{0}^{\infty} \left( \left\|\Sigma^\top e^{-s\frac{1}{n} X^\top X }u\right\|_{2}^{\alpha} - \left\|\Sigma^\top e^{-s\frac{1}{n} \hat{X}^\top \hat{X} }u\right\|_{2}^{\alpha} \right) \rmd s \right| \\
    &\leq \int_{0}^{\infty} \left| \left \|\Sigma^\top e^{-s\frac{1}{n} X^\top X }u\right\|_{2} - \left\|\Sigma^\top e^{-s\frac{1}{n} \hat{X}^\top \hat{X} }u\right\|_2 \right| 
    \\
    &\qquad\qquad\qquad\qquad\qquad\cdot\left( \left \|\Sigma^\top e^{-s\frac{1}{n} X^\top X }u\right\|_{2}^{\alpha-1} + \left\|\Sigma^\top e^{-s\frac{1}{n} \hat{X}^\top \hat{X} }u\right\|_2^{\alpha-1} \right) \rmd s \\
    &\leq \int_{0}^{\infty}  \left \|\Sigma^\top e^{-s\frac{1}{n} X^\top X }u - \Sigma^\top e^{-s\frac{1}{n} \hat{X}^\top \hat{X} }u\right\|_2  \left( \left \|\Sigma^\top e^{-s\frac{1}{n} X^\top X }u\right\|_{2}^{\alpha-1} + \left\|\Sigma^\top e^{-s\frac{1}{n} \hat{X}^\top \hat{X}}u\right\|_2^{\alpha-1} \right) \rmd s \\
    &= \int_{0}^{\infty}  \left \|\Sigma^\top e^{-s\frac{1}{n} X^\top X }\left(I -   e^{s\frac{1}{n} X^\top X - s\frac{1}{n} \hat{X}^\top \hat{X} }\right)u\right\|_2  
    \\
    &\qquad\qquad\qquad\qquad\qquad\cdot\left( \left \|\Sigma^\top e^{-s\frac{1}{n} X^\top X }u\right\|_{2}^{\alpha-1} + \left\|\Sigma^\top e^{-s\frac{1}{n} \hat{X}^\top \hat{X}}u\right\|_2^{\alpha-1} \right) \rmd s\\
    &\leq \lambda_{\max}^\alpha \underbrace{\int_{0}^{\infty}  \left \|e^{-s\frac{1}{n} X^\top X }\left(I - e^{s\frac{1}{n} X^\top X - s\frac{1}{n} \hat{X}^\top \hat{X} }\right)u\right\|_2  \left( \left \|e^{-s\frac{1}{n} X^\top X }u\right\|_{2}^{\alpha-1} + \left\|e^{-s\frac{1}{n} \hat{X}^\top \hat{X}}u\right\|_2^{\alpha-1} \right) \rmd s}_{\text{This term has been analyzed  as an upper bound on term C in Lemma~\ref{lem:dd_char_func} (Equation~\eqref{eq:upper_bnd_C}).}}.
\end{align*}
Using the result directly from equation~\eqref{eq:final_C_bound}, we have,
\begin{align}
    C \leq \lambda_{\max}^{\alpha}\frac{2(\sigma_1+\sigma_2)\|u\|_{2}^{\alpha}}{n \alpha^2 \sigma_{\min}^2}.
\end{align}
Next, let us consider the term $B$. Using the similar arguments as in Lemma~\ref{lem:dd_char_func} (equation~\eqref{eq:bound_B}), we have,
\begin{align}
    \exp\left(- \int_{0}^{\infty} \left\|\Sigma^\top e^{-s\frac{1}{n} X^\top X }u\right\|_{2}^{\alpha} \rmd s\right) &\leq \exp\left(- \lambda_{\min}^{\alpha}\| u\|_{2}^{\alpha} \int_{0}^{\infty} e^{-s\alpha \sigma_{\min} }  \rmd s\right) \nonumber
    \\
    &= \exp\left(- \frac{\lambda_{\min}^{\alpha}\| u\|_{2}^{\alpha}}{\alpha \sigma_{\min}} \right).
\end{align}
Hence, we have the final result, 
\begin{align}
    |\psi_{\theta}(u)- \psi_{\hat{\theta}}(u)|  \leq \lambda_{\max}^{\alpha}\frac{2(\sigma_1+\sigma_2)\|u\|_{2}^{\alpha}}{n \alpha^2 \sigma_{\min}^2}\exp\left(- \frac{\lambda_{\min}^{\alpha}\| u\|_{2}^{\alpha}}{\alpha \sigma_{\min}} \right),
\end{align}
which completes the proof.
\end{proof}

\begin{theorem}  Consider  the $d$-dimensional loss function $f(x) =|\theta^\top x|^p$ such that $\theta, x \in \mathbb{R}^d$.  For any $x\sim P_X$ if $\|x\|_2 \leq R $,  for any $X$ sampled uniformly at random from the set $\mathcal{X}_n$, if  $\frac{1}{n} \| X^\top X u\|_2 \geq  \sigma_{\min} \|u\|_2$ for $u\in \mathbb{R}^d$ and for any two $X\cong \hat{X}$ sampled from $\mathcal{X}_n$   generating two stochastic process given by SDEs in equations~\eqref{eq:sigma_theta_1} and \eqref{eq:sigma_theta_2}, $\|x_i x_i^\top - \tilde{x}_i \tilde{x}_i^\top \|_2  \leq 2\sigma$ holds  with high probability. Then,
\begin{enumerate}
    \item[(i)] For $\alpha \in (1,2)$, the algorithm is not stable when $p \in [\alpha,2]$ i.e. $\varepsilon_{\text{stab}}(\mathcal{A}_{\text{cont}})$ diverges. When $\alpha = p =2$ then with high probability $ \varepsilon_{\text{stab}} (\mathcal{A}_{\text{cont}})\leq \frac{2R^2}{\pi}  \frac{\lambda_{\max}^2\sigma}{n   \sigma_{\min}}$.
    \item[(ii)] For $p\in [1,\alpha)$, we have the following upper bound for the algorithmic stability,
    \begin{align*}
         &\varepsilon_{\text{stab}} (\mathcal{A}_{\text{cont}})
         \\
    &\leq   \frac{8R^p}{\pi} \lambda_{\min}^p \left( \frac{\lambda_{\max}}{\lambda_{\min}}\right)^{\alpha}\frac{\sigma}{n\alpha^2 \sigma_{\min} } \Gamma(p+1) \cos\left( \frac{(p-1)\pi}{2}\right)  \left(\frac{1}{\alpha \sigma_{\min}}\right)^{\frac{p}{\alpha}} \Gamma\left(1-\frac{p}{\alpha}\right)
    \\
    &= c(\alpha),
    \end{align*}
    which holds with high probability. Furthermore, for some $\alpha_0 > 1$, if we have $$\sigma_{\min} \geq \exp\left( 1 + \frac{4}{p} - \log \alpha_0 -  \phi\left( 1 - \frac{p}{\alpha_0}\right) - \alpha_0^{2}\log \left( \frac{\lambda_{\max}}{\lambda_{\min}}\right)\right),$$
    where $\phi$ is the digamma function, then the map $\alpha\rightarrow c(\alpha)$ is increasing for $\alpha \in [\alpha_0,2)$.
    \end{enumerate}
\end{theorem}

\begin{proof}
We have $d$-dimensional loss function for an $x \in \mathbb{R}^d$ sampled uniformly at random from $P_X$, $f(\theta) = |\theta^\top x|^p $ Let us denote the Fourier transform of $f$, $\mathcal{F}f(u)$ as $h(u)$. For an orthogonal matrix $A$ such that $Ae_1 = \frac{x}{\|x\|_2}$, we have from the results in Lemma~\ref{lem:ft_theta_x},  
\begin{align}
    h(Au) =  2 \|x\|_2^p (2\pi)^{d-1} \delta(u_2,\cdots,u_d) \Gamma(p+1) \cos\left( \frac{(p+1)\pi}{2}\right) \frac{1}{|u_1|^{p+1}} \quad\text{for~} p\in[1,2), \label{eq:fou_1_sig}
\end{align}
and
\begin{align}
    h(Au) = 2\|x\|_2^p (2\pi)^{d-1} \delta(u_1,u_2,\cdots,u_d) \frac{2}{u_1^2} \quad\text{for~} p=2, \label{eq:fou_2_sig}
\end{align}
where $\delta$ is the Dirac-delta function. Let us first consider the case when $p \in [1,2)$.  From equation~\eqref{eq:stab_char} and Lemma~\ref{lem:Sigma_Char_bound}, 
\begin{align*}
    \varepsilon_{\text{stab}} (\mathcal{A}_{\text{cont}})
    &= \sup_{X\cong \hat{X}}\sup_{x \in \mathcal{X}}  \frac{1}{(2\pi)^d} \int_{\mathbb{R}^d} |\psi_{\theta}(u) - \psi_{\hat{\theta}}(u)| |h(u)|~\rmd u  \\
    & = \sup_{X\cong \hat{X}}\sup_{x \in \mathcal{X}} \frac{1}{(2\pi)^d} \int_{\mathbb{R}^d} |\psi_{\theta}(u) - \psi_{\hat{\theta}}(u)| \left| \int_{\mathbb{R}^d} |\theta^\top x|^p e^{\imagi  u^\top \theta}~\rmd \theta\right|~\rmd u \\
    &= \sup_{X\cong \hat{X}}\sup_{x \in \mathcal{X}} \frac{1}{(2\pi)^d} \int_{\mathbb{R}^d} \lambda_{\max}^{\alpha}\frac{2(\sigma_1+\sigma_2)\|u\|_{2}^{\alpha}}{n \alpha^2 \sigma_{\min}^2}
    \\
    &\qquad\qquad\qquad\qquad\qquad\qquad\cdot\exp\left(- \frac{\lambda_{\min}^{\alpha}\| u\|_{2}^{\alpha}}{\alpha \sigma_{\min}} \right) \left| \int_{\mathbb{R}^d} |\theta^\top x|^p e^{\imagi  u^\top \theta}~\rmd \theta \right|~\rmd u.
\end{align*}
In the above equation, we make change of variable $u=Av$ and use the result from Lemma~\ref{lem:ft_theta_x} (equation~\eqref{eq:fou_1_sig}) to get the following, 
\begin{align*}
  \varepsilon_{\text{stab}} (\mathcal{A}_{\text{cont}})&=  \sup_{X\cong \hat{X}}\sup_{x \in \mathcal{X}} \frac{1}{(2\pi)^d} \int_{\mathbb{R}^d} \lambda_{\max}^{\alpha} \frac{2(\sigma_1+\sigma_2)\|Av\|_{2}^{\alpha}}{n \alpha^2 \sigma_{\min}^2}  
  \\
  &\qquad\qquad\qquad\qquad\qquad\qquad\cdot\exp\left(- \frac{ \lambda_{\min}^{\alpha} \| Av\|_{2}^{\alpha}}{\alpha \sigma_{\min}} \right) 
\left| \int_{\mathbb{R}^d} |\theta^\top x|^p e^{\imagi  (Av)^\top \theta}~\rmd \theta \right|~\rmd v \\
  &=  \sup_{X\cong \hat{X}}\sup_{x \in \mathcal{X}} \frac{1}{(2\pi)^d} \int_{\mathbb{R}^d} \lambda_{\max}^{\alpha} \frac{2(\sigma_1+\sigma_2)\|v\|_{2}^{\alpha}}{n \alpha^2 \sigma_{\min}^2}  \exp\left(- \frac{\lambda_{\min}^{\alpha} \| v\|_{2}^{\alpha}}{\alpha \sigma_{\min}} \right) \left| h(Av) \right|~\rmd v \\
  &= \sup_{X\cong \hat{X}}\sup_{x \in \mathcal{X}} \frac{1}{(2\pi)^d} \int_{\mathbb{R}^d} \left[ \left( \lambda_{\max}^{\alpha} \frac{2(\sigma_1+\sigma_2)\|v\|_{2}^{\alpha}}{n \alpha^2 \sigma_{\min}^2}  \exp\left(- \frac{\lambda_{\min}^{\alpha} \| v\|_{2}^{\alpha}}{\alpha \sigma_{\min}} \right) \right) \right. ~ \\
  &\qquad \qquad \left.  \cdot\left(\left|  2 \|x\|_2^p (2\pi)^{d-1} \delta(v_2,\cdots,v_d) \Gamma(p+1) \cos\left( \frac{(p+1)\pi}{2}\right) \frac{1}{|v_1|^{p+1}} \right| \right)\right]~\rmd v \\
  &= \sup_{X\cong \hat{X}}\sup_{x \in \mathcal{X}} \frac{2\|x\|_2^p}{\pi} \Gamma(p+1) \cos\left( \frac{(p-1)\pi}{2}\right) \frac{\lambda_{\max}^{\alpha}(\sigma_1+\sigma_2)}{n\alpha^2 \sigma_{\min}^2} \\
  &\qquad \qquad \qquad \qquad \qquad \qquad \qquad \qquad\cdot\int_{-\infty}^{\infty}   |v_1|^\alpha \exp\left(  \frac{-\lambda_{\min}^{\alpha} |v_1|^\alpha}{\alpha \sigma_{\min}}\right)  \frac{1}{|v_1|^{p+1}} ~\rmd v_1 \\
  & = \sup_{X\cong \hat{X}}\sup_{x \in \mathcal{X}} \frac{4\|x\|_2^p}{\pi} \frac{\lambda_{\max}^{\alpha}(\sigma_1+\sigma_2)}{n\alpha^2 \sigma_{\min}^2} \Gamma(p+1) \cos\left( \frac{(p-1)\pi}{2}\right)\\
  &\qquad \qquad \qquad \qquad \qquad \qquad \qquad \qquad   \cdot\int_{0}^{\infty} v_1^{\alpha-p-1} \exp\left(  \frac{-\lambda_{\min}^{\alpha} v_1^\alpha}{\alpha \sigma_{\min}}\right)~\rmd v_1.
\end{align*}
In the above integral, by substituting $\frac{\lambda_{\min}^{\alpha}v^{\alpha}}{\alpha \sigma_{\min}}$ with $t$ so that
\begin{align}
    \rmd t =  \lambda_{\min}^{\alpha}v^{\alpha-1} \frac{1}{\sigma_{\min}}\rmd v,~ \text{and } \frac{1}{v^p} = \lambda_{\min}^{p}\left( \frac{1}{\alpha \sigma_{\min}}\right)^{\frac{p}{\alpha}} t^{-p/\alpha},
\end{align}
we have,
\begin{align}
 \varepsilon_{\text{stab}} (\mathcal{A}_{\text{cont}})&= \sup_{X\cong \hat{X}}\sup_{x \in \mathcal{X}} \frac{4\|x\|_2^p}{\pi} \lambda_{\min}^p\left(\frac{\lambda_{\max}}{\lambda_{\min}}\right)^{\alpha}\frac{(\sigma_1+\sigma_2)}{n\alpha^2\sigma_{\min}  } \nonumber
 \\
 &\qquad\qquad\qquad\cdot\Gamma(p+1) \cos\left( \frac{(p-1)\pi}{2}\right)  \left(\frac{1}{\alpha \sigma_{\min}}\right)^{\frac{p}{\alpha}} \int_{0}^{\infty} t^{-p/\alpha} e^{-t}~\rmd t. \label{eq:p_1_2_d_sigma}
\end{align}
It is clear that, the above integral diverge for $p\geq \alpha$, hence the algorithm is not stable for $p \in [1,2)$. Now, we check the case for $p=2$. For $p=2$, we have,
\begin{align*}
    \varepsilon_{\text{stab}} (\mathcal{A}_{\text{cont}})
    &= \sup_{X\cong \hat{X}}\sup_{x \in \mathcal{X}} \frac{1}{(2\pi)^d} \int_{\mathbb{R}^d} \lambda_{\max}^{\alpha}\frac{2(\sigma_1+\sigma_2)\|u\|_{2}^{\alpha}}{n \alpha^2 \sigma_{\min}^2}  
    \\
    &\qquad\qquad\qquad\qquad\qquad\cdot\exp\left(- \frac{\lambda_{\min}^{\alpha}\| u\|_{2}^{\alpha}}{\alpha \sigma_{\min}} \right) \left| \int_{\mathbb{R}^d} |\theta^\top x|^2 e^{\imagi  u^\top \theta}~\rmd \theta \right|~\rmd u.
\end{align*}
In the above equation, we make change of variable $u=Av$ and use the result from Lemma~\ref{lem:ft_theta_x} (equation~\eqref{eq:fou_2_sig}) to get the following,
\begin{align*}
  &\varepsilon_{\text{stab}} (\mathcal{A}_{\text{cont}})
  \\
  &=  \sup_{X\cong \hat{X}}\sup_{x \in \mathcal{X}}\Bigg\{ \frac{1}{(2\pi)^d} \int_{\mathbb{R}^d} \lambda_{\max}^{\alpha}\frac{2(\sigma_1+\sigma_2)\|Av\|_{2}^{\alpha}}{n \alpha^2 \sigma_{\min}^2}  
  \\
  &\qquad\qquad\qquad\qquad\qquad\qquad\cdot\exp\left(- \frac{\lambda_{\min}^{\alpha}\| Av\|_{2}^{\alpha}}{\alpha \sigma_{\min}} \right) \left| \int_{\mathbb{R}^d} |\theta^\top x|^2 e^{\imagi  (Av)^\top \theta}~\rmd \theta \right|~\rmd v\Bigg\} \\
  &=  \sup_{X\cong \hat{X}}\sup_{x \in \mathcal{X}} \frac{1}{(2\pi)^d} \int_{\mathbb{R}^d} \lambda_{\max}^{\alpha}\frac{2(\sigma_1+\sigma_2)\|v\|_{2}^{\alpha}}{n \alpha^2 \sigma_{\min}^2}  \exp\left(- \frac{\lambda_{\min}^{\alpha}\| v\|_{2}^{\alpha}}{\alpha \sigma_{\min}} \right) \left| h(Av) \right|~\rmd v  \\
  &= \sup_{X\cong \hat{X}}\sup_{x \in \mathcal{X}} \frac{2}{\pi} \int_{\mathbb{R}^d} \lambda_{\max}^{\alpha}\frac{(\sigma_1+\sigma_2)\|v\|_{2}^{\alpha}}{n \alpha^2 \sigma_{\min}^2}  \exp\left(- \frac{\lambda_{\min}^{\alpha}\| v\|_{2}^{\alpha}}{\alpha \sigma_{\min}} \right) \|x\|_2^2  \delta(v_1,v_2,\cdots,v_d) \frac{2}{v_1^2} ~\rmd v.
\end{align*}
In the last equation, we used the result from Lemma~\ref{lem:ft_theta_x}.
The above integral clearly diverges for $\alpha < 2$. However, when $\alpha =2$, then
\begin{align*}
    \varepsilon_{\text{stab}} (\mathcal{A}_{\text{cont}})&\leq \frac{\|x\|_2^2}{\pi}  \frac{\lambda_{\max}^{2}(\sigma_1+\sigma_2)}{n   \sigma_{\min}^2} . 
\end{align*}

Now, if $\sigma$ is the upper bound on $\sigma_1$ and $\sigma_2$ for all $X\cong \hat{X} \in \mathcal{X}_n$ and $\|x\|_2 \leq R$ for $x\sim P_X$ with high probability then, 
\begin{align*}
    \varepsilon_{\text{stab}} (\mathcal{A}_{\text{cont}})&\leq \frac{2R^2}{\pi}  \frac{\lambda_{\max}^2\sigma}{n   \sigma_{\min}^2},
\end{align*}
holds with high probability. This proves part (i) of our claim. 

Next, we will prove part (ii) when $p< \alpha$. We have from equation~\eqref{eq:p_1_2_d_sigma},
\begin{align*}
 &\varepsilon_{\text{stab}} (\mathcal{A}_{\text{cont}})
 \\
 &= \sup_{X\cong \hat{X}}\sup_{x \in \mathcal{X}} \Bigg\{\frac{4\|x\|_2^p}{\pi} \lambda_{\min}^p\left( \frac{\lambda_{\max}}{\lambda_{\min}}\right)^{\alpha} \frac{(\sigma_1+\sigma_2)}{n\alpha^2 \sigma_{\min} } \Gamma(p+1)
 \\
 &\qquad\qquad\qquad\qquad\qquad\qquad\cdot\cos\left( \frac{(p-1)\pi}{2}\right)  \left(\frac{1}{\alpha \sigma_{\min}}\right)^{\frac{p}{\alpha}} \int_{0}^{\infty} t^{-p/\alpha} e^{-t}~\rmd t\Bigg\} \\
 &=\sup_{X\cong \hat{X}}\sup_{x \in \mathcal{X}} \frac{4\|x\|_2^p}{\pi} \lambda_{\min}^p\left( \frac{\lambda_{\max}}{\lambda_{\min}}\right)^{\alpha} \frac{(\sigma_1+\sigma_2)}{n\alpha^2\sigma_{\min} } \Gamma(p+1) \cos\left( \frac{(p-1)\pi}{2}\right)  \left(\frac{1}{\alpha \sigma_{\min}}\right)^{\frac{p}{\alpha}} \Gamma\left(1-\frac{p}{\alpha}\right).
\end{align*}
Now, if $\sigma$ is the upper bound on $\sigma_1$ and $\sigma_2$ for all $X\cong \hat{X} \in \mathcal{X}_n$ and $\|x\|_2 \leq R$ for $x\sim P_X$ with high probability then,
\begin{align*}
 \varepsilon_{\text{stab}} (\mathcal{A}_{\text{cont}})&=  \frac{8R^p}{\pi} \lambda_{\min}^p\left( \frac{\lambda_{\max}}{\lambda_{\min}}\right)^{\alpha}\frac{\sigma}{n\alpha^2 \sigma_{\min} } \Gamma(p+1) \cos\left( \frac{(p-1)\pi}{2}\right)  \left(\frac{1}{\alpha \sigma_{\min}}\right)^{\frac{p}{\alpha}} \Gamma\left(1-\frac{p}{\alpha}\right)
\end{align*}
holds with high probability.   Now, consider the function,
\begin{align*}
    \Lambda(\alpha ) =  \frac{1}{\alpha^2}\left( \frac{\lambda_{\max}}{\lambda_{\min}}\right)^{\alpha} \left( \frac{1}{\alpha \sigma_{\min} }\right)^{\frac{p}{\alpha}} \Gamma \left( 1 - \frac{p}{\alpha}\right).
\end{align*}
We can compute that
\begin{align*}
    \partial_{\alpha} \log \Lambda(\alpha) =  \log \left( \frac{\lambda_{\max}}{\lambda_{\min}}\right)+ \frac{p}{\alpha^2} \left[ \log \alpha  + \log \sigma_{\min} -1 - \frac{2 \alpha}{p} +  \phi\left( 1 - \frac{p}{\alpha}\right)  \right], 
\end{align*}
where $\phi$ is the digamma function. For any arbitrary $\alpha_0$, if we choose $$\sigma_{\min} \geq \exp\left( 1 + \frac{2}{p} - \log \alpha_0 -  \phi\left( 1 - \frac{p}{\alpha_0}\right) - \alpha_0^{2}\log \left( \frac{\lambda_{\max}}{\lambda_{\min}}\right)\right),$$ then $\partial_{\alpha} \log \Lambda(\alpha) >0 $ for  $\alpha \in [\alpha_0,2)$. Hence, for all $\alpha_1,\alpha_2 \in [\alpha_0,2)$,  $\alpha_1 < \alpha_2$ it follows that $\Lambda(\alpha_1) \leq \Lambda(\alpha_2)$. This proves that $c(\alpha)$ is an increasing map in $\alpha$. This completes the proof.
\end{proof}

\section{Useful Results} \label{ap:useful_results}
{ Here below, we provide a few technical results which are used in the proofs of the main results.} 
\begin{lemma} \label{lem:ab_alpha}
For any two positive numbers $a$ and $b$, and for some $0 < \alpha \leq 2$, we have
\begin{align}\label{eqn:RHS}
    |a^\alpha - b^\alpha | \leq |a-b|(a^{\alpha-1}+b^{\alpha-1}).
\end{align}
\end{lemma}
\begin{proof}
When $a=b>0$, the result is obviously true. Without loss of generality, let us assume that $a> b>0$ and by considering the RHS of \eqref{eqn:RHS}, we get 
\begin{align*}
    |a-b|(a^{\alpha-1}+b^{\alpha-1}) &=(a-b)(a^{\alpha-1}+b^{\alpha-1}) \\
    &=a^{\alpha} + ab^{\alpha-1} - a^{\alpha-1}b - b^{\alpha} \\
    &= |a^\alpha - b^\alpha |  + ab^{\alpha-1} - a^{\alpha-1}b.
\end{align*}
Since, we have assumed that $a>b>0$ and $\alpha > 0$, hence $ab^{\alpha-1} - a^{\alpha-1}b > 0$ always which essentially means, 
\begin{align*}
    |a^\alpha - b^\alpha | \leq |a-b|(a^{\alpha-1}+b^{\alpha-1}).
\end{align*}
Same argument can be given while assuming $b > a> 0$.
This completes the proof.
\end{proof}

{ 
\begin{lemma}\label{lem:mean}
For any $a>0$, and $k\in\mathbb{N}$, 
\begin{equation*}
\sum_{j=0}^{k-1}ja^{j}
=\frac{(k-1)a^{k+1}-ka^{k}+a}{(a-1)^{2}}.
\end{equation*}
In particular, for any $0<a<1$, 
\begin{equation*}
\sum_{j=0}^{\infty}ja^{j}
=\frac{a}{(a-1)^{2}}.
\end{equation*}
\end{lemma}

\begin{proof}
We can compute that
\begin{equation*}
\sum_{j=0}^{k-1}ja^{j}
=a\sum_{j=1}^{k-1}ja^{j-1}
=a\frac{d}{da}\sum_{j=1}^{k-1}a^{j}
=a\frac{d}{da}\left(\frac{a^{k}-a}{a-1}\right)=\frac{(k-1)a^{k+1}-ka^{k}+a}{(a-1)^{2}}.
\end{equation*}
The proof is complete.
\end{proof}
}

\begin{lemma}[Fourier transform of $|\theta^\top x|^p$] \label{lem:ft_theta_x}
Consider the function $f(\theta) = |\theta^\top x|^p$ for $p \in [1,2]$ and $h(u)$ denotes the Fourier transform of $f(\theta)$ where $u= [u_1,\cdots,u_d]$ is a vector in $d$-dimension. Given an unitary matrix $A \in \mathbb{R}^{d\times d}$ such that $A^\top A = AA^\top = I$ where $I$ is an identity matrix in $\mathbb{R}^{d\times d}$ and $Ae_1 = \frac{x}{\|x\|_2}$ where $e_i$ is vector in $\mathbb{R}^{d}$ with all entries set to $0$ except $i$th entry which is set to 1, we have
\begin{align*}
    h(Au) =  2 \|x\|_2^p (2\pi)^{d-1} \delta(u_2,\cdots,u_d) \Gamma(p+1) \cos\left( \frac{(p+1)\pi}{2}\right) \frac{1}{|u_1|^{p+1}} \quad\text{for~} p\in[1,2),
\end{align*}
and
\begin{align*}
    h(Au) = 2\|x\|_2^p (2\pi)^{d-1} \delta(u_1,u_2,\cdots,u_d) \frac{2}{u_1^2} \quad\text{for~} p=2,
\end{align*}
where $\delta$ is the Dirac-delta function. 
\end{lemma}

\begin{proof}
We recall that the Fourier transform is given by
\begin{align*}
    \mathcal{F}f(u) = \int_{\mathbb{R}^d} f(\theta) e^{- \imagi   u^\top \theta} \rmd \theta. 
 \end{align*}
 Let
 \begin{align*}
     h(u) :=\mathcal{F}[|\langle x,\cdot \rangle|^p]  = \| x\|_2^p \mathcal{F}\left[\left|\left\langle \frac{x}{\|x\|_2},\cdot \right\rangle\right|^p\right].
 \end{align*}
 We consider now an unitary matrix $A \in \mathbb{R}^{d\times d}$ such that $A^\top A = AA^\top = I$ where $I$ is an identity matrix in $\mathbb{R}^{d\times d}$ and $Ae_1 = \frac{x}{\|x\|_2}$ where $e_i$ is vector in $\mathbb{R}^{d}$ with all entries set to $0$ except $i$th entry which is set to 1. Now let us compute   $h(Au)$.
 \begin{align*}
     h(Au) &=  \| x\|_2^p \int_{\mathbb{R}^d} \left|\left\langle \frac{x}{\|x\|_2},\theta \right\rangle\right|^p e^{-\imagi  (Au)^\top \theta} ~\rmd \theta \\
     &= \| x\|_2^p \int_{\mathbb{R}^d} \left|\left\langle Ae_1,\theta  \right\rangle\right|^p e^{-\imagi  (Au)^\top \theta} ~\rmd \theta.
 \end{align*}
 In the above integral we substitute, $\beta = A^\top \theta$. Hence, when $p \in [1,2)$, we have
 \begin{align*}
h(Au) &= \|x\|_2^p \int_{\mathbb{R}^d} |\langle e_1, \beta\rangle|^p e^{-\imagi  u^\top \beta}~d\beta   \\
&= \|x\|_2^p (2\pi)^{d-1} \delta(u_2,\cdots,u_d) \int_{-\infty}^{\infty} |\beta_1|^p ~e^{-\imagi  u_1  \beta_1}~d\beta_1 \\
&= \|x\|_2^p (2\pi)^{d-1} \delta(u_2,\cdots,u_d) \int_{0}^{\infty} \left(e^{-\imagi  u_1 \beta_1}+e^{\imagi  u_1 \beta_1} \right)\beta_1^p~d\beta_1  \\
& = 2 \|x\|_2^p (2\pi)^{d-1} \delta(u_2,\cdots,u_d) \Gamma(p+1) \cos\left( \frac{(p+1)\pi}{2}\right) \frac{1}{|u_1|^{p+1}}.
 \end{align*}
 When $p=2$, we have
 \begin{align}
     h(Au) &= \|x\|_2^p \int_{\mathbb{R}^d} |\langle e_1, \beta\rangle|^2 e^{-\imagi  u^\top \beta}~d\beta  \notag  \\
&= \|x\|_2^p (2\pi)^{d-1} \delta(u_2,\cdots,u_d) \int_{-\infty}^{\infty} |\beta_1|^2 ~e^{-\imagi  u_1  \beta_1}~d\beta_1  \notag \\
&= 2\|x\|_2^p (2\pi)^{d-1} \delta(u_1,u_2,\cdots,u_d) \frac{2}{u_1^2}.
 \end{align}
This completes the proof. 
\end{proof}

\section{Further Details on Experiment Settings and Resources}\label{ap:further_expts}

This section contains further details regarding the experiments presented in the main paper. As the synthetic data experiment setting was fully described in the text, most of the information below will pertain to the real data experiments {with the exception of additional synthetic data results that include mean estimates}. See the accompanying code regarding the implementation of the experiments described.

{ 
\subsection{Additional synthetic data results}

In addition to median and interquartile range based results presented in the paper, we add the following results in Figure \ref{fig:synth_data_mean} with a robust mean estimate of the results, demonstrating a similar pattern to that observed in the main paper.  

\begin{figure}[h]
    \centering
    \includegraphics[width=0.99\textwidth]{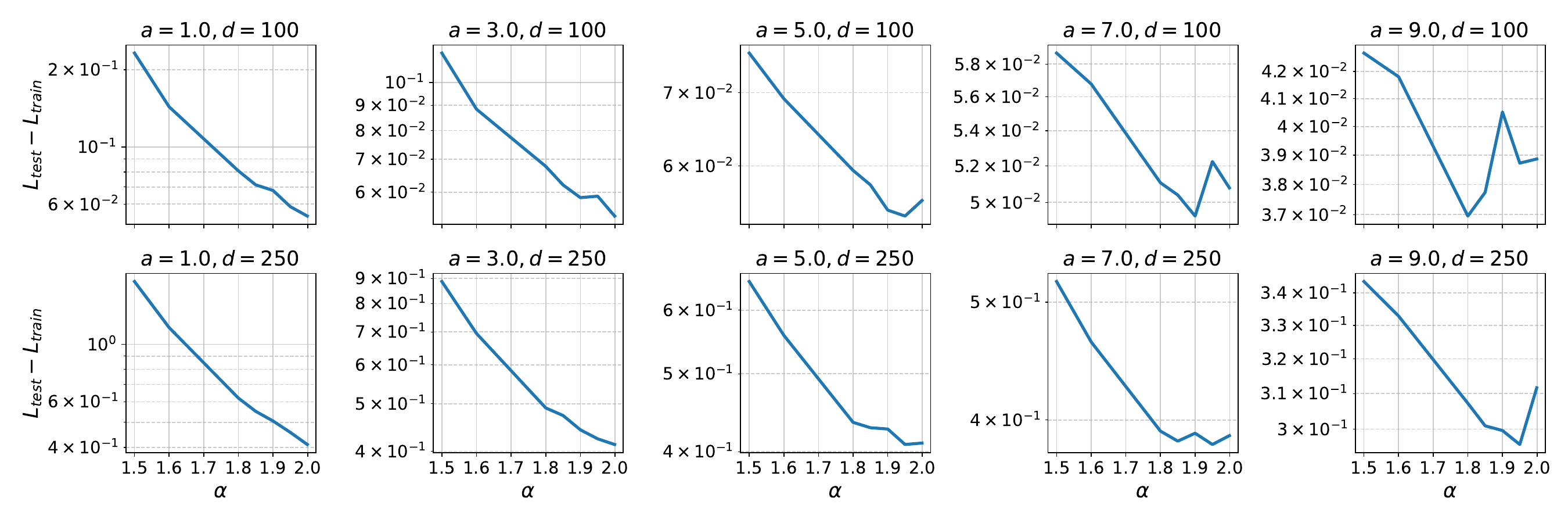}
    \caption{{ Results of the synthetic data experiments with varying $a$, $\alpha$, and $d$. Each experiment was repeated $500$ times with $n=1000$. The lines correspond to the robust mean estimates with the samples with losses above the 90\% quantile removed from among the respective experiments.}}
    \label{fig:synth_data_mean}
\end{figure}}

\subsection{Datasets}
The real data experiments involved a supervised learning setting, where images are classified into a number of predefined class labels. Each model architecture with given hyperparameters were trained on MNIST \citep{mnist2010}, CIFAR10, and CIFAR100 \citep{cifar102009} data sets\footnote{MNIST and CIFAR10/100 data sets have been shared under Creative Commons Attribution-Share Alike 3.0 license and MIT License respectively.}. The MNIST data set includes $28\times 28$ black and white handwritten digits, with digits ranging from $0$ to $9$. The data set in its original form includes $60000$ training and $10000$ test samples. CIFAR10 and CIFAR100 are also image classification dataset comprising  $32\times 32$ color images of objects or animals, making up $10$ and $100$ classes respectively. There are $50000$ training and $10000$ test images in either of these data sets, and the instances are divided among classes equally. We used the standard train-test splits in all data sets.

\subsection{Models}

We used three different architectures in our experiments: a fully connected network with 4 hidden layers (FCN4), another fully connected network with 6 hidden layers (FCN6), and a convolutional neural network (CNN). In both FCN architectures, all hidden layer widths were $2048$. All architectures featured ReLU activation functions. Batch normalization, dropout, residual layers, or any explicit regularization term in the loss function were not used in any part of the experiments. The architecture we chose for our CNN model closely follows that of VGG11 model \citep{simonyanVeryDeepConvolutional2015}, with the significant difference that only a single linear layer with a softmax output follows the convolutional layers presented below:

$$
64, M, 128, M, 256, 256, M, 512, 512, M, 512, 512, M.
$$

Here, integers describe the number of filters for 2-dimensional convolutional layers - for which the kernel sizes are $3 \times 3$. $M$ stands for $2 \times 2$ max-pooling operations with a stride value of $2$. This architecture was slightly modified for the MNIST experiments by removing the first max-pooling layer due to the smaller dimensions of the MNIST images. The Table~\ref{tab:parameter_count} describes the number of different parameters used for each model-dataset combination.
\begin{table}[h]
\centering
\begin{tabular}{ |c| c| c| c|}
\hline
 & FCN4 & FCN6 & CNN \\ 
 \hline
 MNIST & 14,209,024 & 22,597,632 & 9,221,696 \\
 \hline
 CIFAR10 & 18,894,848 & 27,283,456 & 9,222,848\\
 \hline
 CIFAR100 & 18,899,456 & 27,288,064 & 9,227,456\\
 \hline
\end{tabular}
\vspace{5pt}
\caption{\label{tab:parameter_count}Number of parameters for model-dataset combinations.}
\end{table}

\subsection{Training and hyperparameters}
As described in the main text, the models were trained with SGD until convergence on the training set. The convergence criteria for \textsc{MNIST} and \textsc{CIFAR-10} is a training negative log-likelihood (NLL) of $<5 \times 10^{-5}$ and a training accuracy of $100\%$, and for \textsc{CIFAR-100} these are a NLL of $<1 \times 10^{-2}$ and a training accuracy of $>99\%$. We use two different batch sizes ($b = 50, 100$) and a diversity of learning rates ($\eta$) to generate a large range of $\eta/b$ values. Table~\ref{tab:eta_b_ranges} presents the $\eta/b$ values created for each experiment setting. The varying nature of these ranges are due to the fact that different $\eta/b$ values might lead to heavy-tailed behavior or divergence under different points in this hyperparameter space. Source code includes the enumerations of specific combinations of these hyperparameters for all settings.

\begin{table}[h]
\centering
\begin{tabular}{ |c| c| c| c|}
\hline
 & FCN4 & FCN6 & CNN \\ 
 \hline
  MNIST & $5 \times 10^{-5}$ to  $1.14\times 10^{-2}$  & $5 \times 10^{-5}$ to  $8.8\times 10^{-3}$ & $1 \times 10^{-5}$ to  $6.35\times 10^{-3}$ \\
 \hline
 CIFAR10 & $5 \times 10^{-5}$ to  $2.7\times 10^{-3}$  & $2.5 \times 10^{-5}$ to  $4\times 10^{-3}$ & $1 \times 10^{-5}$ to  $1.5\times 10^{-3}$\\
 \hline
 CIFAR100 & $1 \times 10^{-5}$ to  $1.6\times 10^{-3}$  & $1 \times 10^{-5}$ to  $2.25\times 10^{-3}$ & $1 \times 10^{-5}$ to  $7\times 10^{-4}$\\
 \hline

\end{tabular}
\vspace{5pt}
\caption{\label{tab:eta_b_ranges}The ranges of $\eta/b$ for all experiments.}
\end{table}

\subsection{Tail-index estimation}

The multivariate estimator proposed by \cite{mohammadi2015estimating} was used for tail-index estimation:

\begin{theorem}[{\cite[Corollary 2.4]{mohammadi2015estimating}}]
Let $\{X_i\}_{i=1}^K$ be a collection of i.i.d.\ random vectors where each $X_i$ is multivariate strictly stable with tail-index $\alpha$, and $K = K_1 \times K_2$.
Define $Y_i := \sum_{j=1}^{K_1} X_{j+(i-1)K_1} \>$ for $i \in \{ 1,\dots, K_2\} $. Then, the estimator
\begin{align}
\label{eqn:alpha_estim}
\widehat{\phantom{a}\frac1{\alpha}\phantom{a}} \hspace{-4pt} \triangleq \hspace{-2pt} \frac1{\log K_1} \Bigl(\frac1{K_2 } \sum_{i=1}^{K_2} \log \|Y_i\|  - \frac1{K} \sum_{i=1}^K \log \|X_i\| \Bigr) 
\end{align}
converges to $1/{\alpha}$ almost surely, as $K_2 \rightarrow \infty$.
\end{theorem}

Previous deep learning research such as \citet{tzagkarakis2018compressive, csimcsekli2019heavy, barsbey2021heavy} have also used this estimator. As described in the main text, tail-index estimation is conducted on the ergodic averaged version of the parameters, an operation which does not change the tail-index of the parameters, to conform to this estimator's assumptions. We use the columns of parameters in FCN's and specific filter parameters in CNN as the random vectors instances for the multivariate distribution. Before conducting the tail-index estimation we center the parameters using the index-wise median values. We observe that (i) centering with mean values, and/or (ii) using the alternative univariate tail-index estimator \citep[Corollary 2.2]{mohammadi2015estimating} from the same paper produces qualitatively identical results. We also observe that using alternative tail index estimators with symmetric $\alpha$-stable assumption produces no qualitatively significant differences in the estimated values \citep{satheEstimationParameters2020}.

\subsection{Hardware and software resources}
The computational resources for the experiments were provided by a research institute. The bulk of the resources were expended on the real data experiments, where a roughly equal division of labor between Nvidia Titan X, 1080 Ti, and 1080 model GPU's. Our results rely on 273 models, training of which brings about a GPU-heavy computational workload. The training of a single model took approximately 4.5 hours, with an approximate estimated total GPU time for the ultimate results 1270 hours. This total also includes the training time for the 40 models which diverged during training, with the training stopping around 1 hour mark on average. The computational time expended for tail-index estimation in real data experiments and the totality of synthetic experiments amounted to approximately 20 hours of computation with similar hardware as described above.

The experiments were implemented in the Python programming language. For the real data experiments, the deep learning framework PyTorch \citep{pytorch2019} was extensively used, including the implementation methodology in some of its tutorials\footnote{\textsc{https://github.com/pytorch/vision/blob/master/torchvision/models/vgg.py}}. PyTorch is shared under the Modified BSD License.

\clearpage
\end{document}